\documentclass[letterpaper]{article} 
\usepackage{arxiv}
\usepackage{times}  
\usepackage{helvet}  
\usepackage{courier}  
\usepackage[hyphens]{url}  
\usepackage{graphicx} 
\usepackage{natbib}  
\usepackage{caption} 
\usepackage{algorithm}
\usepackage{algorithmic}
\usepackage{authblk}
\usepackage{color}
\usepackage{amsmath}
\usepackage{amsfonts}
\usepackage{amsthm}
\usepackage{dsfont}
\usepackage{booktabs}
\usepackage{verbatim}
\usepackage{multirow}
\usepackage[dvipsnames]{xcolor}
\usepackage[colorlinks=true]{hyperref}

\usepackage{newfloat}
\usepackage{listings}
\DeclareCaptionStyle{ruled}{labelfont=normalfont,labelsep=colon,strut=off} 
\floatstyle{ruled}
\newfloat{listing}{tb}{lst}{}
\floatname{listing}{Listing}
\hypersetup{
    colorlinks = true,
    citecolor  = Green,      
}

\title{\textbf{Generalization Analysis for Deep Contrastive Representation Learning}}
\date{}

\author[1, 2]{\textbf{Nong Minh Hieu}}
\author[2]{\textbf{Antoine Ledent}}
\author[3]{\textbf{Yunwen Lei}}
\author[1]{\textbf{Cheng Yeaw Ku}}
\affil[1]{School of Physical and Mathematical Sciences, Nanyang Technological University, Singapore 639798}
\affil[2]{School of Computing and Information Systems, Singapore Management University, Singapore 188065}
\affil[3]{Department of Mathematics, Hong Kong University, Pok Fu Lam, Hong Kong}

\newcommand{\R}{\mathbb{R}}
\newcommand{\E}{\mathbb{E}}
\newcommand{\1}[1]{\mathds{1}\{ #1 \}}
\newcommand{\bigO}{{\mathcal{O}}}
\newcommand{\ERC}{\mathfrak{\hat R}}
\newcommand{\RC}{\mathfrak{R}}
\newcommand{\F}{\mathcal{F}}
\newcommand{\Sds}{\mathcal{S}}
\newcommand{\V}{\mathcal{V}}
\newcommand{\G}{\mathcal{G}}
\newcommand{\X}{\mathcal{X}}
\newcommand{\FA}{\mathcal{F_A}}
\newcommand{\Hf}{\mathcal{H}}
\newcommand{\B}{\mathcal{B}}

\newcommand{\pA}{{\bf A}}
\newcommand{\bpA}{{{\bf \bar{A}}}}
\newcommand{\A}[1]{A^{(#1)}}
\newcommand{\bA}[1]{{\bar A}^{(#1)}}
\newcommand{\tA}[1]{{\tilde A}^{(#1)}}
\newcommand{\M}[1]{M^{(#1)}}

\newcommand{\bigSquare}[1]{\Big[ #1 \Big]}
\newcommand{\biggSquare}[1]{\Bigg[ #1 \Bigg]}
\newcommand{\bigRound}[1]{\Big( #1 \Big)}
\newcommand{\biggRound}[1]{\Bigg( #1 \Bigg)}
\newcommand{\bigCurl}[1]{\Big\{ #1 \Big\}}
\newcommand{\biggCurl}[1]{\Bigg\{ #1 \Bigg\}}
\newcommand{\bigAbs}[1]{\Big| #1 \Big|}
\newcommand{\biggAbs}[1]{\Bigg| #1 \Bigg|}

\newcommand{\f}{F_\pA}
\newcommand{\fhat}{\widehat{F}_\pA}

\newcommand{\barf}{F_\bpA}
\newcommand{\vf}{\f}

\newcommand{\ff}[2]{F_{\pA^{#1 \to #2}}}
\newcommand{\bff}[2]{F_{\bpA^{#1 \to #2}}}
\newcommand{\xin}{{\bf X}^{\mathrm{(in)}}}
\newcommand{\xtrip}{{\bf X}^{\mathrm{(trip)}}}
\newcommand{\Lun}{\mathrm{L_{un}}}
\newcommand{\Lunhat}{\mathrm{\widehat{L}_{un}}}

\newtheorem{theorem}{Theorem}
\newtheorem{proposition}{Proposition}
\theoremstyle{definition}
\newtheorem{definition}{Definition}
\newtheorem{remark}{Remark}
\newtheorem{lemma}{Lemma}
\newtheorem{corollary}{Corollary}

\begin{document}
\maketitle

\begin{abstract}
    In this paper, we present generalization bounds for the unsupervised risk in the Deep Contrastive Representation Learning framework, which employs deep neural networks as representation functions. We approach this problem from two angles. On the one hand, we derive a parameter-counting bound that scales with the overall size of the neural networks. On the other hand, we provide a norm-based bound that scales with the norms of neural networks' weight matrices. Ignoring logarithmic factors, the bounds are independent of $k$, the size of the tuples provided for contrastive learning. To the best of our knowledge, this property is only shared by one other work, which employed a different proof strategy and suffers from very strong exponential dependence on the depth of the network which is due to a use of the peeling technique. Our results circumvent this by leveraging powerful results on covering numbers with respect to uniform norms over samples. In addition, we utilize loss augmentation techniques to further reduce the dependency on matrix norms and the implicit dependence on network depth. In fact, our techniques allow us to produce many bounds for the contrastive learning setting with similar architectural dependencies as in the study of the sample complexity of ordinary loss functions, thereby bridging the gap between the learning theories of contrastive learning and DNNs.
\end{abstract}

\section{Introduction}
Contrastive Representation Learning (CRL) is a powerful framework that focuses on learning good data representations in an unsupervised learning manner. The CRL framework can be informally described as follows: given a dataset comprising of data tuples $\mathcal{S}=\{(x_j, x_j^+, x_{j1}^-, \dots, x_{jk}^-)\}_{j=1}^n$ where each data instance belongs to an input space $\mathcal{X}$, the key idea of CRL is to pull similar pairs $(x_j, x_j^+)$ closer together and to push apart dissimilar pairs $(x_j, x_{ji}^-)$ in a representation space $\mathcal{R}\subset \R^d$. This is accomplished by training a representation function $f:\mathcal{X}\to\mathcal{R}$ that minimizes the empirical unsupervised risk:

\begin{equation}
    \label{eq:empirical_unsupervised_risk}
    \Lunhat(f) = \frac{1}{n}\sum_{j=1}^n \ell\Big(\Big\{ f(x_j)^\top( f(x_j^+) - f(x_{ji}^-) ) \Big\}_{i=1}^k\Big),
\end{equation}

\noindent where $k$ is the number of negative samples per input data tuple and $\ell:\R^k \to \R_+$ is a contrastive loss function for which popular choices include the hinge and logistic losses:
\begin{equation}
    \label{eq:common_contrastive_losses}
    \begin{aligned}
        \text{Hinge loss: }& \ell(v) = \max\Big\{0, 1 + \max_{1\le i\le k}\{-v_i\}\Big\}, \\
        \text{Logistic loss: }& \ell(v) = \log\Bigg(
            1 + \sum_{i=1}^k \exp(-v_i)
        \Bigg).
    \end{aligned}
\end{equation}

\noindent The learned representations are then used for downstream tasks like classification, clustering or visualization.

Owing to its simplicity and effectiveness, CRL has been applied in a wide variety of machine learning tasks, ranging from computer vision \citep{article:chen2020, article:kaiming2019, article:spyros2018}, graph representation learning \citep{article:hassani2020, article:zhu2020, article:petar2019}, natural language models \citep{article:tianyu2021, article:dejiao2021, article:reimers2021} and time-series forecasting \citep{article:seunghan2024, article:yang2022, article:yuqi2023, article:emadeldeen2021}. Despite the aforementioned successes, very few contributions have been made to explain the good performance of CRL. Even though there are several empirical studies that demonstrate the effectiveness of CRL \citep{article:chen2020, article:kaiming2019}, there are limited theoretical analyses conducted to explain its generalization behaviour. 

In the work of \citet{article:arora2019theoretical}, a theoretical framework to study the generalization behaviour of CRL is proposed. Let $\mathcal{F}=\{f:\mathcal{X}\to\R^d| \|f(x)\|_2 \le B\}$ be a class of representation functions and assume that the loss function is $\ell^\infty$-Lipschitz with constant $\eta>0$, where $\|\cdot\|_p$ denotes the $\ell^p$ norm for $p\geq1$. The authors provided a bound that scales in the order of $\bigO({\eta B\sqrt{k}\mathcal{R_S}(\mathcal{F})}/{n})$, where $\mathcal{R_S}(\mathcal{F})$ is a measure of complexity for the function class $\mathcal{F}$. However, the $\sqrt k$ dependency on negative samples is inconsistent with some of the works that suggest large number of negative samples implicitly implies better generalization or at least does not degrade generalization capability \citep{article:awasthi2022, article:tian2020, article:olivier2020, article:khosla2020}. Therefore, the bound does not fully explain the good generalization behaviour in existing empirical works.

Later in \citet{article:lei2023generalization}, an improvement is made by making the reliance on $k$ at most logarithmic, obtaining the bound in the order of $\tilde \bigO(\eta B\mathcal{R_S}(\mathcal{F})/n)$ (where the $\tilde\bigO$ notation hides logarithmic terms). However, in the case of Deep Contrastive Representation Learning (DCRL) where $\mathcal{F}$ is a class of neural networks with $L$ layers, the authors made use of the peeling technique proposed by \citet{golowich2018size} to derive the following complexity order for the class $\mathcal{F}$: $\mathcal{R_S}(\mathcal{F}) = \bigO\Big(B_x\sqrt{ndL}\prod_{l=1}^L B_{Fr}^{(l)}\Big)$, where $B_x$ is the upper bound on the $\ell^2$ norm of input vectors in the input space $\mathcal{X}$ and $B_{Fr}^{(l)}$ is the Frobenius norm of the weight matrix at the $l^{th}$ layer. Due to the product of Frobenius norms in non-logarithmic terms, the bound suffers from a strong dependency on the neural networks' depth. Unfortunately, this downside is particularly unfavourable in practice when the network architectures are usually deep and the constraints on the weight matrices are not strict.

In terms of proof techniques, both \cite{article:arora2019theoretical} and \cite{article:lei2023generalization} focus on general function classes and rely on vector contraction inequalities \citep{book:ledouxtalagrand2011, article:maurer2019} and inequalities between various complexity measures of the loss class and the feature mapping \citep{book:bartlett2002nnfoundation,srebro2010smoothness,lei2019data,article:lei2023generalization}. This approach is prone to introducing architectural information (final layer's dimension) and dataset size information (number of negative samples) into the generalization bound. Even though the dependency on negative samples is resolved by \citet{article:lei2023generalization} using fat-shattering dimension and worst-case Rademacher complexity, the use of the peeling technique makes the bound scale impractically for (deep) neural networks.

In this work, we demonstrate how to achieve generalization bounds for the Contrastive Learning setting with more flexible tools such as covering numbers. This is achieved through the construction of \textit{auxiliary datasets} consisting of all individual samples involved in any of the input tuples (for further details, cf. Appendix \ref{app:basic_bound}, `Basic Bounds'). This immediately allows us to prove generalization bounds for the DCRL setting with a spectral-type complexity term for the neural network component, a great improvement over the product of Frobenius norms present in the previous state-of-the-art results. Furthermore, 
by exploiting the $\ell^\infty$-Lipschitzness of popular losses such as the hinge loss and logistic loss, we show that this approach can naturally alleviate the strong reliance on the number of negative samples (with at most logarithmic dependency) without the need for other complexity measures such as fat-shattering dimension or peeling techniques. Moreover, we further tighten the complexity bound for neural networks by applying loss augmentation technique  \citep{article:nagarajan2019, article:wei2020datadependent, article:ledent2021normbased} to incorporate data-dependent terms in the bounds. Finally, we derive a parameter-counting bound that scales with the number of neurons in the network with no dependence on the number of negative samples. 

\section{Related Work}
\citet{article:arora2019theoretical} developed a framework to study the generalization behavior of CRL in terms of Rademacher complexity. The analysis there based on $\ell^2$-Lipschitz continuity implies generalization bounds with a linear dependency on the number of negative examples. This linear dependency was recently improved to a logarithmic dependency in \cite{article:lei2023generalization} by arguments relying on $\ell^\infty$-Lipschitz continuity inspired from work on multi-class and multi-label classification \citep{lei2019data,mustafa2021fine,wu2021fine}. These discussions were later extended to CRL with adversarial training examples~\citep{ijcai2024p574,zou2023generalization}. The above discussions are mainly based on Rademacher complexities. Other than this approach, there are also increasing discussions on CRL from the perspective of PAC-Bayesian analysis~\citep{nozawa2020pac}, mutual information~\citep{tsai2020demystifying}, spectral clustering~\citep{haochen2021provable}, gradient-descent dynamics~\citep{tian2020understanding}, distributionally robust optimization~\citep{wu2024understanding} and causality~\citep{mitrovic2020representation}. There is also some work on the generalization analysis of pairwise or triplet wise loss functions in a similar i.i.d. setting as we consider \citep{lei2020sharper,alves2024context,yang2021simple,lei2021generalization}. However, such works do not control the dependence on the number of samples in each input tuple. The benefit of representative learning to improve the generalization of downstream classification tasks were also studied extensively~\citep{article:arora2019theoretical,zou2023generalization,chuang2020debiased,bao2022surrogate}.

CRL often learns nonlinear features by neural networks, and therefore one needs to study the complexity of neural networks to get the corresponding generalization bounds. Nearly tight VC dimension and pseudodimension bounds were developed~\citep{bartlett2019nearly}. Rademacher complexity bounds were developed for neural networks under a norm constraint, which, however, exhibit an exponential dependency on the depth~\citep{neyshabur2015norm}. This exponential dependency was improved to a square-root dependency by using the homogeneity of ReLU networks~\citep{golowich2018size}. Spectrally-normalized margin bounds were developed based on induction arguments with covering numbers~\citep{article:bartlett2017spectrallynormalized,hsugeneralization}. The benefit of weight sharing in convolutional neural networks was also studied based on covering numbers~\citep{article:ledent2021normbased,lin2022universal} and parameter counting~\citep{article:long2020generalization,zhou2024learning}. The benefits of connection-sparsity in CNNs and related architectures was also ingeniously investigated in~\cite{galanti2024norm}.

\section{Problem Formulation}
We begin by briefly describing the theoretical framework from \cite{article:arora2019theoretical} for unsupervised learning task, which we will use to formulate our generalization bounds for unsupervised risk. Let $\mathcal{X}$ denote the space of all possible data points and let $\mathcal{C}$ denote the set of all latent classes. Let $\rho$ be the discrete probability measure over $\mathcal{C}$ and for any $c\in\mathcal{C}$, denote $\mathcal{D}_c$ as the class-conditional distribution such that for any $x\in\mathcal{X}$, $\mathcal{D}_c(x)$ quantifies the likelihood of $x$ being relevant to class $c$. We also define the distribution $\mathcal{\bar D}_c$:

\begin{equation}
    \mathcal{\bar D}_c(x) = \frac{\sum_{z\in \mathcal{C}, z \ne c} \rho(z)\mathcal{D}_{z}(x)}{\sum_{z \in \mathcal{C}, z \ne c} \rho(z)},
\end{equation}

\noindent which quantifies the conditional distribution of $x\in\mathcal{X}$, conditionally given that the class is not equal to $c$. Then, we can define the population unsupervised risk for a representation function as follows.
\begin{definition}{(Unsupervised risk).}
    Let $f:\mathcal{X}\to\mathcal{R}\subset\R^d$ be a representation function and $\ell:\R^k\to\R_+$ be a loss function. The population unsupervised risk of $f$ is:
        \begin{align}
        \Lun(f) = \E_{
            \substack{c\sim \rho, (x, x^+) \sim \mathcal{D}_c^2 \\
            (x_1^-, \dots, x_k^-) \sim \mathcal{\bar D}_c^k}
        }\Bigg[
            \ell\Big(\Big\{ f(x)^\top( f(x^+) - f(x_{i}^-) ) \Big\}_{i=1}^k\Big)
        \Bigg].\nonumber
        \end{align}
\end{definition}

A natural way to find a representation function with low expected unsupervised risk is via empirical risk minimization. Specifically, given a hypothesis class $\F$ and a dataset of the form $\Sds = \bigCurl{(x_j, x_j^+, x_{j1}^-, \dots, x_{jk}^-)}_{j=1}^n$, the best representation function is then determined as the empirical risk minimizer $\widehat{f}_n=\arg\min_{f\in\F}\Lunhat(f)$.

In this paper, we are interested in the performance of $\widehat{f}_n$ on testing dataset. More precisely, we are concerned with its capability to generalize to unseen data. This is often quantified by the generalization gap between the expected unsupervised risk and the empirical unsupervised risk $\Lun(\widehat{f}_n) - \Lunhat(\widehat{f}_n)$. We bound this gap by controlling the Rademacher complexity $\ERC_\mathcal{S}(\G)$ of the loss function class, which we define for a general loss function $\ell:\R^k\to\R_+$ as follows:
\begin{equation}
        \G = \bigCurl{
            (x, x^+, x_1^-, \dots, x_k^-) \mapsto \ell\bigRound{
                \bigCurl{f(x)^\top(f(x^+) - f(x_i^-))}_{i=1}^k
            } : f\in\F
        }.
\end{equation}
\noindent More specifically, we are interested in the case where $\F$ is a class of multi-layered deep neural networks and the loss function $\ell$ is $\ell^\infty$-Lipscthiz.

\begin{table*}[ht]
  \centering
  \begin{tabular}{lccc}
    \toprule
    \textbf{References} & \textbf{Analysis Technique} & \textbf{Generalization Bound} & \textbf{Result} \\
    \midrule 

    \cite{article:arora2019theoretical} & Peeling technique  & 
    $\tilde\bigO\bigRound{\frac{\eta B_x^2\sqrt{kdL}}{\sqrt n}\prod_{l=1}^L \rho_ls_lB_{Fr}^{(l)}}$ & 
    -- \\

    \cite{article:lei2023generalization} & Peeling technique  & 
    $^\mathbb{*}\tilde\bigO\bigRound{\frac{\eta B_x^2\sqrt{dL}}{\sqrt n}\prod_{l=1}^L \rho_ls_lB_{Fr}^{(l)}}$ & 
    -- \\

    \midrule
    Ours & Covering number & 
    $^\mathbb{*}\tilde\bigO\bigRound{\frac{\eta B_x^2}{\sqrt n}\prod_{m=1}^L\rho_m^2s_m^2\bigSquare{\sum_{l=1}^L\frac{a_l^{2/3}}{s_l^{2/3}}}^{3/2}}$ & 
    Thm. \ref{thm:basic_bound} \\
    
    Ours & Covering number \& augmentation &
    $^\mathbb{*}\tilde\bigO\bigRound{\frac{\eta RB_x}{\sqrt n}\prod_{m=1}^L\rho_ms_m\bigSquare{\sum_{l=1}^L\frac{a_l^{2/3}}{s_l^{2/3}}}^{3/2}}$ & 
    Thm. \ref{thm:last_act_augmentation} \\
    
    Ours & Covering number \& augmentation &
    $^\mathbb{*}\tilde\bigO\bigRound{\frac{\eta b_L^2}{\sqrt n}\bigSquare{\sum_{l=1}^L(a_l b_{l-1}\hat\rho_l)^{2/3}}^{3/2}}$& 
    Thm. \ref{thm:all_act_augmentation} \\

    Ours & Parameter counting &
    $\bigO\bigRound{\sqrt{\frac{\mathcal{W}}{n}}\log\bigRound{\eta LnB_x^2\prod_{l=1}^L\rho_l^2s_l^2}}$& 
    Thm. \ref{thm:paracount_bound} \\
    \bottomrule 
  \end{tabular}
  \caption{Summary of main results for Deep Contrastive Representation Learning (DCRL). We assume that the loss function of concern is $\ell^\infty$-Lipschitz with constant $\eta\ge1$. The $\tilde\bigO$ notation hides poly-logarithmic terms of ALL variables and ($^\mathbb{*}$) marks the bounds that have hidden logarithmic dependency on $k$.}
  \label{tab:main_result}
\end{table*}

\section{Main Results}
\subsection{Contributions}
We aim to establish a solid theoretical foundation for DCRL using the flexibility of covering number arguments. The advantages of our bounds are two-fold. Firstly, we manage to alleviate the strong reliance on the number of negative samples and the product of spectral norms using only covering numbers without introducing complexity measures other than Rademacher complexity. Secondly, through loss function augmentation schemes, we are able to further alleviate implicit depth dependency by incorporating data-dependent properties in the bounds. We summarize our key contributions as follows:
\begin{enumerate}
    \item \textbf{Basic generalization bound (Thm~\ref{thm:basic_bound})}: Using a pure covering number approach, we establish a bound for $\ell^\infty$-Lipschitz loss functions with logarithmic dependency on the number of negative samples, which involves a spectral-type complexity measure for the neural network component, but features the square of the spectral norms.
    \item \textbf{Loss function augmentation (Thms~\ref{thm:last_act_augmentation} \&~\ref{thm:all_act_augmentation})}: We improve the basic bound through loss function augmentation: Theorem 2 replaces the extra factor of the product of spectral norms by an empirical maximum output norm, whilst Theorem 3 further improves depth dependency by introducing empirical estimates of intermediate norm activations. 
    \item \textbf{Parameter counting bound (Thm~\ref{thm:paracount_bound})}: In a different style from the above results, we derive a bound that scales with the overall size of the neural networks, i.e. the total number of neurons.
\end{enumerate}

In table \ref{tab:main_result}, we provide a comprehensive summary of our main results as well as the results from the previous works.

\begin{remark}
    We note that in table \ref{tab:main_result}, the original bounds from \citet{article:arora2019theoretical} and \citet{article:lei2023generalization} involve an upper bound $B$ on the output's $\ell^2$ norm that is assumed to hold for \textit{any} representation function in the class. Since we are dealing with a class of neural networks, the upper bound $B$ is expanded to $B_x\prod_{l=1}^L\rho_ls_l$.
\end{remark}

\subsection{Notations}
Let $L\ge 1$, $d_0, d_1, \dots, d_L$ be known natural numbers and $\M{l} \in \R^{d_l\times d_{l-1}}$ be fixed reference matrices. Let $\{a_l\}_{l=1}^L$, $\{s_l\}_{l=1}^L$ be sequences of positive real numbers. We define the following matrix spaces:
\begin{equation}
    \begin{aligned}
    \mathcal{B}_l = \Big\{
        \A{l}\in \R^{d_l \times d_{l-1}} : \|\A{l}\|_\sigma \le s_l, \|(\A{l} - \M{l})^\top\|_{2, 1}\le a_l
    \Big\},
    \end{aligned}
\end{equation}
where $\|\cdot\|_\sigma$ denotes the spectral norm and $\|\cdot\|_{2,1}$ denotes the entry-wise matrix norm quantified by the sum of matrix columns' $\ell^2$ norms. The reference matrices $\M{l}$ are fixed before training and often interpreted as initializations of weight matrices \citep{article:bartlett2017spectrallynormalized, article:ledent2021normbased}. We define the product $\mathcal{A} = \prod_{l=1}^L \B_l$ as the parameters space for the class of neural networks $\FA$:
\begin{equation}
    \FA = \F_L \circ \F_{L-1} \circ \dots \circ \F_1,
\end{equation}

\noindent where $\F_l = \sigma_l \circ \V_l$ such that:
\begin{itemize}
    \item $\sigma_l:\R^{d_l}\to\R^{d_l}$ are $\ell^2$-Lipschitz activation functions with constants $\rho_l$ chosen a priori.
    \item $\V_l = \bigCurl{z \mapsto \A{l}z : \A{l} \in \B_l}$ are classes of linear maps corresponding to pre-activated linear layers.
\end{itemize}

For a given set of weights $\pA = (\A{L}, \dots, \A{1})$ where for each $1\le l \le L$,  $\A{l}\in\B_l$, we denote $\f\in\FA$ as the corresponding neural network parameterized by $\pA$. To be specific, for any $x\in\X$:
\begin{align*}
    \f(x) = \sigma_L\bigRound{
        \A{L}\sigma_{L-1}\bigRound{
            \dots \sigma_1\bigRound{\A{1}x}\dots
        }
    }.
\end{align*}

In the results that follow, we present generalization bounds for unsupervised risk applied for neural networks in the hypothesis class $\FA$. However, we note that our results can be easily made post-hoc to apply for any neural network.

\subsection{Basic Bound}
In this section, we present the basic generalization bound without applying loss augmentation. We begin by stating the definition for $\ell^\infty$-Lipschitz continuity~\citep{lei2019data}.

\begin{definition}[Lipschitz continuity]
    \label{def:lipschitz_continuity}
    We say that a function $\ell:\R^k\to\R_+$ is Lipschitz continuous with respect to the $\ell^\infty$ norm with a constant $\eta>0$ if and only if:
    \begin{equation}
        |\ell(v) - \ell(\bar v)| \le \eta \cdot \|v - \bar v\|_\infty,\quad \forall v,\bar v\in\R^k.
    \end{equation}
\end{definition}

\begin{theorem}
    \label{thm:basic_bound}
    Let $\ell:\R^k\to[0, M]$ be a loss function that is $\ell^\infty$-Lipschitz with constant $\eta>0$. Then, for any $\f\in\FA$ and $\delta\in(0,1)$, the following bound holds with probability of at least $1-\delta$:
    \begin{equation}
        \begin{aligned}
            \Lun(\f) - \Lunhat(\f) \le 3M\sqrt{\frac{\log 2/\delta}{2n}}+\mathcal{\tilde O}\biggRound{
                \frac{\eta B_x^2}{\sqrt n}\log(W)\prod_{m=1}^L\rho_m^2s_m^2\biggSquare{
                    \sum_{l=1}^L \frac{a_l^{2/3}}{s_l^{2/3}}
                }^{3/2}
            }, \\
        \end{aligned}
    \end{equation}

    \noindent where $W = \max_{1\le l \le L}d_l$ (maximum hidden width),  $B_x=\sup_{x\in\X}\|x\|_2$, and the $\mathcal{\tilde O}$ notation hides logarithmic factors in all relevant quantities. 
\end{theorem}

\begin{remark}
Whilst it is standard practice to assume that the loss function is bounded by a fixed constant~\citep{article:long2020generalization,article:bartlett2017spectrallynormalized,ledenticml24,ledentimc21,shamir14}, even in the case where the loss function is not bounded, we can still find an upper bound $M$ for the loss owing to the fact that the weight matrices have bounded norms. Specifically, for all $\f\in\FA$, we can make the following estimation $M = \bigO\bigRound{\eta B_x^2\prod_{l=1}^L\rho_l^2s_l^2}$. Furthermore, no additional dependence on the number of classes is introduced implicitly through $\eta$ when working with the logistic and hinge losses as we do: indeed, the $L^\infty$ Lipschitz constant is bounded by $\eta=1$ in both cases, as shown in Appendix \ref{app:lipschitzness_of_common_un_losses}. Furthermore, we discuss the relationship between the cross-entropy loss from standard classification and its analogues used in CRL. Among the most common analogues are the N-pair loss (which is the logistic loss) and the InfoNCE loss ~\citep{article:oord2018}. The results in this paper extend naturally to the InfoNCE loss (with $\ell^\infty$-Lipschitz constant $\eta=\tau^{-1}$ where $\tau$ is a hyper-parameter specific to InfoNCE. See appendix \ref{app:ce_inspired_losses}, table \ref{tab:ce_inspired_losses}).
\end{remark}

The above bound is the result of directly using covering number to bound the Rademacher complexity. Unlike the vector contraction inequality approach in \citet{article:arora2019theoretical} where the $\sqrt{k}$ dependency creeps into the bound, we immediately observe an absence of significant reliance on the number of negative samples in this result. 

However, the bound also features a factor of the \textit{square} of the product of spectral norms of all layers. This is in contrast to existing norm-based generalization bounds for ordinary neural networks, which typically feature a single product of spectral norms~\citep{article:bartlett2017spectrallynormalized}. Roughly speaking, this new square dependency in the Contrastive Learning Setting is a byproduct of the presence of multiplicative interactions between $f(x)$ and $f(x^{+})$ or $f(x^{-})$, which means that the errors propagate through the network twice. In the next section, we discuss how we can make use of data-dependent properties to alleviate this issue with simple loss augmentation techniques.

\subsection{Loss Augmentation}
In previous works dedicated to multi-class classification problem \citep{article:nagarajan2019, article:wei2020datadependent, article:ledent2021normbased}, it has been shown that we can obtain tighter Rademacher complexity bound by incorporating data-dependent quantities. Informally, this is accomplished by augmenting the original loss function in a way that the augmented loss collapses to a large value if certain data-dependent well-behaved-ness properties do not hold. For instance, given the list of data-dependent properties $\{\gamma_l\}_{l=1}^m$ and their corresponding desired bounds ${\bf B}=\{b_l\}_{l=1}^m$, \citet{article:wei2020datadependent} employ an augmentation scheme involving products of soft indicators of the data-dependent properties:
\begin{align}
    \label{eq:loss_aug_weima}
    \tilde \ell(x) = 1 + (\ell(x) - 1)\prod_{l=1}^m \lambda_{b_l}(\gamma_l(x)),
\end{align}
\noindent where $\ell:\R^k \to [0, 1]$ is the original loss function and $\lambda_{b_l}$ are soft indicators with margins $b_l$ (whose definition is identical to that of the ramp loss), defined as follows:

\begin{equation}
    \begin{aligned}
    \lambda_\gamma(r) = \begin{cases}
        0 & r < -\gamma \\
        1 + r/\gamma & r \in [-\gamma, 0] \\
        1 & r > 0
    \end{cases}.
    \end{aligned}
\end{equation}

Another example of loss augmentation is the work of \citet{article:ledent2021normbased} where the augmented loss is the maximum value between the original loss and the maximum of the soft indicators themselves:
\begin{equation}
    \label{eq:loss_aug_ledent}
    \tilde \ell(x) = \max\bigSquare{
        \ell(x), \max_{1\le l \le m}\lambda_{b_l}(\gamma_l(x))
    }.
\end{equation}

\noindent These soft indicators act as validation filters for the intended data-dependent properties. Specifically, the value of $\tilde \ell$ will coincide with the original loss value if all bound conditions are met. On the other hand, when $\gamma_l(x) \ge 2b_l$ for any $1 \le l \le m$, $\tilde\ell$ will collapse to the upper bound of $\ell$, making the augmented loss uniformly larger than the original loss. As a result, we can bound the excess risk of the original loss indirectly via the augmented loss. To be more precise, let $\mathcal{D}$ be a distribution over an input space $\X$ and $S=\{x_1, \dots, x_n\}\in\X^n$ be a dataset drawn i.i.d from $\mathcal{D}$. Define the excess risk for a particular loss $\ell:\R^k\to\R_+$ as $\mathcal{E}[\ell;S]=\E_{x\sim\mathcal{D}}[\ell(x)] - \frac{1}{n}\sum_{j=1}^n\ell(x_j)$, we have:
\begin{align}
    \mathcal{E}[\ell;S] \le \mathcal{E}[\tilde \ell;S]  + \frac{\mathcal{I}_{\bf B}}{n},
\end{align}

\noindent where $\mathcal{I}_{\bf B} = \bigAbs{\bigCurl{x_j \in S: \exists l \text{ s.t } \gamma_l(x_j) > b_l}}$, which is the count of data points that do not satisfy all bound conditions. Notice that we can bound the augmented generalization gap $\E_{x\sim\mathcal{D}}[\tilde\ell(x)] - \frac{1}{n}\sum_{j=1}^n\tilde\ell(x_j)$ by controlling the Rademacher complexity of the augmented loss class (which we denote by default as $\tilde \G$). The difficulty of this approach is that the Rademacher complexity of the augmented loss class can be much more complex than the original class. 

In this section, we consider augmentation schemes that tighten the generalization bound and improve on the result in Theorem \ref{thm:basic_bound}. By default, we consider the original loss function $\ell:\R^k\to\R_+$ to be $\ell^\infty$-Lipschitz with constant $\eta>0$ and, without loss of generality, we assume that $\ell$ is bounded by $1$. Our first attempt is through imposing bound conditions on the representation output of the neural networks. To be more precise, let $R>0$ be a fixed real constant intended to be the upper bound for the output representation's $\ell^2$ norm, we consider the following augmented loss function class:
    \begin{align}
        \tilde \G = \bigCurl{
            \xin = (x, x^+, x_1^-, \dots, x_k^-) \mapsto  
            &\max\bigSquare{
                \ell\bigRound{\bigCurl{\f(x)^\top(\f(x^+) - \f(x_i^-)}_{i=1}^k}, \\ \nonumber 
                &\max_{\tilde x\in\xin}\lambda_R(\|\f(\tilde x)\|_2)
            } : \f\in\FA
        }.
    \end{align}
\noindent Bounding the Rademacher complexity of the above class results in the following theorem, which is our second main contribution. 
\begin{theorem}
    \label{thm:last_act_augmentation}
    Let $\ell:\R^k\to[0,1]$ be a loss function that is $\ell^\infty$-Lipschitz with constant $\eta\ge1$ and let $R\ge1$ be given. Then, for any $\f\in\FA$ and $\delta\in(0,1)$, the following bound holds with probability of at least $1-\delta$:
        \begin{align}
            \Lun(\f) - \Lunhat(\f) \le \frac{\mathcal{I}_{\pA, R}}{n} + 3\sqrt{\frac{\log 2/\delta}{2n}}+\tilde\bigO\biggRound{
                \frac{\eta RB_x}{\sqrt n}\log(W) \prod_{m=1}^L\rho_ms_m\biggSquare{
                    \sum_{l=1}^L \frac{a_l^{2/3}}{s_l^{2/3}}
                }^{3/2}
            }, 
        \end{align}
    \noindent where $W = \max_{1\le l \le L}d_l$, $B_x = \sup_{x\in\X}\|x\|_2$, and $\mathcal{I}_{\pA, R}$ is defined as:
    \begin{align*}
        \mathcal{I}_{\pA, R} = \bigAbs{
            \bigCurl{
                \xin_j \in \mathcal{S}: \exists \tilde x\in\xin_j \textup{ s.t } \|\f(\tilde x)\|_2 > R
            }
        },
    \end{align*}

    \noindent where $\xin_j=(x_j, x_j^+, x_{j1}^-, \dots, x_{jk}^-)$ is the $j^{th}$ input tuple from the dataset $\mathcal{S}$.
\end{theorem}

\begin{remark}
Note that if all the samples satisfy $\max_{\tilde x\in\xin_j} \|\f(\tilde x)\|_2 \leq R$, then $\frac{\mathcal{I}_{\pA, R}}{n}=0$. Furthermore, by a union bound, it is not difficult to show that a similar result holds even if $R$ is selected from the data by observing the value of $\max_{\tilde x\in\xin_j} \|\f(\tilde x)\|_2 $. For further details, we refer the reader to Appendix \ref{app:post_hoc_analysis}, `Post Hoc Analysis'. This remark also applies to Theorem~\ref{thm:all_act_augmentation}. 
\end{remark}

Further, although Theorem \ref{thm:last_act_augmentation} assumes the original loss function is bounded by $1$, we can easily generalize to any loss function $\ell:\R^k\to[0,M]$ by considering a slightly different augmentation scheme where the original loss is normalized to $M^{-1}\ell$ inside the max function. Then, an equivalent bound to theorem \ref{thm:last_act_augmentation} for loss functions bounded by an arbitrary $M>0$ is:
\begin{align}
        \Lun(\f) - \Lunhat(\f) \le \frac{M\mathcal{I}_{\pA, R}}{n} + 3M\sqrt{\frac{\log 2/\delta}{2n}}+\tilde\bigO\biggRound{
            \frac{\eta RB_x}{\sqrt n}\log(W) \prod_{m=1}^L\rho_ms_m\biggSquare{
                \sum_{l=1}^L \frac{a_l^{2/3}}{s_l^{2/3}}
            }^{3/2}
        }. 
    \end{align}

Compared to theorem \ref{thm:basic_bound}, we have successfully reduced the product of squared spectral norms dependency down to a single product of spectral norms at the cost of a multiplicative factor of the more well-behaved empirical quantity $R$ and an additive term of $\mathcal{I}_{\pA, B}/n$, which is the proportion of inputs in $\Sds$ that do not satisfy the output bound condition. In the following, we illustrate that the bound can be improved further by considering an augmentation scheme that enforces bounds on all the hidden layers' activations. Specifically, given ${\bf B}=\{b_0, b_1, \dots, b_L\}$ a sequence of known positive constants, we consider the following augmented class:
    \begin{align}
        \tilde\G = \bigCurl{
           \xin = (x, x^+, x_1^-, \dots, x_k^{-}) \mapsto &\max\bigSquare{
                \ell\bigRound{\bigCurl{\f(x)^\top(\f(x^+) - \f(x_i^-)}_{i=1}^k}, \\
                &\max_{1\le l \le L}\max_{\tilde x\in\xin}\lambda_{b_l}(\|\f^{1\to l}(\tilde x)\|_2)
           }: \f\in\FA
        }.\nonumber 
    \end{align}

\noindent where $\f^{1\to l}$ denotes the sub-network that consists of the first $l$ layers of $\f\in\FA$. Bounding the above class Rademacher complexity yields the following result:

\begin{theorem}
    \label{thm:all_act_augmentation}
    Let $\ell:\R^k\to[0,1]$ be a loss function that is $\ell^\infty$-Lipschitz with constant $\eta\ge1$. Let ${\bf B}=\{b_0, b_1, \dots, b_L\}$ be a sequence of known positive constants such that $b_l\ge1$ for all $0 \le l \le L$. Then, for any $\f\in\FA$ and $\delta\in(0,1)$, the following bound holds with probability of at least $1-\delta$:
    \begin{equation}
        \begin{aligned}
            \Lun(\f) - \Lunhat(\f) \le \tilde\bigO\biggRound{
                \frac{\eta b_L^2\mathcal{\widehat{R}_A}}{\sqrt n}\log(W)} +
            \frac{\mathcal{I}_{\pA, {\bf B}}}{n} + 3\sqrt{\frac{\log 2/\delta}{2n}},
        \end{aligned}
    \end{equation}

    \noindent where $\mathcal{\widehat{R}_A}$ is defined as follows:
    \begin{align*}
        \mathcal{\widehat{R}_A}^{2/3} &= \sum_{l=1}^L (a_l b_{l-1}\hat\rho_l)^{2/3} \textup{ where } \hat\rho_l = \rho_l \sup_{u\ge l}b_u^{-1}\prod_{m=l+1}^u s_m\rho_m, \\
        \textup{and }
        \mathcal{I}_{\bf A, B} &= \bigAbs{
                \bigCurl{
                    \xin_j \in \Sds: \exists l \in [L], \tilde x \in \xin_j \textup{ s.t } \|\f^{1\to l}(\tilde x)\|_2 > b_l
                }
            }.
    \end{align*}
\end{theorem}

Again, without loss of generality, we can derive an analogous bound to the above result for loss functions bounded by any $M>0$. Unlike the previous result which depends on the full $L$ layers product of spectral norms, the $l^{th}$ term in the above result only involves spectral norms of layers $l+1$ up to $L$ (but not necessarily all the way to $L$). 

\subsection{Parameter Counting Bound}
Inspired by previous works developed for neural networks used in multi-class classification \citep{article:long2020generalization, article:florian2022, thesis:sebro2004, book:mohri2018}, our result below scales with network's size rather than the magnitude of weight matrices norms like the bounds presented in the previous section. The advantage of this type of bounds is the absence of a product of spectral norms (outside logarithmic factors), which effectively eliminates the strong dependency on neural network's depth.
\begin{theorem}
    \label{thm:paracount_bound}
    Let $\ell:\R^k\to[0, M]$ be a loss function that is $\ell^\infty$-Lipschitz with constant $\eta > 0$ and $\mathcal{W}=\sum_{l=1}^L d_l$. Then, for any $\f\in\FA$ and $\delta\in(0,1)$, the following bound holds with probability of at least $1-\delta$:
        \begin{align}
            \Lun(\f) - \Lunhat(\f) \le 3M\sqrt{\frac{\log 2/\delta}{2n}}+\bigO\biggRound{
                M\sqrt{\frac{\mathcal{W}}{n}}\log\bigRound{
                    1 + 24\eta LnB_x^2\prod_{l=1}^L\rho_l^2s_l^2
                }
            }. 
        \end{align}
\end{theorem}

Essentially, the above result scales with the total number of parameters of the neural networks. This characteristic can be disadvantageous compared to the previous norm-based results because $(1)$ the bound can become unreasonably large for massive architectures and $(2)$ the bound will still scale with $\mathcal{W}$ even if the weight matrices are arbitrarily close to the reference matrices. Even though the above bound might scale unfavourably in the case of large neural networks, we note that it has no reliance on the number of negative samples. Hence, it can be particularly useful in cases when the networks are small and we have a large amount of negative samples.

\subsection{Downstream Classification}
In this section, we discuss the application of the generalization bounds for unsupervised risk in the downstream classification task. We begin with the following definition of a classifier's population supervised risk.
\begin{definition}{(Supervised risk).}
    Fixing a $(K+1)$-way supervised task $\mathcal{T}=\{c_1, \dots, c_{K+1}\}\subseteq \mathcal{C}$ (where $\mathcal{C}$ is the set of latent classes defined in the previous section). Let $g:\X\to\R^{K+1}$ be a multi-class classifier and $\ell:\R^{K}\to\R_+$ be a loss function. The population supervised risk of $g$ is defined as follows:
    \begin{align*}
        \mathrm{L}_\mathrm{sup}(\mathcal{T}, g) = \E_{(x, c)\sim \mathcal{D_T}}\bigSquare{
            \ell\bigRound{
                \bigCurl{
                    g(x)_c - g(x)_{c'}
                }_{c' \ne c}
            }
        },
    \end{align*}
    \noindent where $\mathcal{D_T}$ is the joint distribution over $\X\times\mathcal{T}$.
\end{definition}

In particular, we are interested in the class of mean classifiers from \citet{article:arora2019theoretical}. Let $\mathcal{T}\subseteq\mathcal{C}$ such that $|\mathcal{T}|=K+1$ and $f:\X\to\R^d$ be a representation function. A mean classifier $g:\X\to\mathcal{T}$ is defined as $g(x) = W^\mu f(x)$, where $W^\mu\in \R^{(K+1)\times d}$ is a weight matrix such that for each $c \in \mathcal{T}$, the $c^{th}$ row of $W^\mu$ is the expected representation of $x\in\X$ given that $x$ is relevant to class $c$. Specifically, $W^\mu_c = \E_{x\sim\mathcal{D}_c} [f(x)]$. Consider the average supervised loss:
\begin{align*}
    \mathrm{L}_\mathrm{sup}^\mu(f) = \E_{\mathcal{T}\sim\rho^{K+1}}\bigSquare{
        \mathrm{L}_\mathrm{sup}(\mathcal{T}, W^\mu f) \Big| c_i \ne c_j
    },
\end{align*}

\noindent which is the expectation of the mean classifier's supervised loss taken over $(K+1)$-way supervised tasks (with unique classes). In a general sense, the average supervised loss can be translated to a performance metric for the representation $f$ when it is used to build a mean classifier. In the following lemma from \citet{article:arora2019theoretical}, it is shown the average supervised loss can be upper bounded by the population unsupervised risk:

\begin{lemma}
    \label{lem:arora_average_sup_loss}
    Fixing a class of representation functions $\F$ and let $\widehat{f}_n = \arg\min_{f\in\F}\Lunhat(f)$. There exists a function $\rho:\mathcal{C}^{K+1}\to \R_+$\footnote{This lemma is an intermediate step in the proof of theorem B.1 from \citet[equation 26]{article:arora2019theoretical}. For the exact form of the function $\rho$, we refer readers to their  proof. For the formal definition of the distribution $\mathcal{D}$ of $(K+1)$-way classification tasks, please refer to \citet[section 6.1]{article:arora2019theoretical}.} such that:
    \begin{equation}
        \E_{\mathcal{T}\sim\mathcal{D}}\bigSquare{\rho(\mathcal{T})\mathrm{L}_\mathrm{sup}^\mu(\widehat{f}_n)} \le \Lun(\widehat{f}_n),
    \end{equation}

    \noindent where $\mathcal{D}$ is a distribution over $(K+1)$-way supervised tasks $\mathcal{T}\in\mathcal{C}^{K+1}$ such that there are no repeated classes in $\mathcal{T}$.
\end{lemma}

In the following results, we directly apply the generalization bounds obtained in the previous sections into lemma \ref{lem:arora_average_sup_loss}.

\begin{corollary}{(Norm-based bound).}
    Let $\ell:\R^k\to[0,M]$ be a an $\ell^\infty$-Lipschitz loss with $\eta\ge1$. Let ${\bf B}=\{b_0, b_1, \dots, b_L\}$ be a sequence of known positive constants such that $b_l\ge1$ for all $0\le l \le L$. Let $\fhat$ be the empirical unsupervised risk minimizer, then, for any $\delta\in(0,1)$, we have:
    \begin{equation}
        \E_{\mathcal{T}\sim\mathcal{D}}\bigSquare{\rho(\mathcal{T})\mathrm{L}_\mathrm{sup}^\mu(\fhat)} \le \Lunhat(\fhat) + 3M\sqrt{\frac{\log 2/\delta}{2n}} + \tilde\bigO\biggRound{
            \frac{\eta b_L^2\mathcal{\widehat{R}_A}}{\sqrt n}\log(W)} + \frac{M\mathcal{I}_{\pA, {\bf B}}}{n},
    \end{equation}
    \noindent where $\mathcal{\widehat{R}_A}$ and $\mathcal{I}_{\bf A, B}$ are defined in Theorem \ref{thm:all_act_augmentation}.
\end{corollary}

\begin{corollary}{(Parameter-counting bound).}
    Let $\ell:\R^k\to[0, M]$ be a loss function that is $\ell^\infty$-Lipschitz with constant $\eta \ge 1$. Let $\fhat$ be the empirical unsupervised risk minimizer, then, for any $\delta\in(0,1)$, we have:
    \begin{equation}
        \E_{\mathcal{T}\sim\mathcal{D}}\bigSquare{\rho(\mathcal{T})\mathrm{L}_\mathrm{sup}^\mu(\fhat)} \le \Lunhat(\fhat) + 3M\sqrt{\frac{\log 2/\delta}{2n}}
        +\bigO\biggRound{
            M\sqrt{\frac{\mathcal{W}}{n}}\log\bigRound{
                1 + 24\eta LnB_x^2\prod_{l=1}^L\rho_l^2s_l^2
            }
        }.
    \end{equation}
\end{corollary}

\section{Experiments}
\begin{figure*}[htbp!]
    \centering
    \includegraphics[width=\linewidth]{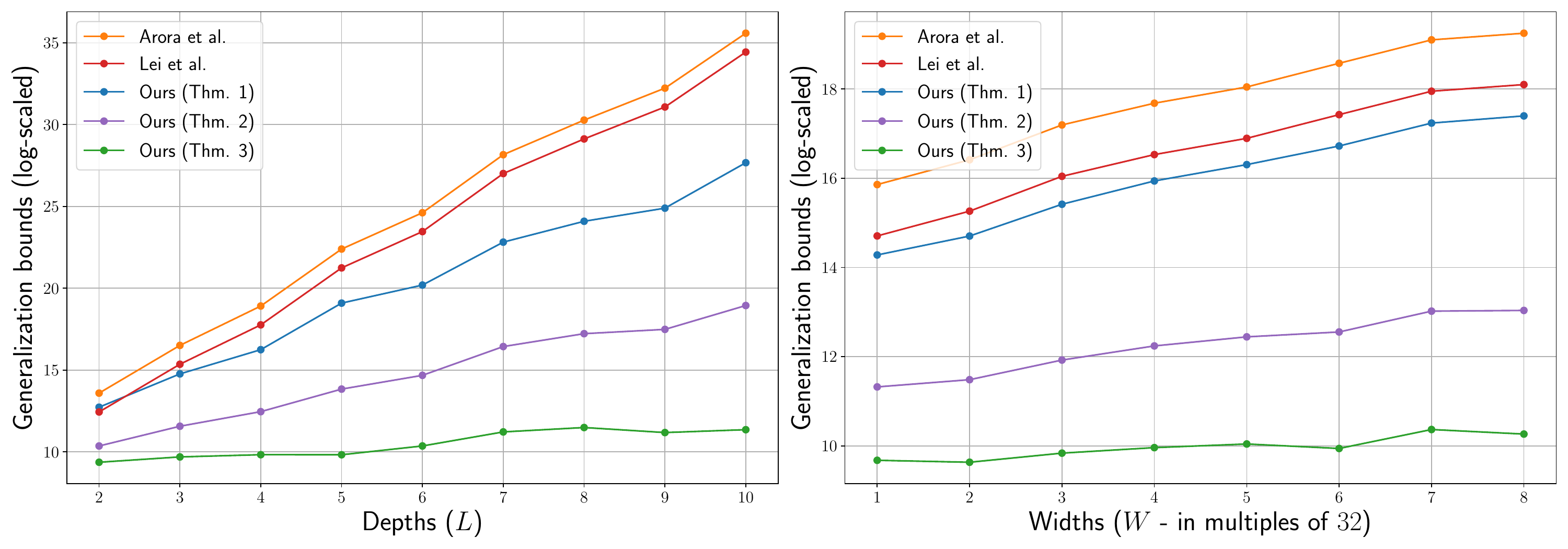}
    \caption{Graphical comparison of our results to that of previous works \citep{article:arora2019theoretical, article:lei2023generalization}. The generalization bounds for all results have their logarithmic terms, constants ($\eta, \rho_i, \dots$) and $\bigO(\sqrt{\log 1/\delta})$ terms truncated. We present the comparison at varying depths (Left) and hidden layer's dimensions (Right).}
    \label{fig:ablation_study_depth}
\end{figure*}
To compare our results with previous works, we conducted experiments by training fully-connected deep neural networks with the MNIST digits dataset \citep{data:yanlecunn1998} with a train-test  ratio of $75\%/25\%$. We ran two ablation studies to test how our bounds vary with network depth and hidden layer dimension compared to the bounds proposed by \citet{article:arora2019theoretical} and \citet{article:lei2023generalization}. For the first experiment, we fixed the hidden layer dimensions to $64$ and trained deep neural networks at different depths in the $[2, 10]$ range. For the second experiment, we fixed the depth to $L=3$ and trained deep neural networks at different hidden layer dimensions of $32, 64, 128, \dots$ (in multiples of $32$). In both experiments, we fixed the output dimension to $d=64$ and the number of negative samples to $k=10$ (furthermore, additional experiments with $k=64$ are provided in Appendix \ref{app:additional_experiments}). For all the neural networks trained in both experiments, we set the maximum number of training iterations to $1000$ and stopped until the empirical unsupervised loss reached $1\mathrm{e}{-4}$ to ensure that all networks roughly converge to the empirical risk minimizers. A summary of our experiment results is presented in figure \ref{fig:ablation_study_depth}: the y axis shows the main factor in our and competing bounds, ignoring constants and logarithmic terms in all cases. The results demonstrate that our generalization bounds outperform the competing ones, especially for larger widths and depths.

\section{Conclusion and Further Works}
There is very limited amount of theoretical work explaining the impressive real world performance that CRL has achieved. Existing works focus on the case of general classes of representation functions through direct arguments on the Rademacher complexity. In the case of neural networks, this introduces strong depth dependency through a product of Frobenius norms of the weight matrices.
In this work, we provided bounds relied on applying covering number arguments to carefully constructed auxiliary datasets to provide bounds with a milder dependency on depth.  We also illustrate that with such a technique, the bounds automatically admit a weak dependency on the number of negative samples. 
Moreover, through loss augmentation, we improve our results by introducing data-dependent terms into the bounds, lessening the effect of residual exponential growth with the neural network's depth. In further work, it would be interesting to generalize our work to other architectures such as CNNs, GNNs, ResNets. Furthermore, a particularly tantalizing direction would be to study the generalization properties of CRL in the more realistic and challenging setting where the input tuples are formed from a fixed pool of reusable labeled examples. This would be much more challenging due to the violation of the i.i.d. assumption.  

\section{Acknowledgements}
YL acknowledges support by the Research Grants Council of Hong Kong (Project No. 22303723).

\newpage
\bibliography{arxiv.bib}

\appendix
\onecolumn
\numberwithin{equation}{section}
\numberwithin{theorem}{section}
\numberwithin{figure}{section}
\numberwithin{table}{section}
\renewcommand{\thesection}{{\Alph{section}}}
\renewcommand{\thesubsection}{\Alph{section}.\arabic{subsection}}
\renewcommand{\thesubsubsection}{\Roman{section}.\arabic{subsection}.\arabic{subsubsection}}
\setcounter{secnumdepth}{-1}
\setcounter{secnumdepth}{3}

\section{Table of Notations}
\renewcommand{\arraystretch}{1.6}
\begin{table}[!ht]
\begin{center}
\begin{tabular}{l|l}
\hline
\textbf{Notation}         & \textbf{Description}                                            \\ \hline

$\f^{l_1\to l_2}$ 
& Extraction from layer $l_1$ to $l_2$. $\vf^{l_1\to l_2}(z)=\sigma_{l_2}(\A{l_2}\sigma_{l_2-1}(\dots \sigma_{l_1}(\A{l_1}z) \dots ))$.
\\

$B_x$                     
& Upper bound of inputs $\ell^2$ norm. $B_x=\max_{x \in \X} \|x\|_2$.
\\

$B_l$                     
& Upper bound of $l^{th}$ layer's activation. $B_l=\max_{x \in \X} \max_{\pA\in\mathcal{A}} \|\f^{1\to l}(x)\|_2$ ($B_0 = B_x$).
\\

$W$ 
& Maximum width of neural networks in $\F_\mathcal{A}$. $W = \max_{1\le l \le L}d_l$.
\\

$\mathcal{W}$ 
& Overall size of neural networks in $\F_\mathcal{A}$. $\mathcal{W} = \sum_{l=1}^Ld_l$.
\\

\hline
\multicolumn{2}{c}{\textbf{Constants relating to parameters space}} \\
\hline

$a_l$ 
& Upper bounds for $\|.\|_{2,1}$ norms of weight matrices translated by fixed reference matrices.
\\

$s_l$ 
& Upper bounds for spectral norms of weight matrices.
\\

$\rho_l$ 
& The Lipschitz constant (w.r.t $\ell^2$ norm) of the $l^{th}$ layer's activation.
\\

$\rho_{l+}$ 
& Forward product of spectral norms. $\rho_{l+} = \rho_l\prod_{l=m+1}^Ls_m\rho_m$ for $1\le l \le L$.
\\

$\M{l}$ 
& Fixed reference matrices. $\M{l}\in\R^{d_l\times d_{l-1}}$ for $1\le l\le L$.
\\

$\B_l$                     
& $l^{th}$ layer parameter space. $\B_l = \bigCurl{\A{l}\in\R^{d_l\times d_{l-1}}: \|\A{l}\|_\sigma \le s_l, \|\A{l}-\M{l}\|_{2,1}\le a_l}$. 
\\

\hline
\multicolumn{2}{c}{\textbf{Main and auxiliary datasets}} \\
\hline

$\xtrip_{ij}$ 
& $j^{th}$ input triplet with $i^{th}$ negative sample. $\xtrip_{ij} = (x_j, x_j^+, x_{ji}^-)$ for all $1\le j \le n$ and $1\le i \le k$.
\\

$\xin_j$ 
& $j^{th}$ input tuple. $\xin_j=(x_j, x_j^+, x_{j1}^-, \dots, x_{jk}^-)$ for all $1\le j \le n$.
\\

$\mathcal{S}$ 
& The collection of input tuples. $\mathcal{S}=\{\xin_j\}_{j=1}^n$.
\\

$\mathcal{S}_1$ 
& The collection of input triplets. $\mathcal{S}_1=\bigcup_{j=1}^n\bigCurl{(x_j, x_j^+, x_{ji}^-): 1\le i \le k}$.
\\

$\mathcal{S}_2$ 
& The collection of input vectors. $\mathcal{S}_2=\bigcup_{j=1}^n\bigCurl{\tilde x : \tilde x \in \xin_j}$.
\\

\hline
\multicolumn{2}{c}{\textbf{Covering number notations}} \\
\hline

$\mathcal{N}(F, \epsilon, d)$
& The empirical $\epsilon$-covering number of a set $F$ with respect to metric $d$.
\\

$\mathcal{N}(\F, \epsilon, L_p(S))$
& The empirical $\epsilon$-covering number of a \textbf{real-valued} function class with respect to $L_p(S)$ metric.
\\

& $\mathcal{N}(\F, \epsilon, L_p(S)) = \min\bigCurl{|\mathcal{C}|:\forall f\in\F, \exists \bar f \in \mathcal{C} \text{ s.t } \|f - \bar f\|_{L_p(S)} \le \epsilon}$,  
\\

& where $\|f - \bar f\|_{L_p(S)} = \bigRound{\frac{1}{|S|}\sum_{x\in S}|f(x) - \bar f(x)|^p}^{1/p}$.
\\

$\mathcal{N}(\F, \epsilon, L_\infty(S))$
& The empirical $\epsilon$-covering number of a \textbf{real-valued} function class with respect to $L_\infty(S)$ metric.
\\

& $\mathcal{N}(\F, \epsilon, L_\infty(S)) = \min\bigCurl{|\mathcal{C}|:\forall f\in\F, \exists \bar f \in \mathcal{C} \text{ s.t } \|f - \bar f\|_{L_\infty(S)} \le \epsilon}$,    
\\

& where $\|f - \bar f\|_{L_\infty(S)} = \max_{x\in S}|f(x) - \bar f(x)|$.
\\

$\mathcal{N}(\F, \epsilon, L_{\infty, 2}(S))$
& The empirical $\epsilon$-covering number of a \textbf{vector-valued} function class with respect to $L_{\infty, 2}(S)$ metric.
\\

& $\mathcal{N}(\F, \epsilon, L_{\infty,2}(S)) = \min\bigCurl{|\mathcal{C}|:\forall f\in\F, \exists \bar f \in \mathcal{C} \text{ s.t } \|f - \bar f\|_{L_{\infty, 2}(S)} \le \epsilon}$,   
\\

& where $\|f - \bar f\|_{L_{\infty, 2}(S)} = \max_{x\in S}\|f(x) - \bar f(x)\|_2$.
\\

\hline
\multicolumn{2}{c}{\textbf{Empirical terms}} \\
\hline

${\bf\hat B}_x$ &
${\bf\hat B}_x = \sup_{x\in\Sds_2}\|x\|_2$. Upper bound of all inputs' $\ell^2$ norm.
\\

${\bf\hat B}_\pA$ &
${\bf\hat B}_\pA = \sup_{x\in\Sds_2}\|\f(x)\|_2$. Upper bound of outputs' $\ell^2$ norm given parameters $\pA\in\mathcal{A}$.
\\

\hline
\end{tabular}
\end{center}
\caption{Notations table for quick reference.}
\label{tab:notation_table}
\end{table}

\section{Rademacher Complexity}
As explained in the main text, the main tool that we will use to bound the generalization gap for unsupervised risk will be the Rademacher complexity. In this section, we provide the definition and explain briefly the intuition behind Rademacher complexity.
Let $\sigma_i$ be independent Rademacher random variables taking values of $\{-1, 1\}$ with equal probabilities. 
\begin{definition}{(Rademacher complexity).}
    Let $\mathcal{Z}$ be a vector space and $\mathcal{D}$ be a distribution over $\mathcal{Z}$. Let $\G$ be a class of functions $g:\mathcal{Z}\to[a, b]$ where $a,b\in\R, \ a < b$. Let $S=(z_1, \dots, z_n)$ be a sample of i.i.d random variables drawn from $\mathcal{D}$. Then, the {\textit{empirical Rademacher complexity}} of $\G$ is defined as:
    \begin{align}
        \ERC_S(\G) = \E_\sigma\biggSquare{
            \sup_{g\in\G}\frac{1}{n}\sum_{i=1}^n \sigma_i g(z_i)
        }.
    \end{align}

    \noindent The {\textit{Rademacher complexity}} is then obtained by taking the expectation of empirical Rademacher complexity over the samples of i.i.d random variables drawn from $\mathcal{D}$:
    \begin{align}
        \RC_n(\G) &= \E_{S\sim\mathcal{D}^n}\bigSquare{\ERC_S(\G)}.
    \end{align}

    \noindent Intuitively, the empirical Rademacher complexity quantifies the ``richness" of a function class. More specifically, it measures the capacity of a function class to to fit randomly generated symmetric signs (the random Rademacher variables). The more expressive the function class, the higher chance of finding $g\in\G$ that correlates well with a given sequence of labels from $\{-1, 1\}^n$. We will use two classic results in learning theory, the Rademacher complexity bound (proposition \ref{prop:rademacher_complexity_bound}) and the Dudley's entropy integral bound (theorem \ref{eq:dudley_entropy_integral}), to relate the Rademacher complexity to covering number~\citep{srebro2010smoothness}.
\end{definition}
\begin{proposition}{(Rademacher Complexity bound).}
    \label{prop:rademacher_complexity_bound}
    Let $\mathcal{Z}$ be a vector space and $\mathcal{D}$ be a distribution defined on $\mathcal{Z}$. Consider a class $\G$ consisting of functions $g:\mathcal{Z}\to [a, b]$ where $a, b\in\R, \ a < b$. Let $S=\{z_1, \dots, z_n\}\in\mathcal{Z}^n$ be a sample of i.i.d random variables drawn from $\mathcal{D}$. Then, for all $\delta\in(0,1)$, with probability of at least $1-\delta$, we have:
    \begin{align}
        \E_{z\sim\mathcal{D}}\bigSquare{g(z)} - \frac{1}{n}\sum_{i=1}^n g(z_i)
        \le 2\ERC_S(\G) + 3(b-a)\sqrt{\frac{\log 2/\delta}{2n}}.
    \end{align}
\end{proposition}

\begin{theorem}[Dudley's entropy integral\label{eq:dudley_entropy_integral}, cf. Lemma E.1 in~\cite{ledenticml24} or Lemma 8.5 in~\cite{article:bartlett2017spectrallynormalized} ]
    Let $\X$ be a vector space and $\F$ be a class of functions $f:\X\to\R$. Let the dataset $S=\{x_1, \dots, x_n\}\in \X^n$ be given. Let $B_\mathcal{F}=\sup_{f\in\F}\|f\|_{L_2(S)}$, we have:
    \begin{align}
        \ERC_S(\F) \le \inf_{\alpha>0}\biggRound{
            4\alpha + \frac{12}{\sqrt{n}} \int_\alpha^{B_\F}\sqrt{
                \log\mathcal{N}\bigRound{\F, \epsilon, L_2(S)}
            }d\epsilon
        },
    \end{align}

    \noindent where the $\|f\|_{L_2(S)}$ norm evaluated for $f\in\F$ is $\|f\|_{L_2(S)} = \bigRound{\frac{1}{n}\sum_{i=1}^n |f(x_i)|^2}^{1/2}$.
\end{theorem}

\section{Important Covering Number Bounds}
As we will see in the next section, the key of the main result on covering number for the loss function class $\G$ (or any augmented class $\tilde\G$ that follows) relies on bounding the $L_{\infty, 2}$ covering number of the neural networks class $\F_\mathcal{A}$. In this section, we will prove the $L_{\infty,2}$ covering number bound for the simplest case of one linear layer, which will be crucial for the multi-layer case later. The proof of this one-layer case follows typical proof techniques as in~\cite{article:florian2022,article:ledent2021normbased} and no strong claim of originality is made for this section. Before that, we revisit the following result from \cite{article:zhang2002cover} without proof:

\begin{proposition}[\citealt{article:zhang2002cover}\label{prop:zhang_linfty_covering_number}]
    Let $2\le p, q < \infty$ such that $1/p + 1/q = 1$ and $a, b\in \R$ are known positive constants. Define the class $\mathcal{L}$ as follows:
    \begin{align}
        \mathcal{L} = \bigCurl{
            x\mapsto Wx : x, W\in\R^d, \|W\|_q \le a
        }.
    \end{align}

    \noindent Then, for a given dataset $S=\{x_1, \dots, x_n\}$ such that $\|x_i\|_p \le b, \forall 1\le i \le n$, for all $\epsilon>0$, we have:
    \begin{align}
        \log\mathcal{N}\bigRound{\mathcal{L}, \epsilon, L_\infty(S)} &\le \frac{36(p-1)a^2b^2}{\epsilon^2}\log\biggRound{
            \biggRound{
                \frac{8ab}{\epsilon} + 7
            }n
        }.
    \end{align}
\end{proposition}

\begin{proposition}[\citealt{article:ledent2021normbased}\label{prop:corollary_of_ledent_proposition_6}]
    Let $a, b\in\R$ be real positive constants. Define the class $\mathcal{L}$ as follows:
    \begin{align*}
        \mathcal{L} = \bigCurl{
            x \mapsto Ax: x\in\R^d, A\in \R^{m\times d}, \|A^\top\|_{2,1} \le a
        },
    \end{align*}
    
    \noindent where $m\ge2$ ($\mathcal{L}$ is a class of vector-valued functions). Then, for a given dataset $S=\{x_1, \dots, x_n\}\in\X^n$ such that $\|x_i\|_2\le b, \forall 1 \le i \le n$, for all $\epsilon > 0$, we have 
    \begin{align}
        \log\mathcal{N}\bigRound{\mathcal{L}, \epsilon, L_{\infty, 2}(S)} \le \frac{64a^2b^2}{\epsilon^2}\log\biggRound{
            \biggRound{
                \frac{11ab}{\epsilon} + 7
            }nm
        }.
    \end{align}
\end{proposition}

\begin{proof}
    Let $\varepsilon>0$ be a constant chosen according to $\epsilon$. Let $a_1, \dots, a_m$ be a sequence of positive real numbers such that $\sum_{u=1}^m a_u\le a$. Let $\varepsilon_1, \dots, \varepsilon_m$ be a sequence of positive numbers depending on $a_1, \dots, a_m$ and $\varepsilon$. Then, we construct a sequence of $\epsilon_u$-covers $\mathcal{C}_1, \dots, \mathcal{C}_m$ with respect to $L_\infty$ metric for the following classes:
    \begin{align*}
        \mathcal{L}_u = \bigCurl{
            x\mapsto\A{u}x: x, \A{u} \in \R^d, \|\A{u}\|_2 \le a_i
        }.
    \end{align*}

    \noindent By proposition \ref{prop:zhang_linfty_covering_number}, we have:
    \begin{align*}
        \log |\mathcal{C}_u| &\le \frac{36a_u^2b^2}{\varepsilon_u^2}\log\biggRound{
            \biggRound{
                \frac{8a_ub}{\varepsilon_u} + 7
            }n
        }.
    \end{align*}

    \noindent Taking the Cartesian product of the covers $\mathcal{C}_1, \dots, \mathcal{C}_m$, we obtain the following new cover:
    \begin{align*}
        \mathcal{C}_{a_1, \dots, a_m} &= \mathcal{C}_1 \times \mathcal{C}_2 \times \dots \times \mathcal{C}_m. \\
        \implies \log |\mathcal{C}_{a_1, \dots, a_m}| &\le 36b^2\sum_{u=1}^m \frac{a_u^2}{\varepsilon_u^2}\log  \biggRound{
            \biggRound{
                \frac{8a_ub}{\varepsilon_u} + 7
            }n
        }.
    \end{align*}

    \noindent Then, for any $A\in\R^{m\times d}$ such that $\|A^\top\|_{2,1}\le a$, we can choose $\bar A\in\mathcal{C}_{a_1, \dots, a_m}$ such that $\max_{1\le i \le n}|(\A{u} - \bA{u})x_i| \le \varepsilon_u$ where $\A{u}$ abd $\bA{u}$ are the $u^{th}$ rows of $A$ and $\bar A$ respectively. Hence, we have:
    \begin{align*}
        \max_{1\le i \le n} \|(A - \bar A)x_i\|_2^2 &\le \max_{1\le i \le n}\sum_{u=1}^m |(\A{u} - \bA{u})x_i|^2 \le \sum_{u=1}^m \varepsilon_u^2.
    \end{align*}

    \noindent Therefore, $\mathcal{C}_{a_1, \dots, a_m}$ is an $\sqrt{\sum_{u=1}^m\varepsilon_u^2}$-cover of the class $\mathcal{L}_{a_1, \dots, a_m} = \bigCurl{x \mapsto Ax: x\in\R^d, A\in \R^{m\times d}, \|\A{u}\|_2 \le a_u}$ (restricted to dataset $S$) with respect to the $L_{\infty, 2}$ metric. To determine the values of $\varepsilon_1, \dots, \varepsilon_m$ such that the above cover has a desired granularity $\varepsilon$, we solve the following optimization problem (ignoring logarithmic terms):
    \begin{align*}
        f(\varepsilon_1, \dots, \varepsilon_m) = \sum_{u=1}^m \frac{a_u^2}{\varepsilon_u^2}, \text{ subjected to } \sum_{u=1}^m \varepsilon_u^2 = \varepsilon^2.
    \end{align*}

    \noindent Let $\lambda$ be the Lagrange multiplier, we obtain the following optimality conditions:
    \begin{align*}
        \bigRound{
            a_1^2/\varepsilon_1^3, \dots, a_m^2/\varepsilon_m^3
        }^\top = \lambda\cdot\bigRound{
            \varepsilon_1, \dots, \varepsilon_m
        }^\top.
    \end{align*}

    \noindent Setting $\lambda = a^2 / \varepsilon^4$, we obtain the optimal values of $\varepsilon_u = \varepsilon\sqrt{{a_u}/{a}}$. Plugging the values of $\varepsilon_u$ back to the upper bound of $\log |\mathcal{C}_{a_1, \dots, a_m}|$, we have:
    \begin{align*}
        \log |\mathcal{C}_{a_1, \dots, a_m}| &\le \frac{36b^2a}{\varepsilon^2}\sum_{u=1}^m a_u \log  \biggRound{
            \biggRound{
                \frac{8a_ub}{\varepsilon_u} + 7
            }n
        } \\
        &= \frac{36b^2a}{\varepsilon^2}\sum_{u=1}^m a_u \log  \biggRound{
            \biggRound{
                \frac{8b\sqrt{a_ua}}{\varepsilon} + 7
            }n
        } \\
        &\le \frac{36b^2a}{\varepsilon^2}\log  \biggRound{
            \biggRound{
                \frac{8ba}{\varepsilon} + 7
            }n
        }\sum_{u=1}^m a_u  \\
        &\le \frac{36b^2a^2}{\varepsilon^2}\log  \biggRound{
            \biggRound{
                \frac{8ba}{\varepsilon} + 7
            }n
        }.
    \end{align*}

    \noindent To extend the cover to apply for any sequence of $a_1, \dots, a_m$ such that $\sum_{u=1}^m |a_u| \le a$, we construct the cover $\mathcal{D}$ with respect to the $\ell^2$ Euclidean metric for the following class:
    \begin{align*}
        \mathcal{B}_a = \biggCurl{
            \boldsymbol{a} \in \R^m : \|\boldsymbol{a}\|_1 = \sum_{u=1}^m |a_u| \le a
        }.
    \end{align*}

    \noindent By lemma \ref{lem:l1_covering_number_for_ball}, we have:
    \begin{align*}
        \log|\mathcal{D}| = \log\mathcal{N}\bigRound{\mathcal{B}_a, \varepsilon', \|.\|_2} &\le \Bigg\lceil \frac{a^2}{\varepsilon'^2} \Bigg\rceil\log(2m) \le 2\Bigg\lceil \frac{a^2}{\varepsilon'^2} \Bigg\rceil\log(m).
    \end{align*}

    \noindent The above inequality comes from the fact that $m\ge2$. Finally, we construct the cover $\mathcal{C}$ for $\mathcal{L}$ by taking the union of all covers $\mathcal{C}_{a_1, \dots, a_m}$ over all sequences $a_1, \dots, a_m$ in $\mathcal{D}$:
    \begin{align*}
        \mathcal{C} = \bigcup_{a_1, \dots, a_m\in\mathcal{D}} \mathcal{C}_{a_1, \dots, a_m} &\implies |\mathcal{C}| \le |\mathcal{D}|\cdot\sup_{a_1,\dots,a_m\in\mathcal{D}} |\mathcal{C}_{a_1, \dots, a_m}|. \\
        &\implies \log|\mathcal{C}| \le \log|\mathcal{D}| + \sup_{a_1, \dots, a_m} \log|\mathcal{C}_{a_1, \dots, a_m}|.
    \end{align*}

    \noindent By the construction of $\mathcal{C}$, for all $A\in\R^{m\times d}$ such that $\|A\|_{2,1}\le a$, we can select a cover element $\bar A$ as follows:
    \begin{itemize}
        \item From $\mathcal{D}$, select a sequence $a_1, \dots, a_m$ closest to $(\|\A{1}\|_2, \dots, \|\A{m}\|_2)$ in terms of $\ell^2$ Euclidean metric.
        \item Create a new matrix $\tilde A \in \R^{m\times d}$ defined as follows:
        \begin{align*}
            \tA{u} = \begin{cases}
                \A{u}, &\text{ if } \|\A{u}\|_2 \le a_u \\
                \frac{a_u}{\|\A{u}\|_2}\A{u}, &\text{ otherwise }
            \end{cases}.
        \end{align*}

        \item Since $\tilde A$ belongs to the subclass $\mathcal{L}_{a_1, \dots, a_m}$, we can choose an element cover $\bar A \in \mathcal{C}_{a_1, \dots, a_m}$ closest to $\tilde A$ in terms of the $L_{\infty, 2}$ metric.
    \end{itemize}

    \noindent By the above selection, we have:
    \begin{align*}
        \max_{1\le i \le n}\|Ax_i - \bar Ax_i\|_2 &\le \max_{1\le i \le n}\|(A - \tilde A)x_i\|_2 + \max_{1\le i \le n}\|(\tilde A - \bar A)x_i\|_2 \\
        &\le \max_{1\le i \le n}\|A - \tilde A\|_F\cdot\|x_i\|_2 + \varepsilon \\
        &\le b\|A - \tilde A\|_F + \varepsilon \\
        &\le b\sqrt{\sum_{u=1}^m \|\A{u} - \tA{u}\|_2^2} + \varepsilon \\
        &\le b\sqrt{\sum_{u=1}^m \Bigg\| \A{u}\biggAbs{1 - \frac{a_u}{\|\A{u}\|_2}} \Bigg\|_2^2} + \varepsilon \\
        &\le b\sqrt{\sum_{u=1}^m |a_u - \|\A{u}\|_2|^2} + \varepsilon \\
        &\le b\varepsilon' + \varepsilon.
    \end{align*}

    \noindent To make the cover $\mathcal{C}$ have a desired granularity $\epsilon>0$, we set $\varepsilon'=\epsilon/4b$ and $\varepsilon = 3\epsilon/4$. Plugging the values of $\varepsilon', \varepsilon$ to the bound on cardinalities of $\mathcal{D}$ and $\mathcal{C}_{a_1, \dots, a_m}, \forall a_1, \dots, a_m\in\mathcal{D}$, we have:
    \begin{align*}
        \log|\mathcal{D}| &\le 2\Bigg\lceil \frac{16a^2b^2}{\epsilon^2} \Bigg\rceil\log(m) \le \frac{64a^2b^2}{\epsilon^2}\log(m), \\ 
        \log |\mathcal{C}_{a_1, \dots, a_m}| &\le  \frac{64a^2b^2}{\epsilon^2}\log\biggRound{
            \biggRound{
                \frac{32ab}{3\epsilon} + 7
            }n
        } \le \frac{64a^2b^2}{\epsilon^2}\log\biggRound{
            \biggRound{
                \frac{11ab}{\epsilon} + 7
            }n
        }.
    \end{align*}

    \noindent For the first inequality, without loss of generality, we assume that $16a^2b^2/\epsilon^2\ge1$. Hence, $\lceil 16a^2b^2/\epsilon^2\rceil \le 32a^2b^2/\epsilon^2$. Finally, we have:
    \begin{align*}
        \log|\mathcal{C}| &\le \log|\mathcal{D}| + \sup_{a_1, \dots, a_m}\log|\mathcal{C}_{a_1, \dots, a_m}| \\
        &\le \frac{64a^2b^2}{\epsilon^2}\log(m) + \frac{64a^2b^2}{\epsilon^2}\log\biggRound{
            \biggRound{
                \frac{11ab}{\epsilon} + 7
            }n
        } \\
        &= \frac{64a^2b^2}{\epsilon^2}\log\biggRound{
            \biggRound{
                \frac{11ab}{\epsilon} + 7
            }nm
        }.
    \end{align*}

    \noindent Hence, we obtained the desired bound.
\end{proof}

\begin{lemma}[\citealt{article:pisier1981maurey, book:bartlett1998svm, book:bartlett2002nnfoundation, article:ledent2021normbased}\label{lem:l1_covering_number_for_ball}]
    Let $\mathcal{B}_\beta$ denote a ball of radius $\beta$ in $\R^d$ with respect to the Euclidean $\ell^1$ norm. Then for any $\epsilon>0$, we have:
    \begin{align}
        \log\mathcal{N}\bigRound{\mathcal{B}_\kappa, \epsilon, \|.\|_2} &\le \Bigg\lceil 
            \frac{\beta^2}{\epsilon^2} 
        \Bigg\rceil\log(2d).
    \end{align} 
\end{lemma}

\begin{proof}
    Without loss of generality, let $\beta=1$. We prove that for all $\boldsymbol{a}=(a_1, \dots, a_d)\in\R^d_+$ and $k\in\mathbb{Z}_+$ such that $\sum_{i=1}^da_i \le 1$, there exists $(k_1, \dots, k_d)$ where $k_i \in\mathbb{N}$ such that $\bar a = \sum_{i=1}^de_i\frac{k_i}{k}$ ($e_i$ are standard basis in $\R^d$) and:
    \begin{align*}
        \|a - \bar a\|_2^2 \le \frac{1}{k}.
    \end{align*}

    \noindent Define $W_1, \dots, W_k$ as i.i.d discrete random vectors taking the values of standard basis in $\R^d$ such that:
    \begin{align*}
        \forall 1\le j \le k: P(W_j = e_i) &= a_i.
    \end{align*}

    \noindent Now, define $W = k^{-1}\sum_{j=1}^kW_j$. Then, we have $\E[W] = \E[W_1] = \sum_{j=1}^k a_ie_i=\boldsymbol{a}$. Then, we have:
    \begin{align*}
        \E\bigSquare{\|W-a\|_2^2} &= \frac{1}{k^2}\E\biggSquare{
            \sum_{j=1}^k \|W_j -a\|_2^2 + \sum_{l\ne j}\big<a-W_j, a-W_l\big>
        }\\
        &= \frac{1}{k^2}\E\biggSquare{
            \sum_{j=1}^k \|W_j -a\|_2^2
        } = \frac{1}{k}\E\bigSquare{\|W_1 - a\|_2^2} \\
        &= \frac{1}{k}\bigRound{
            \E\bigSquare{\|W_1\|_2^2} - \|a\|_2^2
        } \\
        &\le \frac{1}{k}\E\bigSquare{\|W_1\|_2^2}  = \frac{1}{k}.
    \end{align*}

    \noindent We can rewrite $W = \sum_{i=1}^d\frac{\sum_{j=1}^k\1{W_j = e_i}}{k}e_i$. Hence, by probabilistic method, we can choose integer solutions $(k_1, \dots, k_d)$ (where $k_i = \sum_{j=1}^k\1{W_j = e_i}$) such that:
    \begin{align*}
        \|a - \bar a\|_2^2 &\le \E\bigSquare{\|W-a\|_2^2} \le \frac{1}{k}.
    \end{align*}

    \noindent Setting $k=\lceil \beta^2/\epsilon^2 \rceil$, it follows that we can find integer solutions such that $\|a - \bar a\|_2^2 \le k^{-1} \le \epsilon^2$. Hence, the $\epsilon$ covering number with respect to the Euclidean $\ell^2$ metric of $\mathcal{B}_\beta$ is equivalent to the number of possible solutions $(k_1, \dots, k_d)$ such that $\sum_{i=1}^d k_i = k$. Hence, we have:
    \begin{align*}
        \mathcal{N}\bigRound{\mathcal{B}_\beta, \epsilon, \|.\|_2} = \begin{pmatrix}
            d+k-1 \\
            d-1
        \end{pmatrix} \le (2d)^k.
    \end{align*}
    \noindent Taking logarithm from both sides, we obtain the desired bound.
\end{proof}

\section{Basic Bound}
\label{app:basic_bound}
In this section, we present the proof of our basic bound from Theorem~\ref{thm:basic_bound}. The key idea to the proof technique is to compose the covering numbers through the computational graph. In particular, one must first relate the covering numbers of the loss class $\G$ to that of the class $\F_\mathcal{A}$ on an auxiliary dataset containing all positive and negative samples, and then propagate the error through all the layers and the multiplication operation in the unsupervised loss. This last step is where the additional product of spectral norms appears. 

\noindent As explained in the main text, we will find the covering number for the loss function class $\G$ restricted to the original dataset $\Sds$ of tuples of $k+2$ vectors:
\begin{align*}
    \Sds = \bigCurl{
        (x_1, x_1^+, x_{11}^-, \dots, x_{1k}^-), \dots, (x_n, x_n^+, x_{n1}^-, \dots, x_{nk}^-) 
    }.
\end{align*}

\noindent Before deriving the covering number bound for $\G$, it is important that we derive the covering number bound for the following simpler function class:
\begin{align}\label{eq:test}
    \Hf = \biggCurl{
        h_\pA : h_\pA(x, x^+, x^-) = \f(x)^\top\bigRound{\f(x^+) - \f(x^-)}, \ \f\in\F_\mathcal{A}
    },
\end{align}

\noindent where the class of neural networks $\mathcal{F_A}$ is defined in the main text. We would like to control the covering number of $\Hf$ restricted to the following dataset $\Sds_1$ of size $nk$ induced by the original dataset $\Sds$:
\begin{equation}
    \label{eq:s_prime}
    \begin{aligned}
    \Sds_1 = \bigCurl{
        \underbrace{(x_1, x_1^+, x_{11}^-), (x_1, x_1^+, x_{12}^-), \dots, (x_1, x_1^+, x_{1k}^-)}_{\text{induced by first tuple in $\Sds$}}, 
        \dots,\underbrace{(x_n, x_n^+, x_{n1}^-), (x_n, x_n^+, x_{n2}^-), \dots, (x_n, x_n^+, x_{nk}^-)}_{\text{induced by last tuple in $\Sds$}}
    }.
    \end{aligned}
\end{equation}

\noindent To do so, we first bound the covering number of $\F_\mathcal{A}$ restricted to the following dataset of size $n(k+2)$:
\begin{align}
    \label{eq:s_doubleprime}
    \Sds_2 &= \bigCurl{
        \underbrace{x_1, x_1^+, x_{11}^-, \dots, x_{1k}^-}_{\text{Induced by first tuple in $\Sds$}},
        \underbrace{x_2, x_2^+, x_{21}^-, \dots, x_{2k}^-}_{\text{Induced by second tuple in $\Sds$}}, 
        \dots, \underbrace{x_n, x_n^+, x_{n1}^-, \dots, x_{nk}^-}_{\text{Induced by last tuple in $\Sds$}}
    }.
\end{align}

\subsection{Covering Number for Neural Networks}
\begin{lemma}
\label{lem:covering_number_of_VA_general}
    \footnote{Similar result can be found in \cite{article:bartlett2017spectrallynormalized} and \cite{article:wei2020datadependent} for covering number of neural networks. However, these results are proven for $L_2$ metric.} Given a dataset $S=\{x_1, \dots, x_n\}\in\X^n$ and let $\epsilon_1, \dots, \epsilon_L$ be known positive constants. Let $\FA$ be the class of neural networks defined in the main text. We have:
    \begin{align}
        \log\mathcal{N}\biggRound{
            {\F_\mathcal{A}}, \sum_{l=1}^L \epsilon_l \rho_{l+}, L_{\infty, 2}(S)
        } \le \sum_{l=1}^L \log\mathcal{N}_{\infty, 2}\bigRound{\mathcal{V}_l, \epsilon_l, n},
    \end{align}

    \noindent where we define $\rho_{l+}=\rho_l\prod_{m=l+1}^L \rho_ms_m$ and the worst-case covering number $\mathcal{N}_{\infty, 2}(\mathcal{V}_l, \epsilon_l, n)$ of each linear layer $\mathcal{V}_l$ with respect to the $L_{\infty, 2}$ metric as follows:
    \begin{align}
        \mathcal{N}_{\infty, 2}\bigRound{\mathcal{V}_l, \epsilon_l, n} &= \sup_{
            \bar S\in\X_{l-1}^n
        } \mathcal{N}\bigRound{{\mathcal{V}_l}, \epsilon_l, L_{\infty, 2}(\bar S)},
    \end{align}

    \noindent where for $1 \le l \le L-1$, we define the spaces $\X_{l}\subset\R^{d_l}$ as $\X_l=\bigCurl{\f^{1\to l}(x): \f\in\FA, x\in\X}$ and $\X_0=\X$.
\end{lemma}

\begin{proof}
    Proving the above lemma by induction, we have:
    \noindent\newline\textbf{1. Base case}: For $L=1$, we assume that the class of neural networks consists of only one layer. Therefore, we have the neural network class $\F_\mathcal{A} = \bigCurl{z\mapsto \sigma_l\bigRound{\A{1}z}: \A{1}\in\mathcal{B}_1}$. Suppose that we have constructed an (internal) $\epsilon_1$-cover, $\mathcal{C}_1$, of the linear class $\mathcal{V}_1$. Then, for all $\A{1}\in\mathcal{B}_1$, there exists $\bA{1}\in\mathcal{C}_1\subset\mathcal{B}_1$ such that:
    \begin{align*}
        \max_{1 \le i \le n}\Big\|
            \sigma_1\bigRound{\A{1}x_i} - \sigma_1\bigRound{\bA{1}x_i}
        \Big\|_2 &\le \rho_1 \max_{1\le i \le n} \Big\|
            \A{1}x_i - \bA{1}x_i
        \Big\|_2 \le \rho_1\epsilon_1.
    \end{align*}

    \noindent Hence, we have $\mathcal{N}({\F_\mathcal{A}}, \rho_1\epsilon_1, L_{\infty, 2}(S)) \le \mathcal{N}({\mathcal{V}_1}, \epsilon_1, L_{\infty, 2}(S)) \le \mathcal{N}_{\infty, 2}(\mathcal{V}_1, \epsilon_1, n)$.

    \noindent\newline\textbf{2. Inductive case}: For $L>1$ and $1 \le l < L$, suppose that we have constructed the cover $\mathcal{C}_{1\to l}\subset \mathcal{B}_l \times \dots \times \mathcal{B}_1$ (which means that $\mathcal{C}_{1\to l}$ is an internal cover) for the composition of the first $l$ layers such that the following are satisfied:
    \begin{itemize}
        \item For all $\pA^{1\to l}=(\A{l}, \dots, \A{1}) \in\mathcal{B}_l \times \dots \times \mathcal{B}_1$, there exists $\bpA^{1\to l}\in\mathcal{C}_{1\to l}$ such that:
        \begin{align*}
            \max_{1\le i \le n}\|\ff{1}{l}(x_i) - \bff{1}{l}(x_i)\|_2 \le \sum_{m=1}^l \epsilon_m\rho_{m\to l},
        \end{align*}
        \noindent where $\rho_{m\to l} = \rho_m \prod_{u=m+1}^l \rho_us_u$.

        \item $|\mathcal{C}_{1\to l}| \le \prod_{m=1}^l \mathcal{N}_{\infty, 2}\bigRound{\mathcal{V}_m, \epsilon_m, n}$.
    \end{itemize}

    \noindent We construct the cover for $\mathcal{B}_{l+1}\times\mathcal{B}_l\times\dots\times\mathcal{B}_1$ as follows:
    \begin{itemize}
        \item For each cover element $\bpA^{1\to l}=(\bA{l}, \dots, \bA{1})\in\mathcal{C}_{1\to l}$, we construct the (internal) $\epsilon_{l+1}$-cover $\mathcal{C}_{l+1}(\bpA^{1\to l})\subset\mathcal{B}_{l+1}$ for the linear class of the next layer:
        \begin{align*}
            \mathcal{V}_{l+1} = \bigCurl{
                z \mapsto \A{l+1}z : \A{l+1}\in\mathcal{B}_{l+1}, z \in \R^{d_l}
            },
        \end{align*}
        \noindent restricted to the auxiliary dataset $S_{\bpA^{1\to l}}=\bigCurl{\bff{1}{l}(x_i): x_i\in S}$, which is the set of outputs of the neural network corresponding to the cover element $\bpA^{1\to l}$.

        \item Form the next cover $\mathcal{C}_{1\to l+1}$ by taking the following union:
        \begin{align*}
            \mathcal{C}_{1\to l+1} = \bigcup_{\bpA^{1\to l} \in \mathcal{C}_{1\to l}}\bigCurl{
                \bA{l+1}\times\bpA^{1\to l} : \bA{l+1}\in\mathcal{C}_{l+1}(\bpA^{1\to l})
            },
        \end{align*}
        \noindent where $\bA{l+1}\times\bpA^{1\to l} = (\bA{l+1}, \bA{l}, \dots, \bA{1})$.
    \end{itemize}

    \noindent From the above construction, we have $\mathcal{C}_{1\to l+1}\subset\mathcal{B}_{l+1}\times\dots\times\mathcal{B}_1$ (an internal cover of the composition of the first $l+1$ layers). Now, we can bound the cardinality of $\mathcal{C}_{1\to l+1}$ as follows:
    \begin{align*}
        \bigAbs{\mathcal{C}_{1\to l+1}} &\le \sum_{\bpA^{1\to l}\in \mathcal{C}_{1\to l}} \bigAbs{ \mathcal{C}_{l+1}(\bpA^{1\to l}) } \\
        &\le \bigAbs{\mathcal{C}_{1\to l}}\cdot\sup_{\bpA^{1\to l}\in \mathcal{C}_{1\to l}} \mathcal{N}\bigRound{
            \mathcal{V}_{l+1}, \epsilon_{l+1}, L_{\infty, 2}(S_{\bpA^{1\to l}})
        } \\
        &\le \bigAbs{\mathcal{C}_{1\to l}}\cdot \mathcal{N}_{\infty, 2}\bigRound{\mathcal{V}_{l+1}, \epsilon_{l+1}, n} \\ 
        &\le \prod_{m=1}^{l+1} \mathcal{N}_{\infty, 2}\bigRound{\mathcal{V}_m, \epsilon_m, n}.
    \end{align*}

    \noindent Furthermore, for any $\pA^{1\to l+1}=\A{l+1}\times\pA^{1\to l}\in\mathcal{B}_{l+1}\times\dots\times\mathcal{B}_1$ and its corresponding neural network $\ff{1}{l+1}(x)=\sigma_{l+1}\bigRound{\A{l+1}\ff{1}{l}(x)}$, there exists $\bpA^{1\to l+1}\in\mathcal{C}_{1\to l+1}$ chosen as follows:
    \begin{itemize}
        \item $\bff{1}{l+1}(x)=\sigma_{l+1}\bigRound{\bA{l+1}\bff{1}{l}(x)}$.
        
        \item $\bpA^{1\to l}\in\mathcal{C}_{1\to l}$ is chosen as the closest cover element to $\pA^{1\to l}$ in $L_{\infty, 2}$ metric. Therefore:
        \begin{align*}
            \max_{1\le i \le n}\Big\|\ff{1}{l}(x_i) - \bff{1}{l}(x_i)\Big\|_2 \le \sum_{m=1}^l \epsilon_m \rho_{m\to l}.
        \end{align*}
        
        \item $\bA{l+1}\in \mathcal{C}_{l+1}(\bpA^{1 \to l})$ is chosen as the closest cover element to $\A{l+1}$ in $L_{\infty, 2}$ metric. Therefore:
        \begin{align*}
            \max_{1\le i \le n}\Big\|\bigRound{\A{l+1} - \bA{l+1}}\bff{1}{l}(x_i)\Big\|_2 \le \epsilon_{l+1}.
        \end{align*}
    \end{itemize}

    \noindent Then, we have:
    \begin{align*}
        \max_{1\le i \le n}\Big\|\ff{1}{l+1}&(x_i) - \bff{1}{l+1}(x_i)\Big\|_2 \\ 
        &\le \rho_{l+1}\max_{1\le i \le n}\Big\|
            \A{l+1}\ff{1}{l}(x_i) - \bA{l+1}\bff{1}{l}(x_i)
        \Big\|_2 \\
        &\le \rho_{l+1}\max_{1\le i \le n}\bigRound{
            \Big\|\A{l+1}\bigRound{\ff{1}{l}(x_i) - \bff{1}{l}(x_i)}\Big\|_2 + \Big\| 
                \bigRound{\A{l+1}-\bA{l+1}}\bff{1}{l}(x_i)
            \Big\|_2
        } \\
        &\le \rho_{l+1}\max_{1\le i \le n}\Big\|\A{l+1}\bigRound{\ff{1}{l}(x_i) - \bff{1}{l}(x_i)}\Big\|_2 + \rho_{l+1}\max_{1\le i \le n}\Big\| 
                \bigRound{\A{l+1}-\bA{l+1}}\bff{1}{l}(x_i)
            \Big\|_2 \\
        &\le \rho_{l+1}s_{l+1}\max_{1\le i \le n}\Big\| \ff{1}{l}(x_i) - \bff{1}{l}(x_i) \Big\|_2 + \rho_{l+1}\epsilon_{l+1}
         \\
        &\le \rho_{l+1}s_{l+1}\sum_{m=1}^l \epsilon_m\rho_{m\to l} + \rho_{l+1}\epsilon_{l+1} = \sum_{m=1}^{l+1}\epsilon_m\rho_{m\to l+1}.
    \end{align*}

    \noindent Recursively construct the cover until the $L^{th}$ layer, we have:
    \begin{align*}
        \log\mathcal{N}\biggRound{
            {\F_\mathcal{A}}, \sum_{l=1}^L \epsilon_l \rho_{l+}, L_{\infty, 2}(S)
        } &= \log\mathcal{N}\biggRound{
            \F_{L}\circ\dots\circ\F_1, \sum_{l=1}^L \epsilon_l \rho_{l+}, L_{\infty, 2}(S)
        }\\ 
        &\le \sum_{l=1}^L \log\mathcal{N}_{\infty, 2}\bigRound{\mathcal{V}_l, \epsilon_l, n},
    \end{align*}

    \noindent as desired.
\end{proof}

\begin{proposition}{(Covering number for $\FA$).}
\label{prop:covering_number_of_FA_specific}
    Given an arbitrary dataset $S=\{x_1, \dots, x_n\}\in \X^n$. Let $\FA$ be the class of neural networks defined in the main text, we have:
    \begin{align}
        \log\mathcal{N}\bigRound{
            {\F_\mathcal{A}}, \epsilon, {L_{\infty, 2}(S)}
        } \le \frac{64\mathcal{\bar R_A}^2}{\epsilon^2}\log\biggRound{
            \biggRound{
                \frac{11\mathcal{\bar R_A}}{\epsilon} + 7
            }nW
        }.
    \end{align}

    \noindent Where we have:
    \begin{itemize}
        \item $W=\max_{1\le l \le L}d_l$.
        \item $\mathcal{\bar{R}_A}^{2/3} = \sum_{l=1}^L \bigRound{a_lB_{l-1}\rho_{l+}}^{2/3}$.
        \item $\rho_{l+} = \rho_l\prod_{m=l+1}^L \rho_ms_m$.
        \item $B_l = \sup_{x\in\X}\sup_{\pA\in\mathcal{A}} \|\f^{1\to l}(x)\|_2$.
    \end{itemize}
\end{proposition}

\begin{proof}
    From lemma \ref{lem:covering_number_of_VA_general} and proposition \ref{prop:corollary_of_ledent_proposition_6}, let $\epsilon_1, \dots, \epsilon_L$ be positive real constants, we have:
    \begin{align*}
        \log\biggRound{
            {\F_\mathcal{A}}, \sum_{l=1}^L\epsilon_l\rho_{l+}, L_{\infty, 2}(S)
        } 
        &\le \sum_{l=1}^L \log\mathcal{N}_{\infty, 2}\bigRound{\mathcal{V}_l, \epsilon_l, n} \\
        &\le \sum_{l=1}^L \frac{64a^2_lB^2_{l-1}}{\epsilon_l^2}\log\biggRound{
            \biggRound{
                \frac{11a_lB_{l-1}}{\epsilon_l} + 7
            }nd_l
        } \\
        &\le \log\bigRound{
            \bigRound{
                11\max_{1\le l \le L} a_lB_{l-1}/\epsilon_l + 7
            }nW
        }\sum_{l=1}^L\frac{64a^2_lB^2_{l-1}}{\epsilon_l^2},
    \end{align*}
    where $W = \max_{1\le l \le L}d_l$. We can cover the class $\F_\mathcal{A}$ with a desired granularity $\epsilon>0$ with respect to the $L_{\infty, 2}$ metric by tweaking the individual granularities $\epsilon_1, \dots, \epsilon_L$. For $1 \le l \le L$, we set:
    \begin{align*}
        \epsilon_l = \frac{\beta_l\epsilon}{\rho_{l+}}, \ \rho_{l+} = \rho_l \prod_{m=l+1}^L \rho_ms_m,
    \end{align*}

    \noindent where $\{\beta_l\}_{l=1}^L$ is a sequence of weights such that $\sum_{l=1}^L\beta_l = 1$. Therefore, the aggregated granularity in lemma \ref{lem:covering_number_of_VA_general} becomes $\sum_{l=1}^L \epsilon_l \rho_{l+} = \sum_{l=1}^L \beta_l \epsilon = \epsilon$. Plugging the values of $\epsilon_l$ into the bound obtained from above, we have:
    \begin{align*}
        \log\mathcal{N}\bigRound{
            {\F_\mathcal{A}}, \epsilon, {L_{\infty, 2}(S)}
        } \le \log\bigRound{
            \bigRound{
                11\max_{1\le l \le L} a_lB_{l-1}/\epsilon_l + 7
            }nW
        }\sum_{l=1}^L \frac{64a^2_lB^2_{l-1}\rho_{l+}^2}{\beta_l^2\epsilon^2}.
    \end{align*}

    \noindent Ignoring the logarithm terms, we obtain the tightest covering number bound by solving the following optimization problem over the choices of $\beta_1, \dots, \beta_L$:
    \begin{align*}
        f(\beta_1, \dots, \beta_L) = \sum_{l=1}^L \frac{C_{l}^2}{\beta_l^2}, \ \text{subjected to: } \sum_{l=1}^L \beta_l = 1,
    \end{align*}

    \noindent where we set $C_l = a_lB_{l-1}\rho_{l+}$ for $1 \le l \le L$. Let $\lambda$ be the Lagrange multiplier, we obtain the following optimality conditions:
    \begin{align*}
        \bigRound{
            -2C_{1}^2\beta_1^{-3}, 
            \dots, 
            -2C_{L}^2\beta_L^{-3}
        }^\top = \lambda \cdot \bigRound{
            1, \dots, 1
        }^\top.
    \end{align*}

    \noindent Setting $\lambda = -2\Big[\sum_{m=1}^L C_{m}^{2/3}\Big]^3$, we have:
    \begin{align*}
        \forall 1 \le l \le L : \frac{C_{l}^2}{\beta_l^{3}} = \biggSquare{\sum_{m=1}^L C_{m}^{2/3}}^3 \implies \beta_l = \frac{C_{l}^{2/3}}{\sum_{m=1}^L C_{m}^{2/3}}.
    \end{align*}

    \noindent Plugging the values of $\beta_l$ back to the non-logarithm parts of the covering number bound, we have:
    \begin{align*}
        \sum_{l=1}^L \frac{64a^2_lB^2_{l-1}\rho_{l+}^2}{\beta_l^2\epsilon^2}
        &= 64\biggSquare{\sum_{l=1}^L (a_lB_{l-1}\rho_{l+})^{2/3}}^2\sum_{l=1}^L\frac{(a_lB_{l-1}\rho_{l+})^{2/3}}{\epsilon^2} \\
        &= \frac{64}{\epsilon^2}\biggSquare{\sum_{l=1}^L (a_lB_{l-1}\rho_{l+})^{2/3}}^3 \\
        &= \frac{64\mathcal{\bar R_A}^2}{\epsilon^2}.
    \end{align*}

    \noindent Plugging the values of $\beta_l$ into $\max_{1\le l \le L} a_lB_{l-1}/\epsilon_l$, we have:
    \begin{align*}
        \max_{1\le l \le L} a_lB_{l-1}/\epsilon_l &= \max_{1\le l \le L}\frac{a_lB_{l-1}\rho_{l+}}{\beta_l \epsilon} \\
        &= \sum_{m=1}^L (a_mB_{m-1}\rho_{m+})^{2/3}\max_{1\le l \le L}\frac{(a_lB_{l-1}\rho_{l+})^{1/3}}{\epsilon} \\
        &= \frac{\mathcal{\bar R_A}^{2/3}}{\epsilon}\max_{1\le l \le L}(a_lB_{l-1}\rho_{l+})^{1/3} \\
        &\le \frac{\mathcal{\bar R_A}^{2/3}}{\epsilon}\cdot \mathcal{\bar R_A}^{1/3} = \frac{\mathcal{\bar R_A}}{\epsilon}.
    \end{align*}

    \noindent Combining them all together, we have:
    \begin{align*}
        \log\mathcal{N}\bigRound{
            {\F_\mathcal{A}}, \epsilon, {L_{\infty, 2}(S)}
        } \le \frac{64\mathcal{\bar R_A}^2}{\epsilon^2}\log\biggRound{
            \biggRound{
                \frac{11\mathcal{\bar R_A}}{\epsilon} + 7
            }nW
        }.
    \end{align*}
\end{proof}

\begin{proposition}{(Covering number for $\Hf$).}
    \label{prop:covering_number_for_hf}
    Let $\Sds_1$ be the collection of triplets defined in equation \ref{eq:s_prime}. We have the following covering number bound of $\Hf$ with respect to the $L_\infty$ metric:
    \begin{align}
        \log\mathcal{N}\bigRound{\Hf, \epsilon, L_\infty(\Sds_1)} \le \frac{1024B_L^2\mathcal{\bar{R}_A}^2}{\epsilon^2}\log\biggRound{
            \biggRound{
                \frac{44B_L\mathcal{\bar{R}_A}}{\epsilon} + 7
            }n(k+2)W
        }.
    \end{align}
\end{proposition}

\begin{proof}
    Let $\xtrip_{ij}$ define the triplet $(x_j, x_j^+, x_{ji}^-)\in \Sds_1$ for $1\le j \le n$ and $1 \le i \le k$. For neural networks $\f, \barf \in \F_\mathcal{A}$ and their corresponding $h_\pA, h_\bpA\in\Hf$, we have:
    \begin{align*}
        &\bigAbs{
            h_\pA(x_j, x_j^+, x_{ji}^-) - h_{\bpA}(x_j, x_j^+, x_{ji}^-)
        } \\ 
        &= \bigAbs{
            \f(x_j)^\top\bigRound{\f(x_j^+) - \f(x_{ji}^-)} - \barf(x_j)^\top\bigRound{\barf(x_j^+) - \barf(x_{ji}^-)}
        } \\
        &= \bigAbs{
            \bigRound{\f(x_j) - \barf(x_j)}^\top\bigRound{\f(x_j^+)-\f(x_{ji}^-)} + \barf(x_j)^\top\bigRound{\f(x_j^+) - \f(x_{ji}^-) - \barf(x_j^+) + \barf(x_{ji}^-)}
        } \\
        &\le \bigAbs{
            \bigRound{\f(x_j) - \barf(x_j)}^\top\bigRound{\f(x_j^+)-\f(x_{ji}^-)}}+ \bigAbs{\barf(x_j)^\top\bigRound{\f(x_j^+) - \f(x_{ji}^-) - \barf(x_j^+) + \barf(x_{ji}^-)}
        } \\
        &\le 
        \Big\|\f(x_j) - \barf(x_j)\Big\|_2\cdot\Big\|\f(x_j^+) - \f(x_{ji}^-)\Big\|_2 
        + \Big\|\barf(x_j)\Big\|_2\bigRound{\Big\|\f(x_j^+) - \barf(x_j^+)\Big\|_2 
        + \Big\|\f(x_{ji}^-) - \barf(x_{ji}^-)\Big\|_2} \\
        &\le \Big\|\f(x_j) - \barf(x_j)\Big\|_2\bigRound{\Big\|\f(x_j^+)\Big\|_2 + \Big\|\f(x_{ji}^-)\Big\|_2} 
        + \Big\|\barf(x_j)\Big\|_2\bigRound{\Big\|\f(x_j^+) - \barf(x_j^+)\Big\|_2 
        + \Big\|\f(x_{ji}^-) - \barf(x_{ji}^-)\Big\|_2} \\
        &\le 2B_L\Big\|\f(x_j) - \barf(x_j)\Big\|_2 + B_L\bigRound{\Big\|\f(x_j^+) - \barf(x_j^+)\Big\|_2 + \Big\|\f(x_{ji}^-) - \barf(x_{ji}^-)\Big\|_2} \\
        &\le 4B_L \cdot \max_{\tilde x \in \xtrip_{ij}}\Big\| \f(\tilde x) - \barf(\tilde x) \Big\|_2.
    \end{align*}

    \noindent From the above, we have:
    \begin{align*}
        \Big\|h_\pA - h_\bpA \Big\|_{L_\infty(\Sds_1)} &= \max_{\xtrip_{ij}\in \Sds_1}\bigAbs{
            h_\pA(x_j, x_j^+, x_{ji}^-) - h_\bpA(x_j, x_j^+, x_{ji}^-)
        } \\
        &\le 4B_L\max_{\xtrip_{ij}\in \Sds_1}\max_{\tilde x \in \xtrip_{ij}}\Big\| \f(\tilde x) - \barf(\tilde x) \Big\|_2 \\
        &= 4B_L\max_{\tilde x \in \Sds_2} \Big\| \f(\tilde x) - \barf(\tilde x) \Big\|_2 \\
        &= 4B_L\Big\|\f - \barf \Big\|_{L_{\infty, 2}(\Sds_2)}.
    \end{align*}

    \noindent Hence, if we want to construct a cover for $\Hf$ (restricted to $\Sds_1$) with respect to the $L_\infty$ metric with granularity $\epsilon>0$, we need to construct a cover for $\F_\mathcal{A}$ (restricted to $\Sds_2$) with respect to the $L_{\infty, 2}$ metric with granularity $\epsilon/4B_L$. From proposition \ref{prop:covering_number_of_FA_specific}, we have: 
    \begin{align*}
        \log\mathcal{N}\bigRound{\Hf, \epsilon, L_\infty(\Sds_1)}
        &\le
        \log\mathcal{N}\bigRound{ {\F_\mathcal{A}}, \epsilon/4B_L, {L_{\infty, 2}(\Sds_2)}} \\
        &\le \frac{1024B_L^2\mathcal{\bar{R}_A}^2}{\epsilon^2}\log\biggRound{
            \biggRound{
                \frac{44B_L\mathcal{\bar{R}_A}}{\epsilon} + 7
            }n(k+2)W
        }.
    \end{align*}
\end{proof}

\subsection{Proof of Theorem \ref{thm:basic_bound}}
\begin{proof}[Proof of Theorem \ref{thm:basic_bound}]
    \noindent For $\xin_j=(x_j, x_j^+, x^-_{j1}, \dots, x^-_{jk}) \in \Sds$ and $\f, \barf\in\F_\mathcal{A}$. Define $v_j, \bar{v}_j\in\R^k$ as follows:
    \begin{align*}
        v_j = \begin{pmatrix}
            h_\pA(x_j, x_j^+, x_{j1}^-) \\
            h_\pA(x_j, x_j^+, x_{j2}^-) \\
            \vdots \\
            h_\pA(x_j, x_j^+, x_{jk}^-)
        \end{pmatrix}, \ \ \ \bar{v}_j = \begin{pmatrix}
            h_\bpA(x_j, x_j^+, x_{j1}^-) \\
            h_\bpA(x_j, x_j^+, x_{j2}^-) \\
            \vdots \\
            h_\bpA(x_j, x_j^+, x_{jk}^-)
        \end{pmatrix}.
    \end{align*}

    \noindent Then, letting $g_\pA, g_\bpA\in\G$ be the loss functions evaluated on the neural networks parameterized by $\pA, \bpA$, \footnote{For notational brevity, we denote $g_\pA(x_j, x_j^+, x^-_{j1}, \dots, x^-_{jk})$ as $g_\pA(x_j, x_j^+, x^-_{j1:k})$ for an arbitrary $\pA\in\mathcal{A}$.} we have:
    \begin{align*}
        \Big|g_\pA(x_j, x_j^+, x^-_{j1:k}) - g_\bpA(x_j, x_j^+, x^-_{j1:k})\Big|
        &= \bigAbs{
            \ell(v_j) - \ell(\bar{v}_j)
        } \\
        &\le \eta \cdot\|v_j - \bar{v}_j\|_\infty \\
        &= \eta \cdot\max_{1\le i \le k}\bigAbs{
            h_\pA(x_j, x_j^+, x_{ji}^-) - h_\bpA(x_j, x_j^+, x_{ji}^-)
        } \\
        &\le \eta \cdot\max_{1 \le j \le n} \max_{1\le i \le k}\bigAbs{
            h_\pA(x_j, x_j^+, x_{ji}^-) - h_\bpA(x_j, x_j^+, x_{ji}^-)
        } \\
        &=  \eta \cdot\max_{\xtrip_{ij}\in \Sds_1}\bigAbs{
            h_\pA(x, x^+, x^-) - h_\bpA(x, x^+, x^-)
        } \ \ \ (\xtrip_{ij} = (x_j, x_j^+, x_{ji}^-)) \\
        &=  \eta \cdot\Big\|
            h_\pA - h_\bpA
        \Big\|_{{L_{\infty}(\Sds_1)}}.
    \end{align*}

    \noindent Therefore, for $\epsilon>0$ we have $\mathcal{N}\bigRound{\G, \epsilon, L_2(\Sds)} \le \mathcal{N}\bigRound{\Hf, \epsilon/\eta, L_\infty(\Sds_1)}$. By proposition \ref{prop:covering_number_for_hf}, we have:

    \begin{align*}
        \log\mathcal{N}\bigRound{\G, \epsilon, L_2(\Sds)} \le \log\mathcal{N}\bigRound{\Hf, \epsilon/\eta, L_\infty(\Sds_1)} \le \frac{1024\eta^2B_L^2\mathcal{\bar{R}_A}^2}{\epsilon^2}\log\biggRound{
            \biggRound{
                \frac{44\eta B_L\mathcal{\bar{R}_A}}{\epsilon} + 7
            }n(k+2)W
        }.
    \end{align*}

    \noindent Denote that $B_\G = \sup_{\pA\in\mathcal{A}} \|g_\pA\|_{L_2(\Sds)}$. By Dudley's entropy integral with the choice of $\alpha=1/n$, we have:
    \begin{align*}
        \ERC_{\Sds}(\G) 
        &\le 4\alpha + \frac{12}{\sqrt n}\int_{\alpha}^{B_\G} \sqrt{\log\mathcal{N}\bigRound{\G, \epsilon, L_2(\Sds)}}d\epsilon \\
        &\le 4\alpha + \frac{384\eta B_L\mathcal{\bar{R}_A}}{\sqrt n} \log^{\frac{1}{2}}\biggRound{
            \biggRound{
                \frac{44\eta B_L\mathcal{\bar{R}_A}}{\alpha} + 7
            }n(k+2)W
        }
        \int_{\alpha}^{B_\G} \frac{1}{\epsilon}d\epsilon \\
        &= 4\alpha + \frac{384\eta B_L\mathcal{\bar{R}_A}}{\sqrt n} \log^{\frac{1}{2}}\biggRound{
            \biggRound{
                \frac{44\eta B_L\mathcal{\bar{R}_A}}{\alpha} + 7
            }n(k+2)W
        }
        \log(B_\G/\alpha) \\
        &= \frac{4}{n} + \frac{384\eta B_L\mathcal{\bar{R}_A}}{\sqrt n} \log^{\frac{1}{2}}\bigRound{
            \bigRound{
                44\eta B_L\mathcal{\bar{R}_A}n + 7
            }n(k+2)W
        }
        \log(nB_\G).
    \end{align*}

    \noindent Using Rademacher complexity bound, for $\f\in\F_\mathcal{A}$, we have:
    \begin{equation*}
    \begin{aligned}
        \Lun(\f) - \Lunhat(\f) &\le 2\ERC_{\Sds}(\G) + 3M\sqrt{\frac{\log 2/\delta}{2n}} \\
        &\le \frac{8}{n} + \frac{768\eta B_L\mathcal{\bar{R}_A}}{\sqrt n} \log^{\frac{1}{2}}\bigRound{
            \bigRound{
                44\eta B_L\mathcal{\bar{R}_A}n + 7
            }n(k+2)W
        }
        \log(nB_\G) + 3M\sqrt{\frac{\log 2/\delta}{2n}}.
    \end{aligned}
    \end{equation*}

    \noindent Furthermore, for $1\le j \le n$, by $\ell^\infty$-Lipschitzness of $g_\pA$, for all $\f\in\F_\mathcal{A}$, we have:
    \begin{align*}
        \bigAbs{g_\pA(x_j, x_j^+, x_{j1:k}^-)} &\le \eta \cdot \max_{\xtrip_{ij}\in \Sds_1} \bigAbs{
            h_\pA(x_j, x_j^+, x_{ji}^-)
        } \ \ \ (\xtrip_{ij} = (x_j, x_j^+, x_{ji}^-), \ 1 \le j \le n, 1 \le i \le k) \\ 
        &\le 4\eta B_L \cdot \max_{\tilde x \in \xtrip_{ij}}\|\f(\tilde x)\|_2 \ \ \ (\text{Proposition }\ref{prop:covering_number_for_hf}) \\
        &\le 4\eta B_L^2.
    \end{align*}

    \noindent As a result:
    \begin{align*}
        B_\G &= \sup_{\pA\in\mathcal{A}}\biggRound{
            \frac{1}{n}\sum_{j=1}^n g_\pA(x_j, x_j^+, x_{j1:k}^-)^{1/2}
        }^{1/2} \le 4\eta B_L^2.
    \end{align*}

    \noindent Finally, replacing $B_\G$ with $4\eta B_L^2$, we have:
    \begin{align*}
        \Lun(\f) - \Lunhat(\f) &\le \frac{8}{n} + \frac{768\eta B_L\mathcal{\bar{R}_A}}{\sqrt n}\log^{\frac{1}{2}}\bigRound{
            \bigRound{
                44\eta B_L\mathcal{\bar{R}_A}n + 7
            }n(k+2)W
        }
        \log(4n\eta B_L^2) + 3M\sqrt{\frac{\log 2/\delta}{2n}} \\
        &\le \tilde\bigO\biggRound{
            \frac{\eta B_L\mathcal{\bar{R}_A}}{\sqrt n}\log(W)
        } + 3M\sqrt{\frac{\log 2/\delta}{2n}}.
    \end{align*}

    \noindent For $1\le l \le L$, we have:
    \begin{align*}
        B_l &= \sup_{\tilde x \in \X}\sup_{\pA\in\mathcal{A}} \|\vf^{1\to l}(\tilde x)\|_2 \le \sup_{\tilde x\in\X}\|\tilde x\|_2 \prod_{m=1}^l \rho_ms_m = B_x \prod_{m=1}^l \rho_ms_m.
    \end{align*}

    \noindent Therefore, we can upper bound $\mathcal{\bar{R}_A}$ as follows:
    \begin{align*}
        \mathcal{\bar{R}_A}^{2/3} &= \sum_{l=1}^L (a_lB_{l-1}\rho_{l+})^{2/3}\\
        &\le \sum_{l=1}^L \biggRound{
            a_l \cdot B_x\prod_{m=1}^{l-1} \rho_ms_m \cdot \rho_l\prod_{m=l+1}^L s_m\rho_m
        }^{2/3}\\
        &= B_x^{2/3}\biggRound{\prod_{m=1}^L \rho_ms_m}^{2/3}\biggSquare{\sum_{l=1}^L (a_l / s_l)^{2/3}}. \\
        \implies \mathcal{\bar{R}_A} &\le B_x \prod_{m=1}^L \rho_ms_m \biggSquare{\sum_{l=1}^L (a_l / s_l)^{2/3}}^{3/2}.
    \end{align*}

    \noindent Hence, we have the final bound with probability of at least $1-\delta$:
    \begin{align*}
        \Lun(\f) - \Lunhat(\f) &\le \tilde\bigO\biggRound{
            \frac{\eta B_L\mathcal{\bar{R}_A}}{\sqrt n}\log(W)
        } + 3M\sqrt{
            \frac{\log 2/\delta}{2n}
        } \\
        &\le \tilde\bigO\biggRound{
            \frac{\eta B_x^2}{\sqrt n}\log(W)\prod_{m=1}^L \rho_m^2s_m^2\biggSquare{\sum_{l=1}^L (a_l / s_l)^{2/3}}^{3/2}
        } + 3M\sqrt{
            \frac{\log 2/\delta}{2n}
        }.
    \end{align*}
    
    \begin{remark}
    Even if $\ell$ is not intrinsically bounded but is assumed to be $\ell^\infty$-Lipschitz, we can still show that:
    \begin{align*}
        M \le 4\eta B_L^2 \le 4\eta B_x^2\prod_{l=1}^L \rho_l^2s_l^2 = \bigO\biggRound{B_x^2\prod_{l=1}^L \rho_l^2s_l^2}.
    \end{align*}
    \end{remark}
\end{proof}

\section{Loss Augmentation}
As discussed briefly in the main text, loss augmentation technique involves encapsulating the original (bounded) loss function $\ell:\R^k\to[0, M]$ in some augmentation schemes (for example, equations \ref{eq:loss_aug_weima}, \ref{eq:loss_aug_ledent}) such that the modified loss collapses to the upper bound if some desired data-dependent properties are not satisfied. For both augmentation schemes introduced in the main text, the generalization gap satisfies:
\begin{align*}
    \E_{x\sim\mathcal{D}}[\ell(x)] - \frac{1}{n}\sum_{j=1}^n \ell(x_j) &\le \E_{x\sim\mathcal{D}}[\tilde\ell(x)] - \frac{1}{n}\sum_{j=1}^n \ell(x_j) \\
    &\le \E_{x\sim\mathcal{D}}[\tilde\ell(x)] - \frac{1}{n}\sum_{j=1}^n \bigRound{\tilde\ell(x_j) - \1{\exists l:\gamma_l(x_j) > b_l}} \\
    &= \E_{x\sim\mathcal{D}}[\tilde\ell(x)] - \frac{1}{n}\sum_{j=1}^n\tilde\ell(x_j)  + \frac{\mathcal{I}_{\bf B}}{n}.
\end{align*}

\noindent In this section, we use the powerful technique of loss function augmentation to alleviate our bounds' dependency on norm-based quantities. In particular, in Subsection~\ref{sec:sepctramild}, we demonstrate the use of the method to replace the additional dependency on the product of spectral norms in Theorem~\ref{thm:basic_bound} by an empirical analogue, leading to Theorem~\ref{thm:last_act_augmentation}. Next, in Subsection~\ref{sec:allmild}, we further extend the technique to further  improve the dependency on the product of spectral norms from the input layer to the layer to cover, culminating in a proof of Theorem~\ref{thm:all_act_augmentation}.
\subsection{Proof of Theorem \ref{thm:last_act_augmentation}}
\label{sec:sepctramild}
\begin{proposition}
    \label{prop:last_act_augmented_loss_covering_number}
    Let $R\ge1$ be a known constant. Define the augmented loss function class $\tilde\G$ as follows:
    \begin{align}
        \label{eq:last_act_augmented_loss}
        \tilde\G = \biggCurl{
            \xin_j = \bigRound{x_j, x_j^+, x_{j1:k}^-} \mapsto \max\biggSquare{
                \ell(V_{\pA, j}), \max_{\tilde x \in \xin_j}\lambda_R\bigRound{
                    \|\f(\tilde x)\|_2
                } 
            } : \pA \in \mathcal{A}
        },
    \end{align}

    \noindent where $\ell:\R^k \to [0, 1]$ is a loss function that is $\ell^\infty$-Lipschitz with constant $\eta\ge1$ and $V_{\pA, j} = \bigCurl{h_\pA(x_j, x_j^+, x_{ji}^-)}_{i=1}^k$. Then for any $\epsilon \in (0,1)$, we have:
    \begin{align}
        \log\mathcal{N}\bigRound{
            \tilde \G, \epsilon, L_2(\Sds)
        } &\le \frac{6400\eta^2 R^2\mathcal{\bar{R}_A}^2}{\epsilon^2}\log\biggRound{
            \biggRound{
                \frac{110\eta R^2\mathcal{\bar{R}_A}}{\epsilon} + 7
            }n(k+2)W
        }.
    \end{align}

    \noindent Where we define $\mathcal{\bar{R}_A}^{2/3} = \sum_{l=1}^L (a_lB_{l-1}\rho_{l+})^{2/3}$ and $\rho_{l+} = \rho_l\prod_{m=l+1}^L\rho_ms_m$.
\end{proposition}

\begin{proof}
    We conduct the proof by constructing the following covers:
    \begin{itemize}
        \item Construct the cover $\mathcal{C}$ for $\mathcal{{F}_A}$ (restricted to dataset $\Sds_2$) with respect to the $L_{\infty, 2}$ metric satisfying: For a constant $0<\varepsilon<1$, for all $\pA\in\mathcal{A}$, there exists $\bpA\in\mathcal{C}$ such that:
        \begin{align*}
            \Big\| \f - \barf \Big\|_{L_{\infty, 2}(\Sds_2)} = \max_{\tilde x\in \Sds_2}\Big\| \f(\tilde x) - \barf(\tilde x) \Big\|_2 < \varepsilon R.
        \end{align*}
        
        \item Construct the cover $\mathcal{\tilde C}$ for $\tilde\G$ (restricted to dataset $\Sds$) with respect to the $L_2$ metric defined as follows:
        \begin{align*}
            \mathcal{\tilde C} &= \biggCurl{
                \xin_j = \bigRound{x_j, x_j^+, x_{j1:k}^-} \mapsto \max\biggSquare{
                \ell(V_{\bpA, j}), \max_{\tilde x \in \xin_j}\lambda_R\bigRound{
                    \|\barf(\tilde x)\|_2
                } 
            } : \bpA \in \mathcal{C}
            }.
        \end{align*}
    \end{itemize}

    \noindent \textbf{1. Construct the cover for $\mathcal{{F}_A}$}: By lemmas \ref{lem:covering_number_of_VA_general} and proposition \ref{prop:covering_number_of_FA_specific}, for known positive constants $\varepsilon_1, \dots, \varepsilon_L$, we have:
    \begin{align*}
        \log\mathcal{N}\biggRound{{\mathcal{{F}_A}}, \sum_{l=1}^L \varepsilon_l\rho_{l+}, L_{\infty, 2}(\Sds_2)} 
        &\le \log\bigRound{\bigRound{11\max_{1\le l \le L}a_lB_{l-1}/\varepsilon_l +7}n(k+2)W}\sum_{l=1}^L \frac{64a_l^2B_{l-1}^2}{\varepsilon_l^2}.
    \end{align*}

    \noindent Let $\{\beta_l\}_{l=1}^L$ be a set of weights (which will be determined later) satisfying that $\sum_{l=1}^L \beta_l = 1$. Let $0<\varepsilon<1$, we define $\varepsilon_1, \dots, \varepsilon_L$ as follows:
    \begin{align*}
        \varepsilon_l = \frac{\beta_l R\varepsilon}{\rho_{l+}}, \ \ \rho_{l+} = \rho_l\prod_{m=l+1}^L \rho_ms_m.
    \end{align*}

    \noindent Define $\mathcal{C}$ as the cover for $\mathcal{{F}_A}$ corresponding to the above choices of $\varepsilon_l$. Given an arbitrary function $\f\in\mathcal{{F}_A}$, let $\bpA\in\mathcal{C}$ be the closest cover element to $\f$. By induction, we have:
    \begin{align*}
        \Big\| \f - \barf \Big\|_{L_{\infty, 2}(\Sds_2)} &\le \sum_{l=1}^L \varepsilon_l\rho_{l+}
        = \sum_{l=1}^L \rho_{l+} \frac{\beta_l R\varepsilon}{\rho_{l+}} = R\varepsilon\sum_{l=1}^L \beta_l = R\varepsilon.
    \end{align*}

    \noindent\textbf{2. Construct the cover for $\tilde\G$}: We use the notation $\tilde\ell_\pA$ to denote the augmented loss that is applied on the neural network parameterized by $\pA\in\mathcal{A}$. Given $\f\in\mathcal{F_A}$, choose $\barf$ as the closest cover element according to the cover $\mathcal{C}$ we constructed above. Then, for an input tuple $\xin_j = (x_j, x_j^+, x_{j1:k}^-) \in \Sds$, we have:
    \begin{align}
        \bigAbs{\tilde\ell_\pA(x_j, x_j^+, x_{j1:k}^-) - \tilde\ell_\bpA(x_j, x_j^+, x_{j1:k}^-)} \nonumber &\le \max\biggSquare{
            \bigAbs{\ell(V_{\pA, j}) - \ell(V_{\bpA, j})}, \max_{\tilde x \in \xin_j}\lambda_R\bigRound{
                \|\f(\tilde x)\|_2
            } -  \max_{\tilde x \in \xin_j}\lambda_R\bigRound{
                \|\barf(\tilde x)\|_2
            }
        }  \nonumber \\
        &\le \max\biggSquare{
            \bigAbs{\ell(V_{\pA, j}) - \ell(V_{\bpA, j})}, \max_{\tilde x \in \xin_j}\bigAbs{\lambda_R\bigRound{
                \|\f(\tilde x)\|_2
            } - \lambda_R\bigRound{
                \|\barf(\tilde x)\|_2
            }}
        }  \nonumber \\
        &\le \max\biggSquare{
            \bigAbs{\ell(V_{\pA, j}) - \ell(V_{\bpA, j})}, \frac{1}{R}\max_{\tilde x \in \xin_j}\bigAbs{\|\f(\tilde x)\|_2 - \|\barf(\tilde x)\|_2}
        }  \nonumber \\
        &\le \max\biggSquare{
            \bigAbs{\ell(V_{\pA, j}) - \ell(V_{\bpA, j})}, \frac{1}{R}\underbrace{\max_{\tilde x \in \xin_j}\|\f(\tilde x) - \barf(\tilde x)\|_2}_{\le \varepsilon R}
        } \nonumber \\
        &\le \max\bigSquare{
            \bigAbs{\ell(V_{\pA, j}) - \ell(V_{\bpA, j})}, \varepsilon
        } \label{eq:equation_1}. \tag{A}
    \end{align}

    \noindent Now, we have to find the bound for $\bigAbs{\ell(V_{\pA, j}) - \ell(V_{\bpA, j})}$. For an input triplet $\xtrip_{ij} = (x_j, x_j^+, x_{ji}^-)$, using a similar argument as proposition \ref{prop:covering_number_for_hf}, we have:
    \begin{align}
        \bigAbs{h_\pA(x_j, x_j^+, x_{ji}^-) - h_\bpA(x_j, x_j^+, x_{ji}^-)}
        &\le \Big\|\f(x_j) - \barf(x_j)\Big\|_2\bigRound{\Big\|\f(x_j^+)\Big\|_2 + \Big\|\f(x_{ji}^-)\Big\|_2} 
         \nonumber \\ & + \Big\|\barf(x_j)\Big\|_2\bigRound{\Big\|\f(x_j^+) - \barf(x_j^+)\Big\|_2 
        + \Big\|\f(x_{ji}^-) - \barf(x_{ji}^-)\Big\|_2}  \nonumber \\
        &\le \max_{\tilde x \in \xtrip_{ij}}\Big\| \f(\tilde x) - \barf(\tilde x)\Big\|_2\bigRound{
            \Big\|\f(x_j^+)\Big\|_2 + \Big\|\f(x_{ji}^-)\Big\|_2 + 2\Big\|\barf(x_j)\Big\|_2
        }  \nonumber \\
        &\le 2\max_{\tilde x \in \xtrip_{ij}}\Big\| \f(\tilde x) - \barf(\tilde x)\Big\|_2\cdot\max_{\bar x \in \xtrip_{ij}}\bigRound{
            \Big\|\barf(\bar x)\Big\|_2 + \Big\|\f(\bar x)\Big\|_2
        }. \nonumber
    \end{align}

    \noindent Then, by the $\ell^\infty$-Lipschitzness of $\ell$, we have:
    \begin{align}
        \bigAbs{\ell(V_{\pA, j}) - \ell(V_{\bpA, j})} &\le \eta \cdot \Big\|
            V_{\pA, j} - V_{\bpA, j}
        \Big\|_\infty  \nonumber \\
        &\le \eta\cdot\max_{1\le i \le k} \bigAbs{h_\pA(x_j, x_j^+, x_{ji}^-) - h_\bpA(x_j, x_j^+, x_{ji}^-)}  \nonumber \\
        &\le 2\eta\cdot\max_{1\le i \le k}\biggSquare{\max_{\tilde x \in \xtrip_{ij}}\Big\| \f(\tilde x) - \barf(\tilde x)\Big\|_2\cdot\max_{\bar x \in \xtrip_{ij}}\bigRound{
            \Big\|\barf(\bar x)\Big\|_2 + \Big\|\f(\bar x)\Big\|_2
        }} \nonumber \\
        &= 2\eta \cdot \underbrace{\max_{\tilde x \in \xin_j}\Big\| \f(\tilde x) - \barf(\tilde x)\Big\|_2}_{\le \varepsilon R}\cdot\max_{\bar x \in \xin_j}\bigRound{
            \Big\|\barf(\bar x)\Big\|_2 + \Big\|\f(\bar x)\Big\|_2
        } \nonumber \\
        &= 2R\eta\varepsilon \cdot \max_{\tilde x \in \xin_j}\bigRound{
            \Big\|\barf(\bar x)\Big\|_2 + \Big\|\f(\bar x)\Big\|_2
        } \label{eq:equation_2}. \tag{B}
    \end{align}

    \noindent Now, we consider the following cases:
    \begin{itemize}
        \item $(a)$ $\exists \tilde x\in \xin_j: \|\f(\tilde x)\|_2 > 2R$ and $\exists \tilde x\in \xin_j: \|\barf(\tilde x)\|_2 > 2R$.
        \item $(b)$ $\forall \tilde x\in \xin_j:  \|\f(\tilde x)\|_2 \le 2R$ or $\|\barf(\tilde x)\|_2 \le 2R$.
    \end{itemize}

    \noindent When $(a)$ occurs, we have $\bigAbs{\tilde\ell_\pA(x_j, x_j^+, x_{j1:k}^-) - \tilde\ell_\bpA(x_j, x_j^+, x_{j1:k}^-)} = 0$. Hence, we focus on the case $(b)$ where either $\f$ or $\barf$ output's $\ell^2$ norm does not exceed $2R$:
    \begin{itemize}
        \item $\forall \tilde x \in \xin_j, \|\f(\tilde x)\|_2 \le 2R$: Then, we have
        \begin{align*}
            \forall \tilde x \in \xin_j: \|\barf(\tilde x)\|_2 &\le \|\f(\tilde x)\|_2  + \|\f(\tilde x)-\barf(\tilde x)\|_2 \\ &\le 2R + \varepsilon R \le 3R.
        \end{align*}

        \item $\forall \tilde x \in \xin_j, \|\barf(\tilde x)\|_2 \le 2R$: Using the same triangle inequality argument, we have $\|\f(\tilde x)\|_2 \le 3R$.
    \end{itemize}

    \noindent Hence, when case $(b)$ occurs, for all $\tilde x \in \xin_j$, we have $\|\f(\tilde x)\|_2 + \|\barf(\tilde x)\|_2 \le 5R$. Therefore, from the inequalities \ref{eq:equation_1} and \ref{eq:equation_2}, we have:
    \begin{align*}
        \bigAbs{\tilde\ell_\pA(x_j, x_j^+, x_{j1:k}^-) - \tilde\ell_\bpA(x_j, x_j^+, x_{j1:k}^-)} &\le \max\biggSquare{
            2R\eta\varepsilon \cdot \underbrace{\max_{\tilde x \in \xin_j}\bigRound{
                \Big\|\barf(\bar x)\Big\|_2 + \Big\|\f(\bar x)\Big\|_2
            }}_{\le 5R}, \varepsilon
        } \\
        &\le \max\bigSquare{10R^2\eta\varepsilon, \varepsilon} \\
        &= 10R^2\eta\varepsilon. \ \ \ (\text{Since } \eta\ge1, R\ge1)
    \end{align*}

    \noindent For a desired granularity $\epsilon\in(0,1)$, set $\varepsilon = \epsilon/10R^2\eta$ ($\varepsilon<1$ as we initially hypothesized). Then, we have:
    \begin{align*}
        \log\mathcal{N}\bigRound{
            \tilde \G, \epsilon, L_2(\Sds)
        } &\le 64\sum_{l=1}^L\frac{a_l^2B_{l-1}^2}{\varepsilon_l^2}\log\bigRound{\bigRound{11\max_{1\le l \le L}a_lB_{l-1}/\varepsilon_l+7}n(k+2)W} \\
        &= 64\sum_{l=1}^L \frac{a_l^2B_{l-1}^2\rho_{l+}^2}{\beta_l^2R^2\varepsilon^2}\log\bigRound{\bigRound{11\max_{1\le l \le L}a_lB_{l-1}/\varepsilon_l+7}n(k+2)W} \\
        &= 6400R^4\eta^2 \sum_{l=1}^L \frac{a_l^2B_{l-1}^2\rho_{l+}^2}{\beta_l^2R^2\epsilon^2}\log\bigRound{\bigRound{11\max_{1\le l \le L}a_lB_{l-1}/\varepsilon_l+7}n(k+2)W} \\
        &= 6400R^2\eta^2\sum_{l=1}^L \frac{a_l^2B_{l-1}^2\rho_{l+}^2}{\beta_l^2\epsilon^2}\log\bigRound{\bigRound{11\max_{1\le l \le L}a_lB_{l-1}/\varepsilon_l+7}n(k+2)W}.
    \end{align*}

    \noindent Using Lagrange multiplier to optimize for the sum of non-logarithm terms over the choices of $\beta_l$, we obtain:
    \begin{align*}
        \beta_l = \frac{(a_lB_{l-1}\rho_{l+})^{2/3}}{\sum_{m=1}^L (a_mB_{m-1}\rho_{m+})^{2/3}}.
    \end{align*}

    \noindent Plugging the above choices of $\beta_l$ to the covering number bound, we have:
    \begin{align*}
        \log\mathcal{N}\bigRound{
            \tilde \G, \epsilon, L_2(\Sds)
        } &\le \frac{6400\eta^2R^2\mathcal{\bar{R}_A}^2}{\epsilon^2}\log\biggRound{
            \biggRound{
                \max_{1\le l \le L}\frac{11a_lB_{l-1}\rho_{l+}}{\beta_l R \varepsilon} + 7
            }n(k+2)W
        } \\
        &= \frac{6400\eta^2R^2\mathcal{\bar{R}_A}^2}{\epsilon^2}\log\biggRound{
            \biggRound{
                \frac{110\eta R}{\epsilon} \cdot \max_{1\le l \le L}\frac{a_lB_{l-1}\rho_{l+}}{\beta_l} + 7
            }n(k+2)W
        } \\
        &\le \frac{6400\eta^2R^2\mathcal{\bar{R}_A}^2}{\epsilon^2}\log\biggRound{
            \biggRound{
                \frac{110\eta R \mathcal{\bar R_A}}{\epsilon} + 7
            }n(k+2)W
        }.
    \end{align*}

    \noindent Where $\mathcal{\bar{R}_A}$ is defined as $\mathcal{\bar{R}_A}^{2/3} = \sum_{l=1}^L (a_lB_{l-1}\rho_{l+})^{2/3}$. With this proposition, we are ready to prove theorem \ref{thm:last_act_augmentation}.
\end{proof}

\begin{proof}[Proof of Theorem \ref{thm:last_act_augmentation}]
    Let $\tilde\G$ be the augmented loss function class defined in equation \ref{eq:last_act_augmented_loss}. Denote $\mathrm{L_{un}^{aug}}(\f)$ and $\mathrm{\widehat{L}_{un}^{aug}}(\f)$ as the population and empirical risks of $\f$ evaluated using the augmented loss. Then, by the Rademacher bound, we have:
    \begin{align*}
        \mathrm{L_{un}^{aug}}(\f) - \mathrm{\widehat{L}_{un}^{aug}}(\f)
            &\le 2\ERC_{\Sds}(\tilde\G) + 3\sqrt{\frac{\log 2/\delta}{2n}}.
    \end{align*}

    \noindent Using proposition \ref{prop:last_act_augmented_loss_covering_number} and Dudley's integral bound with the choice of $\alpha = \frac{1}{n}$, we have:
    \begin{align*}
        \ERC_{\Sds}(\tilde\G) &\le 4\alpha + \frac{12}{\sqrt n}\int_\alpha^1 \sqrt{\log\mathcal{N}\bigRound{\tilde\G, \epsilon, L_2(\Sds)}}d\epsilon \\
        &\le 4\alpha + \frac{960\eta R\mathcal{\bar R_A}}{\sqrt n}\log^{\frac{1}{2}}\biggRound{
            \biggRound{\frac{110\eta R\mathcal{\bar R_A}}{\alpha} + 7}n(k+2)W
        }\int_\alpha^1\frac{1}{\epsilon}d\epsilon \\
        &= \frac{4}{n} + \frac{960\eta R\mathcal{\bar R_A}}{\sqrt n}\log^{\frac{1}{2}}\bigRound{
            \bigRound{110n\eta R\mathcal{\bar R_A} + 7}n(k+2)W
        }\log(n).
    \end{align*}

    \noindent Plugging the above back to the Rademacher bound, we have:
    \begin{align*}
        \mathrm{L_{un}^{aug}}(\f) - \mathrm{\widehat{L}_{un}^{aug}}(\f) &\le 2\ERC_{\Sds}(\tilde\G) + 3\sqrt{\frac{\log 2/\delta}{2n}} \\
            &\le \frac{8}{n} + \frac{1920\eta R\mathcal{\bar R_A}}{\sqrt n}\log^{\frac{1}{2}}\bigRound{
            \bigRound{110n\eta R\mathcal{\bar R_A} + 7}n(k+2)W
        }\log(n)  +3\sqrt{\frac{\log 2/\delta}{2n}} \\
        &\le \tilde\bigO\biggRound{
            \frac{\eta R\mathcal{\bar R_A}}{\sqrt n}\log(W)
        } + 3\sqrt{\frac{\log 2/\delta}{2n}}.
    \end{align*}

    \noindent Furthermore, we have:
    \begin{align*}
        \mathrm{\widehat{L}_{un}^{aug}}(\f)
        &= \frac{1}{n}\sum_{j=1}^n \max\bigSquare{
            \ell(V_{\pA, j}), \max_{\tilde x \in \xin_j}\lambda_R\bigRound{
                \|\f(\tilde x)\|_2
            }
        }\\
        &\le \frac{1}{n}\sum_{j=1}^n \bigSquare{\ell(V_{\pA, j}) + \boldsymbol{1}\Big\{\exists \tilde x \in \xin_j : \|\f(\tilde x)\|_2 > R\Big\}} \\
        &= \frac{1}{n}\sum_{j=1}^n \ell(V_{\pA, j}) + \frac{\mathcal{I}_{\pA, R}}{n} 
        = \Lunhat(\f) + \frac{\mathcal{I}_{\pA, R}}{n}.
    \end{align*}

    \noindent Therefore:
    \begin{align*}
        \Lun(\f) &\le \mathrm{L_{un}^{aug}}(\f)
            \le \mathrm{\widehat{L}_{un}^{aug}}(\f) + \tilde\bigO\biggRound{
            \frac{\eta R\mathcal{\bar R_A}}{\sqrt n}\log(W)
        } + 3\sqrt{\frac{\log 2/\delta}{2n}} \\
        &\le \Lunhat(\f) + \frac{\mathcal{I}_{\pA, R}}{n} + \tilde\bigO\biggRound{
            \frac{\eta R\mathcal{\bar R_A}}{\sqrt n}\log(W)
        } + 3\sqrt{\frac{\log 2/\delta}{2n}}. 
    \end{align*}

    \noindent Using the following inequality:
    \begin{align*}
        \mathcal{\bar{R}_A} &\le B_x \prod_{m=1}^L \rho_ms_m \biggSquare{\sum_{l=1}^L (a_l / s_l)^{2/3}}^{3/2},
    \end{align*}
    we have:
    \begin{align*}
        \Lun(\f) &\le \Lunhat(\f) + \frac{\mathcal{I}_{\pA, R}}{n} + \tilde\bigO\biggRound{
            \frac{\eta R\mathcal{\bar R_A}}{\sqrt n}\log(W)
        } + 3\sqrt{\frac{\log 2/\delta}{2n}} \\
        &\le \Lunhat(\f) + \frac{\mathcal{I}_{\pA, R}}{n} + \tilde\bigO\biggRound{
            \frac{\eta RB_x}{\sqrt n}\log(W) \prod_{m=1}^L\rho_ms_m\biggSquare{
                \sum_{l=1}^L \frac{a_l^{2/3}}{s_l^{2/3}}
            }^{3/2}
        }.
    \end{align*}

    \noindent Hence, we obtain the desired excess risk bound.
\end{proof}

\subsection{Proof of Theorem \ref{thm:all_act_augmentation}}
\label{sec:allmild}
\noindent In this subsection, we extend the techniques from the previous section to further alleviate the dependency on the product of spectral norms from the input to the layer to be covered. The proof techniques, inspired from modifications and simplifications of~\cite{article:wei2020datadependent,article:nagarajan2019,article:ledent2021normbased}, rely on successively imposing soft constraints on the intermediate activations layers, each time applying Proposition~\ref{prop:zhang_linfty_covering_number} to a trimmed down dataset. 
\noindent To demonstrate the procedure we use to cover the augmented class of loss functions, we first present the following simple result for two-layer neural networks case:
\begin{proposition}{(Two-layer case). }
    \label{prop:all_layers_act_covering_num_2layers}
    Let $a_1, a_2, s_1, s_2$ be positive real constants, $d_2, d_1, d_0$ be positive integers and let reference matrices $\M{1}\in\R^{d_1\times d_0}, \M{2}\in\R^{d_2\times d_1}$ be fixed. Define the parameter spaces $\mathcal{B}_1, \mathcal{B}_2$ as follows: 
    \begin{align}
        \mathcal{B}_l &= \bigCurl{
            \A{l} \in \R^{d_l\times d_{l-1}} : \|(\A{l} - \M{l})^\top\|_{2,1} \le a_l, \ \|\A{1}\|_\sigma \le s_l
        }, \ \ \ l\in\{1, 2\}.
    \end{align}

    \noindent Let $\mathcal{A}=\mathcal{B}_1\times\mathcal{B}_2$ be the parameter space for the class of two-layer neural networks. Let $b_0, b_1, b_2$ be known positive real constants greater than or equal to $1$ where $b_1\ge b_2$ and suppose that $\|\tilde x\|_2 \le 3b_0$ for all $\tilde x\in \Sds_2$, we define the augmented class of loss functions as follows:
    \begin{align}
        \tilde \G = \biggCurl{
            \xin_j = \bigRound{x_j, x_j^+, x_{j1:k}^-} \mapsto \max\biggSquare{
                \ell(V_{\pA, j}), \max_{\tilde x\in \xin_j}\max_{l\in\{1,2\}}\lambda_{b_l}\bigRound{\|\f^{1\to l}(\tilde x)\|_2}
            }: \pA\in\mathcal{A}
        },
    \end{align}

    \noindent where $\ell:\R^k \to [0, 1]$ is a loss function that is $\ell^\infty$-Lipschitz with constant $\eta\ge1$ and $V_{\pA, j} = \bigCurl{h_\pA(x_j, x_j^+, x_{ji}^-)}_{i=1}^k$. Then for any $\epsilon \in (0,1)$, we have:
    \begin{align}
        \log\mathcal{N}\bigRound{
            \tilde \G, \epsilon, L_2(\Sds)
        } &\le \frac{76800 \eta^2 b_2^2R^2}{\epsilon^2}\log\biggRound{
            \biggRound{
                \frac{660 \eta b_2\Gamma}{\epsilon} + 7
            }n(k+2)W
        }.
    \end{align}

    \noindent Where $W = \max\bigCurl{d_1, d_2}$, $\Gamma = \max\{1, s_2\rho_2\}\cdot\max_{l\in\{1, 2\}}a_l\rho_lb_{l-1}$ and $R$ is defined as follows:
    \begin{equation}
        R^2 = {a_2^2b_1^2\rho_2^2 + a_1^2b_0^2\rho_1^2\max\{1, s_2\rho_2\}^2}
    \end{equation}
\end{proposition}

\begin{proof}
    Let $0<\varepsilon<1$ and $\varepsilon_1, \varepsilon_2>0$ be positive real numbers which we define as follows:
    \begin{align*}
        \varepsilon_1=\frac{\varepsilon b_2}{2\rho_1\max\{1, s_2\rho_2\}}, \ \ 
        \varepsilon_2 = \frac{\varepsilon b_2}{2\rho_2}.
    \end{align*}
    
    \noindent We construct the cover for $\mathcal{F_A}$:
    \begin{itemize}
        \item \textbf{1. First layer} - We construct the internal cover $\mathcal{C}_1$ for $\mathcal{B}_1$ such that:
        \begin{align*}
            \forall \A{1} \in \mathcal{B}_1, \exists \bA{1}\in\mathcal{C}_1 : \max_{\tilde x \in \Sds_2} \Big\|(\A{1} - \bA{1})\tilde x\Big\|<\varepsilon_1.
        \end{align*}

        \noindent Therefore, we have:
        \begin{align*}
            \max_{\tilde x\in \Sds_2}\Big\|\sigma_1(\A{1}\tilde x) - \sigma_1(\bA{1}\tilde x)\Big\|_2 &\le \rho_1\max_{\tilde x\in \Sds_2}\Big\|(\A{1} - \bA{1})\tilde x\Big\|_2 \le \rho_1\varepsilon_1 \\
                &= \frac{\varepsilon b_2}{2\max\{1, s_2\rho_2\}} \le \frac{\varepsilon b_2}{2} \\ &\le \varepsilon b_1.
        \end{align*}

        \noindent By lemma \ref{prop:corollary_of_ledent_proposition_6}, we have:
        \begin{align*}
            \log|\mathcal{C}_1| \le \frac{192a_1^2b_0^2}{\varepsilon_1^2}\log\biggRound{
                \biggRound{
                    \frac{33a_1b_0}{\varepsilon_1} + 7
                }n(k+2)d_1
            }.
        \end{align*}

        \item \textbf{2. Second layer} - For each cover element of the first layer $\bA{1}\in\mathcal{C}_1$, we construct the following auxiliary datasets that depend on $\bA{1}$:
        \begin{align*}
            \Sds(\bA{1}) &= \bigCurl{
                \xin_j \in \Sds : \|\sigma_1(\bA{1}\tilde x)\|_2 \le 3b_1, \ \forall \tilde x\in \xin_j
            }, \\
            \Sds_2(\bA{1}) &= \bigcup_{
                \xin_j \in \Sds(\bA{1})
            }\bigCurl{
                \tilde x : \tilde x \in \xin_j
            }.
        \end{align*}

        \noindent For each $\bA{1}\in\mathcal{C}_1$, we construct the cover $\mathcal{C}_2(\bA{1})\subset\mathcal{B}_2$ for the second layer restricted to the auxiliary dataset $\Sds_2(\bA{1})$ such that: 
        \begin{align*}
            \forall \A{2}\in\mathcal{B}_2, \exists \bA{2}\in\mathcal{C}_2(\bA{1}) : \max_{\tilde x\in \Sds_2(\bA{1})} \Big\|(\A{2} - \bA{2})\sigma_1(\bA{1}\tilde x)\Big\|_2 < \varepsilon_2.
        \end{align*}

        \noindent Therefore, we have:
        \begin{align*}
            \max_{\tilde x\in \Sds_2(\bA{1})}\Big\|&\sigma_2(\A{2}\sigma_1(\A{1}\tilde x)) - \sigma_2(\bA{2}\sigma_1(\bA{1}\tilde x))\Big\|_2 \\ 
            &\le \rho_2\max_{\tilde x\in \Sds_2(\bA{1})} \biggRound{
                \Big\|(\A{2} - \bA{2})\sigma_1(\bA{1}\tilde x)\Big\|_2 
                + \Big\|\A{2}\bigRound{\sigma_1(\A{1}\tilde x) - \sigma_1(\bA{1}\tilde x)}\Big\|_2
            } \\
            &\le \rho_2 \max_{\tilde x\in \Sds_2(\bA{1})}\Big\|(\A{2} - \bA{2})\sigma_1(\bA{1}\tilde x)\Big\|_2 + \rho_2s_2\rho_1\Big\|\A{1}\tilde x - \bA{1}\tilde x\Big\|_2 \\
            &\le \rho_2\varepsilon_2 + \rho_2s_2\rho_1\varepsilon_1 \\
            &= \frac{\varepsilon b_2}{2} + \frac{\varepsilon s_2\rho_2 b_2}{2\max\{1, s_2\rho_2\}} \\
            &\le \varepsilon b_2.
        \end{align*}
        
        \noindent By lemma \ref{prop:corollary_of_ledent_proposition_6}, we have:
        \begin{align*}
            \log|\mathcal{C}_2(\bA{1})| &= \log\mathcal{N}\bigRound{
                {\mathcal{V}_2}, \varepsilon_2, L_{\infty, 2}(\Sds_2(\bA{1})) 
            } \\
            &\le \frac{192a_2^2b_1^2}{\varepsilon_2^2}\log\biggRound{
                \biggRound{
                    \frac{33a_2b_1}{\varepsilon_2} + 7
                }|\Sds(\bA{1})|(k+2)d_2
            } \\
            &\le \frac{192a_2^2b_1^2}{\varepsilon_2^2}\log\biggRound{
                \biggRound{
                    \frac{33a_2b_1}{\varepsilon_2} + 7
                }n(k+2)d_2
            }.
        \end{align*}

        \item \textbf{3. Combine the covers} - Construct the final cover $\mathcal{C}$ for $\mathcal{A}$ defined as follows:
        \begin{align*}
            \mathcal{C} = \bigCurl{
                (\bA{1}, \bA{2}) : \bA{1} \in \mathcal{C}_1, \bA{2}\in\mathcal{C}_2(\bA{1})
            }.
        \end{align*}

        \noindent From the above construction of $\mathcal{C}$, we have:
        \begin{align*}
            |\mathcal{C}| &\le \sum_{\bA{1}\in\mathcal{C}_1} |\mathcal{C}_2(\bA{1})| \le |\mathcal{C}_1|\cdot \sup_{\bA{1}\in\mathcal{C}_1}|\mathcal{C}_2(\bA{1})|.
        \end{align*}

        \noindent Taking logarithm from both sides of the inequality, we have:
        \begin{align*}
            \log|\mathcal{C}|&\le \log|\mathcal{C}_1| + \sup_{\bA{1}\in\mathcal{C}_1} \log|\mathcal{C}_2(\bA{1})| \\
            &\le 192(a_2^2b_1^2/\varepsilon_2^2 + a_1^2b_0^2/\varepsilon_1^2)\log\biggRound{
                \biggRound{
                    \frac{66\Gamma}{\varepsilon} + 7
                }n(k+2)W
            }.
        \end{align*}

        \noindent Where we have $W = \max\bigCurl{d_1, d_2}$ and $\Gamma = \max\{1, s_2\rho_2\}\cdot\max_{l\in\{1, 2\}}a_l\rho_lb_{l-1}$. Substituting the values of $\varepsilon_1, \varepsilon_2$ to the above bound, we have:
        \begin{align*}
            \log|\mathcal{C}| &\le \frac{768R^2}{\varepsilon^2b_2^2}\log\biggRound{
                \biggRound{
                    \frac{66\Gamma}{\varepsilon b_2} + 7
                }n(k+2)W
            }.
        \end{align*}

        \noindent Where $R$ is defined as:
        \begin{align*}
            R^2 = {a_2^2b_1^2\rho_2^2 + a_1^2b_0^2\rho_1^2\max\{1, s_2\rho_2\}^2}.
        \end{align*}
    \end{itemize}

    \noindent From the above construction, for all $\pA = (\A{1}, \A{2}) \in \mathcal{A}$, we can choose a cover element $\bpA = (\bA{1}, \bA{2})\in\mathcal{C}$ by:
    \begin{itemize}
        \item Select $\bA{1}\in\mathcal{C}_1$ as the closest cover element to $\A{1}$.
        \item Select $\bA{2}\in\mathcal{C}_2(\bA{1})$ as the closest cover element to $\A{2}$.
    \end{itemize}

    \noindent For a given tuple of weight matrices $\pA\in\mathcal{A}$, we define $E_\pA^{(1)}, E_\pA^{(2)}$ as the following auxiliary datasets:
    \begin{align*}
        E_\pA^{(1)} &= \bigCurl{
            \xin_j \in \Sds: \exists \tilde x\in \xin_j \text{ where } \|\sigma_1(\A{1}\tilde x)\|_2 > 2b_1
        }, \\
        E_\pA^{(2)} &= \bigCurl{
            \xin_j \in \Sds: \exists \tilde x\in \xin_j \text{ where } \|\sigma_2(\A{2}\sigma_1(\A{1}\tilde x))\|_2 > 2b_2
        }.
    \end{align*}

    \noindent Define the auxiliary dataset $\Sds'\subseteq \Sds$ as follows:
    \begin{align*}
        \Sds' &= \Sds\setminus\bigSquare{
            (E_\pA^{(1)} \cup E_\pA^{(2)})\cap (E_\bpA^{(1)} \cup E_\bpA^{(2)})
        }.
    \end{align*}

    \noindent In other words, we construct the set $\Sds'$ by removing the input tuples from $\Sds$ that makes the augmented losses $\tilde\ell_\pA, \tilde\ell_\bpA$ evaluated for $\pA, \bpA$ both collapse to $1$ due to large activations. It is straightforward that $\Sds'\subseteq \Sds(\bA{1})$ because for all $\xin_j\in \Sds'$, we have $\max_{\tilde x\in \xin_j}\|\sigma_1(\A{1}\tilde x)\|_2 \le 2b_1$ or $\max_{\tilde x\in \xin_j}\|\sigma_1(\bA{1}\tilde x)\|_2 \le 2b_1$. When the former inequality is satisfied, by the triangle inequality, we have:
    \begin{align*}
        \max_{\tilde x \in \xin_j}\|\sigma_1(\bA{1}\tilde x)\|_2 &\le \max_{\tilde x\in \xin_j}\bigRound{\|\sigma_1(\A{1}\tilde x)\|_2 + \|\sigma_1(\A{1}\tilde x)-\sigma_1(\bA{1}\tilde x)\|_2} \\
        &\le 2b_1 + \varepsilon b_1 \le 3b_1.
    \end{align*}

    \noindent Hence, $\forall \xin_j \in \Sds' : \xin_j \in \Sds(\A{1}) \implies \Sds'\subseteq \Sds(\bA{1})$. Furthermore, by the construction of set $\Sds(\bA{1})$, we have:
    \begin{align*}
        \forall \xin_j \in \Sds(\bA{1}), l \in \{1, 2\}: \max_{\tilde x\in \xin_j}\|\f^{1\to l}(\tilde x)-\barf^{1\to l}(\tilde x)\|_2 \le \varepsilon b_l.
    \end{align*}

    \noindent Therefore, for all input tuple $\xin_j \in \Sds'$, we have:
    \begin{align*}
        &\bigAbs{
            \tilde\ell_\pA(x_j, x_j^+, x_{j1:k}^-) - \tilde\ell_\bpA(x_j, x_j^+, x_{j1:k}^-)
        } \\
        &\le \max\biggSquare{
            \bigAbs{\ell(V_{\pA, j}) - \ell(V_{\bpA, j})}, \max_{\tilde x\in \xin_j}\max_{l\in\{1,2\}}\bigAbs{
                \lambda_{b_l}(\|\f^{1\to l}(\tilde x)\|_2) - \lambda_{b_l}(\|\barf^{1\to l}(\tilde x)\|_2)
            }
        }\\
        &\le \max\biggSquare{
            \bigAbs{\ell(V_{\pA, j}) - \ell(V_{\bpA, j})}, \max_{\tilde x\in \xin_j}\max_{l\in\{1,2\}}\frac{1}{b_l}\bigAbs{
                \|\f^{1\to l}(\tilde x)\|_2 - \|\barf^{1\to l}(\tilde x)\|_2
            }
        } \\
        &\le \max\biggSquare{
            \bigAbs{\ell(V_{\pA, j}) - \ell(V_{\bpA, j})}, \max_{\tilde x\in \xin_j}\max_{l\in\{1,2\}}\frac{1}{b_l}\Big\|\f^{1\to l}(\tilde x) - \barf^{1\to l}(\tilde x)\Big\|_2
        } \\
        &\le \max\biggSquare{
            \bigAbs{\ell(V_{\pA, j}) - \ell(V_{\bpA, j})}, \max_{\tilde x\in \xin_j}\max_{l\in\{1,2\}}\frac{1}{b_l}\cdot\varepsilon b_l
        } \\
        &= \max\biggSquare{
            \bigAbs{\ell(V_{\pA, j}) - \ell(V_{\bpA, j})}, \varepsilon
        }.
    \end{align*}

    \noindent Using the same argument from proposition \ref{prop:last_act_augmented_loss_covering_number}, for all $\xin_j\in \Sds'$, we have:
    \begin{align*}
        \bigAbs{\ell(V_{\pA, j}) - \ell(V_{\bpA, j})} &\le 2\eta \cdot \underbrace{\max_{\tilde x \in \xin_j}\Big\| \f(\tilde x) - \barf(\tilde x)\Big\|_2}_{\le \varepsilon b_2}\cdot\underbrace{\max_{\bar x \in \xin_j}\bigRound{
            \Big\|\barf(\bar x)\Big\|_2 + \Big\|\f(\bar x)\Big\|_2
        }}_{\le 5b_2} \\
        &\le 10\eta b_2^2\varepsilon.
    \end{align*}

    \noindent Since $10\eta b_2^2\varepsilon>\varepsilon$ ($b_2\ge1, \eta\ge1$), we have:
    \begin{align*}
        \forall \xin_j \in \Sds' : \bigAbs{
            \tilde\ell_\pA(x_j, x_j^+, x_{j1:k}^-) - \tilde\ell_\bpA(x_j, x_j^+, x_{j1:k}^-)
        } &\le \max\bigSquare{10\eta b_2^2\varepsilon, \varepsilon} = 10\eta b_2^2\varepsilon.
    \end{align*}

    Therefore:
    \begin{align*}
        \|\tilde\ell_\pA - \tilde\ell_\bpA\|_{L_2(\Sds)} &= \sum_{j=1}^n \bigAbs{
            \tilde\ell_\pA(x_j, x_j^+, x_{j1:k}^-) - \tilde\ell_\bpA(x_j, x_j^+, x_{j1:k}^-)
        } \\
        &= \sum_{\xin_j \in \Sds'}\bigAbs{
            \tilde\ell_\pA(x_j, x_j^+, x_{j1:k}^-) - \tilde\ell_\bpA(x_j, x_j^+, x_{j1:k}^-)
        } \\
        &\le 10\eta b_2^2\varepsilon.
    \end{align*}

    \noindent For a desired granularity $\epsilon\in(0,1)$, set $\varepsilon = \epsilon/10\eta b_2^2$. Then, we can construct an $\epsilon$-cover $\mathcal{\tilde C}$ for $\tilde\G$ (restricted to dataset $\Sds$) with respect to the $L_2$ metric as follows:
    \begin{align*}
        \mathcal{\tilde C} = \biggCurl{
            \xin_j = \bigRound{x_j, x_j^+, x_{j1:k}^-} \mapsto \max\biggSquare{
                \ell(V_{\bpA, j}), \max_{\tilde x\in \xin_j}\max_{l\in\{1,2\}}\lambda_{b_l}\bigRound{\|\barf^{1\to l}(\tilde x)\|_2}
            }: \bpA\in\mathcal{C}
        }.
    \end{align*}

    \noindent Therefore, we have:
    \begin{align*}
        \log|\mathcal{\tilde C}| &\le \frac{768 R^2}{b_2^2(\epsilon/10\eta b_2^2)^2}\log\biggRound{
                \biggRound{
                    \frac{66\Gamma }{b_2(\epsilon/10\eta b_2^2)} + 7
                }n(k+2)W
            } \\
            &= \frac{76800 \eta^2 b_2^2R^2}{\epsilon^2}\log\biggRound{
                \biggRound{
                    \frac{660 \eta b_2\Gamma}{\epsilon} + 7
                }n(k+2)W
            }.
    \end{align*}
\end{proof}

\begin{proposition}
    \label{eq:covering_number_all_act_augmentation}
    Let ${\bf B} = (b_0, b_1, \dots, b_L)$ be known constants such that $b_l \ge 1$ for all $0\le l \le L$. Assume that $\|\tilde x\|_2 \le 3b_0$ for all $\tilde x \in \Sds_2$. Define the augmented loss function class $\tilde\G$ as follows:
    \begin{align}
        \tilde\G = \biggCurl{
            \xin_j = \bigRound{x_j, x_j^+, x_{j1:k}^-} \mapsto \max\biggSquare{
                \ell(V_{\pA, j}), \max_{1 \le l \le L}\max_{\tilde x \in \xin_j}\lambda_{b_l}\bigRound{
                    \|\f^{1\to l}(\tilde x)\|_2
                } 
            } : \pA \in \mathcal{A}
        },
    \end{align}

    \noindent where $\ell:\R^k \to [0, 1]$ is a loss function that is $\ell^\infty$-Lipschitz with constant $\eta\ge1$ and $V_{\pA, j} = \bigCurl{h_\pA(x_j, x_j^+, x_{ji}^-)}_{i=1}^k$. Then, for any $\epsilon \in (0,1)$, we have:
    \begin{align}
        \log\mathcal{N}\bigRound{
            \tilde \G, \epsilon, L_2(\Sds)
        } &\le \frac{19200\eta^2b_L^4\mathcal{\widehat{R}_A}^2}{\epsilon^2}\log\biggRound{
            \biggRound{
                \frac{330\eta b_L^2\mathcal{\widehat{R}_A}}{\epsilon} + 7
            }n(k+2)W
        },
    \end{align}

    \noindent where $\mathcal{\widehat{R}_A}^{2/3} = \sum_{l=1}^L(a_lb_{l-1}\hat\rho_l)^{2/3}$ and $\hat\rho_l=\sup_{u \ge l}{\rho_{l\to u}}/{b_{u}}$, where $\rho_{l\to u}=\rho_l\prod_{m=l+1}^u s_m\rho_m$.
\end{proposition}

\begin{proof}
    Let $\varepsilon_1, \dots, \varepsilon_L$ and $0<\varepsilon<1$ be constants defined as $\varepsilon_l = {\beta_l\varepsilon}/{\hat\rho_l}$ for all $1\le l \le L$ such that $\sum_{l=1}^L\beta_l = 1$ and $\hat\rho_l=\sup_{u \ge l}{\rho_{l\to u}}/{b_{u}}$ where $\rho_{l\to u}=\rho_l\prod_{m=l+1}^u s_m\rho_m$. We construct the cover for $\F_\mathcal{A}$ iteratively using the following procedure:
    \begin{itemize}
        \item \textbf{For the first layer}: Construct the $\varepsilon_1$-cover $\mathcal{C}_1$ for the class $\mathcal{V}_1 = \bigCurl{z\mapsto\A{1}: \A{1}\in\mathcal{B}_1}$ restricted to dataset $\Sds_2$ with respect to the $L_{\infty, 2}$ metric. By proposition \ref{prop:corollary_of_ledent_proposition_6}, we have:
        \begin{align*}
            \log|\mathcal{C}_1| &\le \frac{192b_0^2a_1^2}{\varepsilon_1^2}\log\biggRound{
                \biggRound{
                    \frac{33a_1b_0}{\varepsilon_1} + 7
                }n(k+2)d_1
            }.
        \end{align*}

        \noindent Furthermore, for all $\A{1}\in\mathcal{B}_1$, there exists $\bA{1}\in\mathcal{C}_1$ such that:
        \begin{align*}
            \max_{\tilde x \in \Sds_2}\|\sigma_1(\A{1}\tilde x) - \sigma_1(\bA{1}\tilde x)\|_2 &\le \rho_1 \max_{\tilde x\in \Sds_2}\|(\A{1} - \bA{1})\tilde x\|_2 \\ 
            &\le \rho_1\varepsilon_1 = \rho_1\frac{\beta_1\varepsilon}{\hat\rho_1} \\
            &\le \rho_1\frac{\beta_1\varepsilon}{\rho_1/b_1} = \beta_1\varepsilon b_1 \le \varepsilon b_1.
        \end{align*}

        \item \textbf{For the ${l+1}^{th}$ layer}: Suppose that we have constructed the cover $\mathcal{C}_{1\to l}\subset\mathcal{B}_{l} \times \mathcal{B}_{l-1}\times\dots\times\mathcal{B}_1$ for the first $l$ layers. For each cover element $\bpA^{1 \to l}=(\bA{l}, \dots, \bA{1})\in\mathcal{C}_{1\to l}$, construct the following auxiliary datasets:
        \begin{equation}
        \label{eq:trimmed_dataset}
        \begin{aligned}
            \Sds(\bpA^{1 \to l}) &= \bigCurl{
                \xin_j \in \Sds : \|\bff{1}{l}(\tilde x)\|_2 \le 3b_{l}, \ \forall \tilde x \in \xin_j
            }, \\
            \Sds_2(\bpA^{1 \to l}) &= \bigcup_{\xin_j\in \Sds(\bpA^{1 \to l})}\bigCurl{\tilde x : \tilde x\in \xin_j}.
        \end{aligned}
        \end{equation}

        \noindent Then, we construct the $\varepsilon_{l+1}$-cover $\mathcal{C}_{l+1}(\bpA^{1 \to l})\subset\mathcal{B}_{l+1}$ with respect to the $L_{\infty, 2}$ metric for the class of linear functions $\mathcal{V}_l = \bigCurl{z\mapsto \A{l}z: \A{l} \in \mathcal{B}_l}$ restricted to the auxiliary dataset $\bigCurl{\bff{1}{l}(\tilde x): \tilde x \in \Sds_2(\bpA^{1 \to l})}$. Finally, construct the cover $\mathcal{C}_{1\to l+1}$ by taking the union:
        \begin{align*}
            \mathcal{C}_{1\to l+1} = \bigcup_{\bpA^{1\to l} \in \mathcal{C}_{1\to l}}\bigCurl{
                \bA{l+1}\times \bpA^{1\to l} : \bA{l+1} \in \mathcal{C}_{l+1}(\bpA^{1\to l}) 
            },
        \end{align*}

        \noindent where $\bA{l+1}\times\bpA^{1\to l} = (\A{l+1}, \A{l}, \dots, \A{1})$. From the above construction, we have $\mathcal{C}_{1\to l+1}\subset\mathcal{B}_{l+1}\times\mathcal{B}_l\times\dots\times\mathcal{B}_1$. Furthermore, we have:
        \begin{align*}
            |\mathcal{C}_{1\to l+1}| 
            &\le |\mathcal{C}_{1\to l}|\cdot\sup_{\bpA^{1\to l}\in\mathcal{C}_{1\to l}}|\mathcal{C}_{l+1}(\bpA^{1\to l})|. \\
            \implies 
            \log|\mathcal{C}_{1\to l+1}| 
            &\le \log|\mathcal{C}_{1\to l}| + \sup_{\bpA^{1\to l}\in\mathcal{C}_{1\to l}}\log|\mathcal{C}_{l+1}(\bpA^{1\to l})| \\
            &\le \log|\mathcal{C}_{1\to l}| + \frac{192b_{l}a_{l+1}}{\varepsilon_{l+1}^2}\log\biggRound{
                \biggRound{
                    \frac{33b_{l}a_{l+1}}{\varepsilon_{l+1}} + 7
                }n(k+2)d_{l+1}
            }. \ \ \ (\text{Proposition } \ref{prop:corollary_of_ledent_proposition_6})
        \end{align*}

        \noindent Then, for $\pA^{1\to l+1}\in\mathcal{B}_{l+1}\times\dots\times\mathcal{B}_1$ and its corresponding closest cover element $\bpA^{1\to l+1}\in \mathcal{C}_{1\to l+1}$ in $L_{\infty, 2}$ metric, by induction, we have:
        \begin{equation}
        \label{eq:trimmed_dataset_bound}
        \begin{aligned}
            \max_{\tilde x\in \Sds_2(\bpA^{1\to l})}\Big\| \ff{1}{l+1}(\tilde x) - \bff{1}{l+1}(\tilde x) \Big\|_2 &\le \sum_{m=1}^{l+1} \varepsilon_m \rho_{m\to l+1} = \sum_{m=1}^{l+1}\frac{\beta_m\varepsilon}{\hat\rho_m}\rho_{m\to l+1} \\
            &\le \sum_{m=1}^{l+1}\frac{\beta_m\varepsilon}{(\rho_{m\to l+1}/b_{l+1})}\rho_{m\to l+1} = \sum_{m=1}^{l+1}\beta_m \varepsilon b_{l+1} \\
            &\le \varepsilon b_{l+1}.
        \end{aligned}
        \end{equation}
    \end{itemize}

    \noindent Inductively expanding until the $L^{th}$ layer, we obtain the final cover $\mathcal{C}_{1\to L}\subset\mathcal{B}_L\times\dots\times\mathcal{B}_1$. Then, we can bound the cardinality of $\mathcal{C}_{1\to L}$ as follows:
    \begin{align*}
        \log|\mathcal{C}_{1\to L}| &\le \sum_{l=1}^L \frac{192a_lb_{l-1}}{\varepsilon_l^2}\log\biggRound{
            \biggRound{
                \frac{33a_lb_{l-1}}{\varepsilon_l} + 7
            }n(k+2)d_l
        } \\
        &= \sum_{l=1}^L \frac{192a_lb_{l-1}\hat\rho_l^2}{\beta_l^2\varepsilon^2}\log\biggRound{
            \biggRound{
                \frac{33a_lb_{l-1}\hat\rho_l}{\beta_l\varepsilon} + 7
            }n(k+2)d_l
        } \\
        &\le \log\biggRound{
            \biggRound{
                33\max_{1\le l \le L}\frac{a_lb_{l-1}\hat\rho_l}{\beta_l\varepsilon}+7
            }n(k+2)W
        }\sum_{l=1}^L \frac{192a_lb_{l-1}\hat\rho_l^2}{\beta_l^2\varepsilon^2}.
    \end{align*}

    \noindent Where we have $W = \max_{1\le l \le L}d_l$. Using Lagrange multiplier to optimize for the sum of non-logarithm terms over the choice of $\beta_l$, we obtain: 
    \begin{align*}
        \beta_l = \frac{(a_lb_{l-1}\hat\rho_{l})^{2/3}}{\sum_{m=1}^L (a_mb_{m-1}\hat\rho_{m})^{2/3}}.
    \end{align*}
    
    \noindent Therefore, the above bound becomes:
    \begin{align*}
        \log|\mathcal{C}_{1\to L}| &\le \frac{192\mathcal{\widehat{R}_A}^2}{\varepsilon^2}\log\biggRound{
            \biggRound{
                \frac{33\mathcal{\widehat{R}_A}}{\varepsilon} + 7
            }n(k+2)W
        }.
    \end{align*}

    \noindent Where $\mathcal{\widehat{R}_A}^{2/3} = \sum_{l=1}^L(a_lb_{l-1}\hat\rho_l)^{2/3}$. For any set of parameters $\pA=\bigRound{\bA{1}, \dots, \bA{L}}\in\mathcal{A}$, we can select a cover element $\bpA=\bigRound{\bA{1}, \dots, \bA{L}}\in\mathcal{C}_{1\to L}$ by inductively adding cover elements for layers $1$ to $L$ as follows:
    \begin{itemize}
        \item For $\A{1}$, choose $\bA{1}\in\mathcal{C}_1$ as the closest cover element with respect to the $L_{\infty, 2}$ metric.
        \item For $\A{l+1}$ where $1 \le l \le L-1$, choose $\bA{l+1}\in\mathcal{C}_{l+1}(\bpA^{1\to l})$ where $\bpA^{1\to l} = \bigRound{\bA{1}, \dots, \bA{l}}$ are the previously chosen cover elements such that $
            \max_{\tilde x\in \Sds_2(\bpA^{1\to l})}\Big\|\bigRound{\A{l+1} - \bA{l+1}}\bff{1}{l}(\tilde x)\Big\|_2 \le \varepsilon_{l+1}$ where the auxiliary dataset $\Sds_2(\bpA^{1\to l})$ is defined in equation \ref{eq:trimmed_dataset}.
    \end{itemize}

    \noindent We prove that with the above selection, we can bound $\|\tilde\ell_\pA - \tilde\ell_\bpA\|_{L_2(\Sds)}$ with an arbitrary granularity by controlling $\varepsilon$. Firstly, we acknowledge the fact that for the cover element $\bpA$ selected above, the following is satisfied by equation \ref{eq:trimmed_dataset_bound}:
    \begin{align}
        \label{eq:prop_A.8_eqA}
        \forall 1 \le l \le L-1, \xin_j \in \Sds(\bpA^{1\to l}) : \max_{\tilde x \in \xin_j} \Big\|\f^{1\to l+1}(\tilde x) - \barf^{1\to l+1}(\tilde x)\Big\|_2 \le \varepsilon b_{l+1}, \tag{A}
    \end{align}

    \noindent where the auxiliary dataset $\Sds(\bpA^{1\to l})$ is defined in equation \ref{eq:trimmed_dataset}. For any $\tilde \pA\in\mathcal{A}$, we define the set of input tuples indexed by $\tilde\pA$, $E_{\tilde \pA}^{(l)} = \bigCurl{\xin_j \in \Sds : \exists \tilde x \in \xin_j \text{ where } \|F_{\tilde \pA}^{1\to l}(\tilde x)\|_2 > 2b_l}$ for $1\le l \le L$. Then, we construct $\Sds'\subseteq \Sds$ as follows:
    \begin{align*}
        \Sds' &= \Sds \setminus \biggSquare{
            \biggRound{
                \bigcup_{l=1}^L E_\pA^{(l)}
            } \cap 
            \biggRound{
                \bigcup_{l=1}^L E_\bpA^{(l)}
            }
        }.
    \end{align*}

    \noindent In other words, we construct $\Sds'$ by removing input tuples from $\Sds$ that cause the augmented losses evaluated on both $\bpA$ and $\pA$ to collapse to $1$ due to large activations. By the same argument as proposition \ref{prop:all_layers_act_covering_num_2layers}, we have:
    \begin{align}
        \label{eq:prop_A.8_eqB}
        \Sds'\subseteq \bigcap_{l=1}^{L-1}\Sds(\bpA^{1\to l}), \tag{B}
    \end{align}

    \noindent because for all $\xin_j\in\Sds'$, we have either $\max_{\tilde x\in\xin_j}\|\f^{1\to l}(x)\|_2\le 2b_l$ or $\max_{\tilde x \in \xin_j}\|\barf^{1\to l}(\tilde x)\|_2 \le 2b_l$ for all $1\le l \le L$. When the former inequality occurs, we have $\max_{\tilde x \in \xin_j}\|\barf^{1\to l}(\tilde x)\|_2 \le 3b_l$ by the triangle inequality. Therefore, for all $\xin_j\in\Sds', \xin_j \in \Sds(\bpA^{1\to l})$ for all $1\le l \le L-1$, resulting in equation~\ref{eq:prop_A.8_eqB}. Then, from \ref{eq:prop_A.8_eqA} and \ref{eq:prop_A.8_eqB}, $\Sds'$ satisfies the following:
    \begin{align*}
        \forall \xin_j \in \Sds', 1 \le l \le L: \max_{\tilde x \in \xin_j} \Big\|\f^{1\to l}(\tilde x) - \barf^{1\to l}(\tilde x)\Big\|_2 \le \varepsilon b_l.
    \end{align*}

    \noindent Therefore, for all $\xin_j=\bigRound{x_j, x_j^+, x_{j1:k}^-}\in \Sds'$, we have:
    \begin{align*} 
        &\bigAbs{\tilde\ell_\pA(x_j, x_j^+, x_{j1:k}^-)-\tilde\ell_\bpA(x_j, x_j^+, x_{j1:k}^-)} \\ &\le \max\biggSquare{
            \bigAbs{\ell(V_{\pA, j}) - \ell(V_{\bpA, j})}, \max_{\tilde x\in \xin_j}\max_{1\le l \le L}\bigAbs{
                \lambda_{b_l}(\|\f^{1\to l}(\tilde x)\|_2) - \lambda_{b_l}(\|\barf^{1\to l}(\tilde x)\|_2)
            }
        }\\
        &\le \max\biggSquare{
            \bigAbs{\ell(V_{\pA, j}) - \ell(V_{\bpA, j})}, \max_{\tilde x\in \xin_j}\max_{1\le l \le L}\frac{1}{b_l}\bigAbs{
                \|\f^{1\to l}(\tilde x)\|_2 - \|\barf^{1\to l}(\tilde x)\|_2
            }
        } \\
        &\le \max\biggSquare{
            \bigAbs{\ell(V_{\pA, j}) - \ell(V_{\bpA, j})}, \max_{\tilde x\in \xin_j}\max_{1\le l \le L}\frac{1}{b_l}\Big\|\f^{1\to l}(\tilde x) - \barf^{1\to l}(\tilde x)\Big\|_2
        } \\
        &\le \max\biggSquare{
            \bigAbs{\ell(V_{\pA, j}) - \ell(V_{\bpA, j})}, \max_{\tilde x\in \xin_j}\max_{1\le l \le L}\frac{1}{b_l}\cdot\varepsilon b_l
        } \\
        &\le \max\bigSquare{10\eta b_L^2\varepsilon, \varepsilon} \ \ \ (\text{Proposition } \ref{prop:last_act_augmented_loss_covering_number}) \\
        &= 10\eta b_L^2\varepsilon. \ \ \ (\text{Since } \eta\ge1, b_L\ge1)
    \end{align*}

    \noindent As a result, we have:
    \begin{align*}
        \|\tilde\ell_\pA - \tilde\ell_\bpA\|_{L_2(\Sds)} &= \sum_{j=1}^n \bigAbs{
            \tilde\ell_\pA(x_j, x_j^+, x_{j1:k}^-) - \tilde\ell_\bpA(x_j, x_j^+, x_{j1:k}^-)
        } \\
        &= \sum_{\xin_j \in \Sds'}\bigAbs{
            \tilde\ell_\pA(x_j, x_j^+, x_{j1:k}^-) - \tilde\ell_\bpA(x_j, x_j^+, x_{j1:k}^-)
        } \\
        &\le 10\eta b_L^2\varepsilon.
    \end{align*}

    \noindent Then, for a desired cover granularity $\epsilon\in(0,1)$, we set $\varepsilon=\epsilon/10\eta b_L^2$. We have:
    \begin{align*}
        \log\mathcal{N}\bigRound{\tilde \G, \epsilon, L_2(\Sds)} &\le \log|\mathcal{C}_{1\to L}| \\
            &\le \frac{192\mathcal{\widehat{R}_A}^2}{(\epsilon/10\eta b_L^2)^2}\log\biggRound{
            \biggRound{
                \frac{33\mathcal{\widehat{R}_A}}{(\epsilon/10\eta b_L^2)} + 7
            }n(k+2)W
        } \\
        &= \frac{19200\eta^2b_L^4\mathcal{\widehat{R}_A}^2}{\epsilon^2}\log\biggRound{
            \biggRound{
                \frac{330\eta b_L^2\mathcal{\widehat{R}_A}}{\epsilon} + 7
            }n(k+2)W
        }.
    \end{align*}

    \noindent Hence, we obtain the desired bound. With this result, we can move on to proving theorem \ref{thm:all_act_augmentation}.
\end{proof}

\begin{proof}[Proof of Theorem \ref{thm:all_act_augmentation}]
    Let $\tilde \G$ be the augmented loss function class defined in proposition \ref{eq:covering_number_all_act_augmentation}. Let $\mathrm{L_{un}^{aug}}(\f)$ and $\mathrm{\widehat{L}_{un}^{aug}}(\f)$ be the population and empirical augmented unsupervised risks. We have:
    \begin{align*}
        \mathrm{L_{un}^{aug}}(\f) - \mathrm{\widehat{L}_{un}^{aug}}(\f) &\le 2 \ERC_{\Sds}(\tilde \G) + 3\sqrt{\frac{\log 2/\delta}{2n}}.
    \end{align*}

    \noindent Using Dudley's entropy integral with the choice of $\alpha=1/n$, we have:
    \begin{align*}
        \ERC_{\Sds}(\tilde\G) &\le 4\alpha + \frac{12}{\sqrt n}\int_\alpha^1 \sqrt{\log\mathcal{N}\bigRound{\tilde \G, \epsilon, L_2(\Sds)}} d\epsilon \\
        &\le 4\alpha + \frac{960\sqrt{3}\eta b_L^2 \mathcal{\widehat{R}_A}}{\sqrt{n}}\log^{\frac{1}{2}}\biggRound{\biggRound{
                \frac{330\eta b_L^2 \mathcal{\widehat{R}_A}}{\alpha} + 7
            }n(k+2)W
        }\int_\alpha^1\frac{1}{\epsilon}d\epsilon \\
        &= \frac{4}{n} + \frac{960\sqrt{3}\eta b_L^2 \mathcal{\widehat{R}_A}}{\sqrt{n}}\log^{\frac{1}{2}}\bigRound{\bigRound{
                {330\eta b_L^2 \mathcal{\widehat{R}_A}n} + 7
            }n(k+2)W
        }\log(n).
    \end{align*}

    \noindent Plugging the above back to the Rademacher bound, we have:
    \begin{align*}
        \mathrm{L_{un}^{aug}}(\f) - \mathrm{\widehat{L}_{un}^{aug}}(\f) &\le \frac{8}{n} + \frac{1920\sqrt{3}\eta b_L^2 \mathcal{\widehat{R}_A}}{\sqrt{n}}\log^{\frac{1}{2}}\bigRound{\bigRound{
                {330\eta b_L^2 \mathcal{\widehat{R}_A}n} + 7
            }n(k+2)W
        }\log(n)+ 3\sqrt{\frac{\log 2/\delta}{2n}} \\
        &\le \tilde \bigO\biggRound{
            \frac{\eta b_L^2\mathcal{\widehat{R}_A}}{\sqrt n}\log(W)
        } + 3\sqrt{\frac{\log 2/\delta}{2n}}. 
    \end{align*}

    \noindent Furthermore, we have:
    \begin{align*}
        \mathrm{\widehat{L}_{un}^{aug}}(\f) 
        &= \frac{1}{n}\sum_{j=1}^n \max\bigSquare{
            \ell(V_{\pA, j}), \max_{\tilde x \in \xin_j}\max_{1\le l \le L}\lambda_{b_l}\bigRound{
                \|\f^{1\to l}(\tilde x)\|_2
            }
        }\\
        &\le \frac{1}{n}\sum_{j=1}^n \bigSquare{\ell(V_{\pA, j}) + \boldsymbol{1}\Big\{\exists \tilde x \in \xin_j, 0 \le l\le L : \|\f^{1\to l}(\tilde x)\|_2 > b_l \Big\}} \\
        &= \frac{1}{n}\sum_{j=1}^n \ell(V_{\pA, j}) + \frac{\mathcal{I}_{\pA, {\bf B}}}{n} \\
        &= \Lunhat(\f) + \frac{\mathcal{I}_{\pA, {\bf B}}}{n}.
    \end{align*}

    \noindent Therefore, we have:
    \begin{align*}
        \Lun(\f) &\le \mathrm{L_{un}^{aug}}(\f) \\
            &\le \mathrm{\widehat{L}_{un}^{aug}}(\f) + \tilde \bigO\biggRound{
            \frac{\eta b_L^2\mathcal{\widehat{R}_A}}{\sqrt n}\log(W)
        } + 3\sqrt{\frac{\log 2/\delta}{2n}} \\
        &\le \mathrm{\widehat{L}_{un}^{aug}}(\f) + \frac{\mathcal{I}_{\pA, {\bf B}}}{n} + \tilde \bigO\biggRound{
            \frac{\eta b_L^2\mathcal{\widehat{R}_A}}{\sqrt n}\log(W)
        } + 3\sqrt{\frac{\log 2/\delta}{2n}}.
    \end{align*}

    \noindent Hence, we obtained the desired bound.
\end{proof}

\section{Parameter-Counting Bound}
In this section, we prove Theorem~\ref{thm:paracount_bound} by applying a parameter counting argument to the class $\F_\mathcal{A}$ before applying similar techniques as in the previous sections to propagate the cover and granularity through to the loss class $\G$. 
\subsection{Some Supporting Results}
In the following section, we adopt a similar technique as \cite{article:long2020generalization} to a obtain generalization bound that is dominantly dependent on the number of parameters of the neural networks. We consider the following key lemma:

\begin{lemma}{\cite{article:long2020generalization}.}
    \label{lem:covering_ball}
    The covering number of a $d$-dimensional ball with radius $\kappa$, $B_\kappa$, with respect to any norm $\|\cdot\|$ is bounded by:
    \begin{align}
        \mathcal{N}\bigRound{B_\kappa, \epsilon, \|\cdot\|} \le \Bigg\lceil 
            \frac{3\kappa}{\epsilon}
        \Bigg\rceil^d \le \biggRound{1 + \frac{3\kappa}{\epsilon}}^d.
    \end{align}
\end{lemma}

\noindent Next, we will use the above lemma to prove the following:
\begin{proposition}
    Given an arbitrary dataset $S=\{x_1, \dots, x_n\} \in \X^n$. Then, we have:
    \begin{align}
        \log\mathcal{N}\bigRound{
            {\F_\mathcal{A}}, \epsilon, L_{\infty, 2}(S)
        } \le \mathcal{W} \log\biggRound{
            1 + \frac{6LB_x\prod_{m=1}^L\rho_ms_m}{\epsilon}
        }.
    \end{align}
    \noindent Where $\mathcal{W}=\sum_{l=1}^L d_l$ is the overall size of the neural network.
\end{proposition}

\begin{proof}
    We use the notation $\f$ to denote a neural network that is parameterized by a tuple of weight matrices where $\pA=(\A{1}, \dots, \A{L}) \in \mathcal{B}_1 \times \dots \times \mathcal{B}_L$. For $1 \le l \le L$, denote the following sets of matrices: 
    \begin{align*}
        \tilde{\bf A}_l&=(\A{1}, \dots, \tA{l}, \dots, \tA{L}), \\
        \tilde{\bf A}_{l+1}&=(\A{1}, \dots, \A{l}, \tA{l+1}, \dots, \tA{L}),
    \end{align*}
    
    \noindent where $\tA{m}\in\mathcal{B}_m$ for all $l \le m \le L$ and $\A{u}\in\mathcal{B}_u$ for all $1\le u \le l$. From the above definitions, for all $x\in\X$, we have:
    \begin{align*}
        F_{\tilde{\bf A}_l}^{l+1\to L}(z) &= F_{\tilde{\bf A}_{l+1}}^{l+1\to L}(z), \ \forall z \in \R^{d_l} \\
        \text{and }
        F_{\tilde{\bf A}_l}^{1\to l-1}(x) &= F_{\tilde{\bf A}_{l+1}}^{1\to l-1}(x), \ \forall x\in\X.
    \end{align*}

    \noindent For any $1 \le i \le n$, we have:
    \begin{align*}
        \Big\|F_{\tilde{\bf A}_l}(x_i) - F_{\tilde{\bf A}_{l+1}}(x_i)\Big\|_2 
        &= \Big\|
            F_{\tilde{\bf A}_l}^{l+1\to L} \circ \sigma_l\bigRound{\A{l}F_{\tilde{\bf A}_l}^{1\to l-1}(x_i)} - F_{\tilde{\bf A}_{l+1}}^{l+1\to L} \circ \sigma_l\bigRound{\tA{l}F_{\tilde{\bf A}_{l+1}}^{1\to l-1}(x_i)}
        \Big\|_2 \\
        &\le \rho_l \prod_{m=l+1}^L \rho_ms_m \Big\|
            \A{l}F_{\tilde{\bf A}_l}^{1\to l-1}(x_i) - \tA{l}F_{\tilde{\bf A}_{l+1}}^{1\to l-1}(x_i)
        \Big\|_2 \\
        &\le \rho_l \prod_{m=l+1}^L \rho_ms_m  \Big\| \A{l} - \tA{l}\Big\|_\sigma \cdot \Big\|
            F_{\tilde{\bf A}_l}^{1\to l-1}(x_i)
        \Big\|_2 \\
        &\le \frac{1}{s_l} \prod_{m=1}^L \rho_ms_m \Big\| \A{l} - \tA{l}\Big\|_\sigma \cdot \|x_i\|_2 \\
        &\le \frac{B_x}{s_l}\prod_{m=1}^L \rho_ms_m \Big\| \A{l} - \tA{l}\Big\|_\sigma. 
    \end{align*}

    \noindent Then, for $\pA=(\A{1}, \dots, \A{l})$ and $\tilde\pA=(\tA{1}, \dots, \tA{L})$ satisfying $\A{l}, \tA{l}\in\mathcal{B}_l$ for all $1\le l \le L$, by the triangle inequality, \footnote{We adopt an identical approach as \cite{article:long2020generalization} used in lemma 2.5 and lemma 2.6 of their paper.}for any $1\le i \le n$, we have:
    \begin{align*}
        \Big\|\f(x_i) - F_{\tilde \pA}(x_i)\Big\|_2 
        &= \Bigg\| \sum_{l=1}^L F_{\tilde{\bf A}_l}(x_i) - F_{\tilde{\bf A}_{l+1}}(x_i) \Bigg\|_2 \\
        &\le \sum_{l=1}^{L}\Big\| F_{\tilde{\bf A}_l}(x_i) - F_{\tilde{\bf A}_{l+1}}(x_i) \Big\|_2 \\
        &\le B_x\prod_{m=1}^L \rho_ms_m \sum_{l=1}^L \frac{\|\A{l} - \tA{l}\|_\sigma}{s_l}.
    \end{align*}

    \noindent Let $\varepsilon > 0$. For $1 \le l \le L$, we set:
    \begin{align*}
        \varepsilon_l = \frac{\beta_ls_l\varepsilon}{B_x\prod_{m=1}^L\rho_ms_m} \text{ where } \ \ \ \sum_{l=1}^L\beta_l = 1.
    \end{align*}

    \noindent Using lemma \ref{lem:covering_ball}, we construct $L$ $\varepsilon_l$-covers, denoted $\mathcal{C}_l(\mathcal{B}_l, \varepsilon_l)$, for parameter spaces $\mathcal{B}_1, \dots, \mathcal{B}_L$ with respect to matrix spectral norm $\|.\|_\sigma$. We have:
    \begin{align*}
        \bigAbs{\mathcal{C}_l(\mathcal{B}_l, \varepsilon_l)} &\le \biggRound{
            1 + \frac{3s_l}{\varepsilon_l}
        }^{d_l} = \biggRound{
            1 + \frac{3B_x\prod_{m=1}^L \rho_ms_m}{\beta_l\varepsilon}
        }^{d_l}.
    \end{align*}

    \noindent Then, for any $\A{l}\in\mathcal{B}_l \ (1\le l \le L)$, there exists $\bA{l}\in \mathcal{C}_l(\mathcal{B}_l, \varepsilon_l)$ such that $\|\A{l} - \bA{l}\|_\sigma \le \varepsilon_l$. By the triangle inequality:
    \begin{align*}
        \Big\|(\bA{l} - \M{l})^\top\Big\|_{2,1} &\le \Big\|(\A{l} - \M{l})^\top\Big\|_{2,1} + \Big\|(\bA{l} - \A{l})^\top\Big\|_{2,1} \\
            &\le a_l + \Big\| (\bA{l} - \A{l})^\top \Big\|_\sigma\sqrt{d_l\cdot\mathrm{rank}(\A{l} - \bA{l})} \ \ \ (\text{Lemma }\ref{lem:l21_and_spectral_norm}) \\
            &\le a_l + \varepsilon_l \sqrt{d_{l}\cdot\mathrm{rank}(\A{l} - \bA{l})}.
    \end{align*}

    \noindent Hence, the cover elements in $\mathcal{C}_l(\mathcal{B}_l, \varepsilon_l)$ might not satisfy the matrix $(2,1)$-norm constraint and thus might not be contained in $\mathcal{B}_l$. To make the covers internal, define a sequence $\{\epsilon_l\}_{l=1}^L$ where $\epsilon_l=2\varepsilon_l$. Then, by lemma \ref{lem:internal_external_cover}, we have:
    \begin{align*}
        \mathcal{N}\bigRound{\mathcal{B}_l, \epsilon_l, \|.\|_\sigma} \le \bigAbs{\mathcal{C}_l(\mathcal{B}_l, \varepsilon_l)} \le \biggRound{
            1 + \frac{3B_x\prod_{m=1}^L \rho_ms_m}{\beta_l\varepsilon}
        }^{d_l}.
    \end{align*}
    
    \noindent We construct the final cover for $\mathcal{B}_1\times\dots\times\mathcal{B}_L$ by taking the Cartesian product $\mathcal{C}_\varepsilon=\mathcal{C}_1(\mathcal{B}_1, \varepsilon_1)\times\dots\times\mathcal{C}_L(\mathcal{B}_L, \varepsilon_L)$. Then, for any $\pA=(\A{1}, \dots, \A{L})\in\mathcal{B}_1\times\dots\times\mathcal{B}_L$, there exists $\tilde \pA = (\tA{1}, \dots, \tA{L})\in\mathcal{C}_\varepsilon$ such that the following is satisfied:
    \begin{itemize}
        \item For $1 \le l \le L: \tA{l} \in \mathcal{B}_l \ $  ($\tA{l}$ is an internal cover element).
        \item For $1 \le l \le L: \|\A{l} - \tA{l}\|_\sigma \le \epsilon_l = 2\varepsilon_l$.
    \end{itemize}

    \noindent Then, for any $1 \le i \le n$, we have:
    
    \begin{align*}
        \Big\|\f(x_i) - F_{\tilde \pA}(x_i)\Big\|_2 &\le B_x\prod_{m=1}^L \rho_ms_m \sum_{l=1}^L \frac{\|\A{l} - \tA{l}\|_\sigma}{s_l} \\
            &\le B_x\prod_{m=1}^L \rho_ms_m \sum_{l=1}^L \frac{\epsilon_l}{s_l} \\
            &= 2B_x\prod_{m=1}^L \rho_ms_m \sum_{l=1}^L \frac{\varepsilon_l}{s_l} \\
            &= 2B_x\prod_{m=1}^L \rho_ms_m \sum_{l=1}^L \frac{\beta_ls_l\varepsilon}{s_lB_x\prod_{m=1}^L \rho_ms_m} \\
            &= 2\sum_{l=1}^L \beta_l \varepsilon = 2\varepsilon.
    \end{align*}

    \noindent Then, we have $\mathcal{C}_\varepsilon$ corresponds to the internal $2\varepsilon$-cover of $\F_\mathcal{A}$ (restricted to $S$) with respect to the $L_{\infty, 2}$ metric. Hence, we have:
    \begin{align*}
        \log\mathcal{N}\bigRound{{\F_\mathcal{A}}, 2\varepsilon, L_{\infty, 2}(S)} &\le 
            \log|\mathcal{C}_\varepsilon| \\ 
            &\le \sum_{l=1}^L \log|\mathcal{C}_l(\mathcal{B}_l, \varepsilon_l)| \\
            &\le \sum_{l=1}^L d_l \log\biggRound{
            1 + \frac{3B_x\prod_{m=1}^L \rho_ms_m}{\beta_l\varepsilon}
        }.
    \end{align*}

    \noindent Setting $\epsilon = 2\varepsilon$ and $\beta_l = L^{-1}$ for all $1 \le l \le L$\footnote{We can tighten the bound by using Lagrange multiplier to find the optimal set of weights $\{\beta_l\}_{l=1}^L$. However, since $\beta_l$ only appears inside the logarithm term, we can afford to be somewhat less stringent in our selection.}:
    \begin{align*}
        \log\mathcal{N}\bigRound{{\F_\mathcal{A}}, \epsilon, L_{\infty, 2}(S)} \le \log\biggRound{
            1 + \frac{6LB_x\prod_{m=1}^L\rho_ms_m}{\epsilon}
        } \sum_{l=1}^L d_l = \mathcal{W} \log\biggRound{
            1 + \frac{6LB_x\prod_{m=1}^L\rho_ms_m}{\epsilon}
        }.
    \end{align*}
\end{proof}

\begin{lemma}{\cite{book:matrixanalysis2012}.}
    \label{lem:l21_and_spectral_norm}
    Let $A \in \R^{m\times n}$. We have $\|A\|_{2,1} \le \|A\|_\sigma\sqrt{rn}$ where $r=\mathrm{rank}(A)$.
\end{lemma}

\begin{proof}
    For $A\in\R^{m\times n}$ and denote $A_{:, i}$ as the $i^{th}$ column ($1\le i \le n$) of the matrix. Then, we have:
    \begin{align*}
        \|A\|_{2,1} 
        &= \sum_{i=1}^n \|A_{:, i}\|_2 = \sum_{i=1}^n 1\times\|A_{:, i}\|_2 \\
        &\le \biggRound{\sum_{i=1}^n 1^2}^{1/2}\cdot \biggRound{\sum_{i=1}^n \|A_{:, i}\|_2^2}^{1/2} \ \ \ (\text{Cauchy-Schwarz Inequality}) \\
        &= \|A\|_F\sqrt{n}.
    \end{align*}

    \noindent Now, we use the identity $\|A\|_F = \mathrm{trace}(A^\top A)^{1/2}$. Let $r=\mathrm{rank}(A)$ and denote $\{\sigma_i\}_{i=1}^r$ as the distinct singular values of $A$, we have:
    \begin{align*}
        \|A\|_F &= \mathrm{trace}(A^\top A)^{1/2} = \biggRound{\sum_{i=1}^r \sigma_i^2}^{1/2} \\
            &\le \sqrt{r} \cdot \max_{1\le i\le r}\sigma_i \\
            &= \sqrt{r} \cdot \|A\|_\sigma.
    \end{align*}

    \noindent From all of the obtained inequalities, we have:
    \begin{align*}
        \|A\|_{2,1} &\le \|A\|_F \sqrt{n} \le \|A\|_\sigma \sqrt{rn}.
    \end{align*}
\end{proof}

\begin{lemma}
    \label{lem:internal_external_cover}
    Let $(X, \|.\|)$ be a normed space and let $V\subset X$. Let $\epsilon > 0$ and denote $\mathcal{N}(V, \epsilon, \|.\|)$ as the (internal) covering number of $V$, $\mathcal{N}^{ext}(V, \epsilon, \|.\|)$ as the external covering number of $V$. We have:
    \begin{align}
        \mathcal{N}\bigRound{V, \epsilon, \|.\|} \le \mathcal{N}^{ext}\bigRound{V, \epsilon/2, \|.\|}.
    \end{align} 
\end{lemma}

\begin{proof}
    Let $\mathcal{C}_{\epsilon/2}^{ext}\bigRound{V, \|.\|}=\bigCurl{v_1^{ext}, \dots, v_{N_0}^{ext}}$ be the minimal $\epsilon/2$-cover of $V$ with respect to the norm $\|.\|$ where $N_0 = \bigAbs{\mathcal{C}_{\epsilon/2}^{ext}\bigRound{V, \|.\|}} = \mathcal{N}^{ext}\bigRound{V, \epsilon/2, \|.\|}$. We have:
    \begin{align*}
        V \subseteq \bigcup_{i=1}^{N_0} \mathcal{B}_{\epsilon/2}(v_i^{ext}).
    \end{align*}

    \noindent Where for $\epsilon>0$, $\mathcal{B}_{\epsilon}(x)$ is the $\epsilon$-ball centered around $x$. For every $v_i^{ext}\in \mathcal{C}_{\epsilon/2}^{ext}\bigRound{V, \|.\|}$, we have:
    \begin{align*}
        V \cap \mathcal{B}_{\epsilon/2}(v_i^{ext}) \ne \emptyset.
    \end{align*}

    \noindent Otherwise, $v_i^{ext}$ is redundant which contradicts the fact that $\mathcal{C}_{\epsilon/2}^{ext}\bigRound{V, \|.\|}$ is a minimum external $\epsilon/2$-cover of $V$. Hence, for all $v_i^{ext}$, we have:
    \begin{align*}
        \exists v_i^{in} \in V \cap \mathcal{B}_{\epsilon/2}(v_i^{ext}) : \mathcal{B}_{\epsilon/2}(v_i^{ext}) \subset \mathcal{B}_{\epsilon}(v_i^{in}).
    \end{align*}

    \noindent Therefore, we have:
    \begin{align*}
        V \subseteq \bigcup_{i=1}^{N_0} \mathcal{B}_{\epsilon/2}(v_i^{ext}) \subseteq \bigcup_{i=1}^{N_0}\mathcal{B}_\epsilon(v_i^{in}).
    \end{align*}

    \noindent Notice that from the above, it is possible to cover $V$ with fewer than $N_0$ $\epsilon$-balls $\mathcal{B}_\epsilon(v_i^{in})$. Hence, we have:
    \begin{align*}
        \mathcal{N}\bigRound{V, \epsilon, \|.\|} \le N_0 = \mathcal{N}^{ext}\bigRound{V, \epsilon/2, \|.\|}.
    \end{align*}
\end{proof}

\subsection{Proof of Theorem \ref{thm:paracount_bound}}
\begin{proof}
    Using a similar argument as the proof of theorem \ref{thm:basic_bound}, for $\epsilon>0$, we have:
    \begin{align*}
        \log\mathcal{N}\bigRound{\G, \epsilon, L_2(\Sds)} &\le \log\mathcal{N}\bigRound{\Hf, \epsilon/\eta, L_\infty(\Sds_1)} \\
        &\le \log\mathcal{N}\bigRound{{\F_\mathcal{A}}, \epsilon/(4\eta B_L), L_{\infty, 2}(\Sds_2)} \\
        &\le \mathcal{W}\log\biggRound{
            1 + \frac{24\eta L B_xB_L\prod_{m=1}^L\rho_ms_m}{\epsilon}
        }.
    \end{align*}

    \noindent Using Dudley's entropy integral with the choice $\alpha=1/n$, we have:
    \begin{align*}
        \ERC_{\Sds}(\G) &\le 4\alpha + \frac{12}{\sqrt n}\int_\alpha^{M} \sqrt{\log\mathcal{N}\bigRound{\G, \epsilon, L_2(\Sds)}}
        d\epsilon \\
        &\le 4\alpha + 12\sqrt{\frac{\mathcal{W}}{n}}\int_\alpha^{M}\log^{\frac{1}{2}}\biggRound{
            1 + \frac{24\eta LB_xB_L\prod_{m=1}^L\rho_ms_m}{\epsilon} 
        }
        d\epsilon \\
        &\le 4\alpha + 12\sqrt{\frac{\mathcal{W}}{n}}\log^{\frac{1}{2}}\biggRound{
                1 + \frac{24\eta LB_xB_L\prod_{m=1}^L\rho_ms_m}{\alpha} 
            }
        (M - \alpha) \\
        &= \frac{4}{n} + 12\sqrt{\frac{\mathcal{W}}{n}}\log^{\frac{1}{2}}\biggRound{
                1 + 24\eta LB_xB_L n \prod_{m=1}^L \rho_ms_m
            }
        (M - 1/n) \\
        &\le \frac{4}{n} + 12M\sqrt{\frac{\mathcal{W}}{n}}\log^{\frac{1}{2}}\biggRound{
                1 + 24\eta LB_xB_L n \prod_{m=1}^L \rho_ms_m
            }.
    \end{align*}

    \noindent Using Rademacher complexity bound, we have:
    \begin{align*}
        \Lun(\f) - \Lunhat(\f) &\le 2\ERC_{\Sds}(\G) + 3M\sqrt{\frac{\log2/\delta}{2n}} \\
        &\le \frac{8}{n} + 24M\sqrt{\frac{\mathcal{W}}{n}}\log^{\frac{1}{2}}\biggRound{
                1 + 24\eta LB_xB_L n \prod_{m=1}^L \rho_ms_m
            }
        + 3M\sqrt{\frac{\log2/\delta}{2n}} \\
        &\le \bigO\biggRound{
            M\sqrt{\frac{\mathcal{W}}{n}}\log\biggRound{
                1 + 24\eta LB_xB_Ln \prod_{m=1}^L\rho_ms_m
            }
        } + 3M\sqrt{\frac{\log 2/\delta}{2n}}.
    \end{align*}

    \noindent Using the estimation $B_L \le B_x\prod_{l=1}^L\rho_ls_l$, we have:
    \begin{align*}
        \Lun(\f) - \Lunhat(\f) &\le \bigO\biggRound{
            M\sqrt{\frac{\mathcal{W}}{n}}\log\biggRound{
                1 + 24\eta LnB_x^2 \prod_{m=1}^L\rho_m^2s_m^2
            }
        } + 3M\sqrt{\frac{\log 2/\delta}{2n}}.
    \end{align*}
\end{proof}

As a consequence, we have the following application for the ramp loss:

\begin{corollary}
    Let $\ell_\gamma:\R^k \to [0, 1]$ be the ramp loss with margin $\gamma\in(0,1)$, defined as follows:
    \begin{align}
        \ell_\gamma(v) &= \begin{cases}
            0 & r_v < -\gamma \\
            1 + r_v/\gamma & r_v \in [-\gamma, 0] \\
            1 & r_v > 0
        \end{cases}, \\ \text{where } r_v &= \max\bigCurl{0, 1 + \max_{1\le i \le k}\{-v_i\}}.
    \end{align}

    \noindent Then, for any $\f\in\F_\mathcal{A}$ and $\delta\in(0,1)$. With probability of at least $1-\delta$, we have:
    \begin{align}
        \Lun(\f) - \Lunhat(\f) \le \bigO\biggRound{
            \sqrt{\frac{\mathcal{W}}{n}}\log\biggRound{
                1 + 24\gamma^{-1} LnB_x^2\prod_{m=1}^L\rho_m^2s_m^2
            }
        } + 3\sqrt{\frac{\log 2/\delta}{2n}}.
    \end{align}
\end{corollary}

\begin{proof}
    For $v_1, v_2\in\R^k$, we have:
    \begin{align*}
        \bigAbs{\ell_\gamma(v_1) - \ell_\gamma(v_2)} 
            &\le \frac{1}{\gamma} |r_{v_1} - r_{v_2}| \le \frac{1}{\gamma}\|v_1 - v_2\|_\infty.
    \end{align*}

    \noindent Hence, $\ell_\gamma$ is $\ell^\infty$-Lipschitz with constant $\eta=\gamma^{-1}$. Then, by theorem \ref{thm:paracount_bound}, replacing $M=1$ and $\eta=\gamma^{-1}$, we obtain the desired bound.
\end{proof}

\section{Post-hoc Analysis}
\label{app:post_hoc_analysis}
In this Section, we show how to translate most of our results to post hoc version through a union bound: in particular, a priori norm constraints such as $a_l$ can be replaced by observed post training values such as $\|(\A{l} - \M{l})^\top\|_{2,1}$ at the mild cost of additional logarithmic dependencies. The techniques for this section are classic \citep{article:bartlett2017spectrallynormalized,article:ledent2021normbased} and reproduced merely to better illustrate the strength of our final results. 
\label{sec:posthoc}
\begin{lemma}{\cite{article:ledent2021normbased}.}
    \label{lem:post_hoc_key_lemma_1}
    Let $R_A$ be a random variable indexed by a set of parameters $A\in\mathcal{A}$. Let $\varphi_1(A), \dots, \varphi_{M}(A)$ and $\phi_1(A), \dots, \phi_{N}(A)$ be positive statistics of $A$. Let $F:\R^{(M\times N)}\to \R_+$ be a function satisfying:
    \begin{itemize}
        \item $F$ monotonically \textit{increases} in $\varphi_i, \ (1\le i \le M)$.
        \item $F$ monotonically \textit{decreases} in $\phi_i, \ (1\le i \le N)$.
        \item For any $\varphi_1, \dots, \varphi_{M}, \phi_1, \dots, \phi_{N}$ and $\delta \in (0,1)$. For any $A\in\mathcal{A}$ such that $\varphi_i(A) \le \varphi_i, \ (1\le i \le M)$ and $\phi_i(A)\le \phi_i, \ (1 \le i \le N)$, we have:
        \begin{align}
            \mathbb{P}\biggRound{R_A \le F(\varphi_1, \dots, \varphi_{M}, \phi_1, \dots, \phi_{N}) + C_1\sqrt{
                \frac{\log 1/\delta}{C_2}
            }} \ge 1 - \delta.
        \end{align}
    \end{itemize}

    \noindent Fix the choices of positive constants $\beta_1, \dots, \beta_M$ and $\alpha_1, \dots, \alpha_N$ a priori. For any $A\in\mathcal{A}$, with probability of at least $1-\delta$, we have:

    \begin{equation}
    \begin{aligned}
        R_A 
        &\le 
        F\biggRound{\varphi_1(A) + \frac{1}{\beta_1}, \dots, \varphi_M(A) + \frac{1}{\beta_M}, \phi_i(A), \dots, \phi_N(A)} \\
        &\hspace{4cm} + \frac{C_1}{\sqrt{C_2}}\sqrt{
            \log1/\delta + \sum_{i=1}^N \log\biggRound{\frac{\alpha_i}{\phi_i(A)}} + 2\sum_{i=1}^M \log\bigRound{\beta_i\varphi_i(A) + 2}
        }.
    \end{aligned}
    \end{equation}
\end{lemma}

\begin{proof}
    Let any $k_1, \dots, k_M, j_1, \dots, j_N \in\mathbb{N}$. For $\delta\in (0,1)$, define:
    \begin{align*}
        \delta(k_1, \dots, k_M, j_1, \dots, j_N ) = \frac{\delta}{\prod_{i=1}^N 2^{j_i}\prod_{i=1}^M k_i(k_i + 1)}.
    \end{align*}

    \noindent\textbf{1. Union bound argument:} Let $\varphi_{1}^{(k_1)}, \dots, \varphi_{M}^{(k_M)}, \phi_1^{(j_1)}, \dots, \phi_N^{(j_N)}$ be constants that depend on the choice of integers. 
    By assumption, for any $A\in\mathcal{A}_{k_1, \dots, k_M, j_1, \dots, j_N}=\bigCurl{A\in\mathcal{A}: \varphi_i(A) \le \varphi_i^{(k_i)}, \ \phi_i(A) \le \phi_i^{(j_i)}}$, with probability of at least $1 - \delta(k_1, \dots, k_M, j_1, \dots, j_N)$, we have:
    \begin{align*}
        R_A &\le F(\varphi_1^{(k_1)}, \dots, \varphi_M^{(k_M)}, \phi_1^{(j_1)}, \dots, \phi_N^{(j_N)}) + \frac{C_1}{\sqrt{C_2}}\cdot\sqrt{
            \log 1/\delta(k_1, \dots, k_M, j_1, \dots, j_N)
        } \\
        &\le F(\varphi_1^{(k_1)}, \dots, \varphi_M^{(k_M)}, \phi_1^{(j_1)}, \dots, \phi_N^{(j_N)}) + \frac{C_1}{\sqrt{C_2}}\cdot\sqrt{
            \log 1/\delta + \sum_{i=1}^Nj_i\log(2) + 2\sum_{i=1}^M \log(k_i+1)
        }. \\
    \end{align*}

    \noindent Notice that $\sum_{k_1, \dots, k_M, j_1, \dots, j_N}\delta(k_1, \dots, k_M, j_1, \dots, j_N) = \delta$. Hence, by the union bound, with probability of at least $1-\delta$, the above bound applies for \underline{any choice of} $k_1, \dots, k_M, j_1, \dots, j_N \in \mathbb{N}$. Hence, to make the bound hold for an arbitrary $A\in\mathcal{A}$, choose the smallest set of $k_1, \dots, k_M, j_1, \dots, j_N$ such that $A\in\mathcal{A}_{k_1, \dots, k_M, j_1, \dots, j_N}$ then plug the corresponding constants $\varphi_{1}^{(k_1)}, \dots, \varphi_{M}^{(k_M)}, \phi_1^{(j_1)}, \dots, \phi_N^{(j_N)}$ to the bound.

    \noindent\newline\textbf{2. Choosing $\varphi_{1}^{(k_1)}, \dots, \varphi_{M}^{(k_M)}, \phi_1^{(j_1)}, \dots, \phi_N^{(j_N)}$:} Define the constants $\varphi_i^{(k_i)}$ and $\phi_i^{(j_i)}$ as follows:
    \begin{align*}
        \varphi_i^{(k_i)} &= \beta_i^{-1} k_i, \ (1\le i \le M) \\
        \phi_i^{(j_i)} &= \alpha_i 2^{-j_i}, \ (1 \le i \le N)
    \end{align*}

    \noindent For an arbitrary $A\in\mathcal{A}$, choose the largest $j_1, \dots, j_N$ such that $2^{j_i} \le \frac{\alpha_i}{\phi_i(A)}$. Hence, we have:
    \begin{align*}
        \phi_i(A) &\le \alpha_i 2^{-j_i} = \phi_i^{(j_i)}.
    \end{align*}

    \noindent Furthermore, choose the smallest $k_1, \dots, k_M$ such that $k_i \ge \beta_i\varphi_i(A)$. Hence, we have:
    \begin{align*}
        &\beta_i\varphi_i(A) \le k_i \le \beta_i\varphi_i(A) + 1. \\
        \implies &\varphi_i(A) \le \varphi_i^{(k_i)} \le \varphi_i(A) + \beta_i^{-1}. \ \ \ (\text{Divide both sides by }\beta_i)
    \end{align*}

    \noindent Clearly, for such choice of $k_1, \dots, k_M, j_1, \dots, j_N$, we have $A\in\mathcal{A}_{k_1, \dots, k_M, j_1, \dots, j_N}$. Therefore, plugging the choices back to the bound, with probability of at least $1-\delta$, we have:
    \begin{align*}
        R_A &\le F(\varphi_1^{(k_1)}, \dots, \varphi_M^{(k_M)}, \phi_1^{(j_1)}, \dots, \phi_N^{(j_N)}) + \frac{C_1}{\sqrt{C_2}}\cdot\sqrt{
            \log 1/\delta + \sum_{i=1}^Nj_i\log(2) + 2\sum_{i=1}^M \log(k_i+1)
        } \\
        &\le 
        F\biggRound{\varphi_1(A) + \frac{1}{\beta_1}, \dots, \varphi_M(A) + \frac{1}{\beta_M}, \phi_i(A), \dots, \phi_N(A)} \\
        &\hspace{4cm} + \frac{C_1}{\sqrt{C_2}}\sqrt{
            \log1/\delta + \sum_{i=1}^N \log\biggRound{\frac{\alpha_i}{\phi_i(A)}} + 2\sum_{i=1}^M \log\bigRound{\beta_i\varphi_i(A) + 2}
        }.
    \end{align*}
\end{proof}

\begin{lemma}
    \label{lem:post_hoc_key_lemma_2}
    Let $R_A$ be a random variable indexed by a set of parameters $A\in\mathcal{A}$ and $\varphi_1(A), \dots, \varphi_M(A)$ be positive statistics of $A$. Let $F:\R^M\to\R_+$ be function satisfying:
    \begin{itemize}
        \item There exists a function $g:\R^M\to\R_+$ that is monotonically increasing at least polynomially with its arguments and $F\in \tilde\bigO(g)$.

        \item For any fixed choices of $\varphi_1, \dots, \varphi_M$ and $\delta\in(0,1)$. For any $A\in\mathcal{A}$ such that $\varphi_i(A)\le\varphi_i$, we have:
        \begin{align}
            \mathbb{P}\biggRound{
                R_A \le F(\varphi_1, \dots, \varphi_M) + C_1\sqrt{\frac{\log 1/\delta}{C_2}}
            } \ge 1 - \delta.
        \end{align}
    \end{itemize}

    \noindent Then, for any choice of parameters $A\in\mathcal{A}$ and for any choice of free variables $\beta_1, \dots, \beta_M$. Let $\delta\in(0,1)$ be given, with probability of at least $1-\delta$, we have:
    \begin{align}
        R_A \le \tilde\bigO\biggRound{g\bigRound{
                \varphi_1(A), \dots, \varphi_M(A)
            } + \sum_{i=1}^M\log\beta_i
        } + C_1 \sqrt{\frac{\log 1/\delta}{C_2}}.
    \end{align}
\end{lemma}

\begin{proof}
    From assumption, we have $F\in\tilde\bigO(g)$, which means for any sequence of arguments $\bigCurl{\varphi_1^{(m)}, \dots, \varphi_M^{(m)}}_{m=1}^\infty$ such that $\varphi_i^{(m)}\to\infty$ as $m\to\infty$, we have:
    \begin{align*}
        \limsup_{m\to\infty} \frac{
            F(\varphi_1^{(m)}, \dots, \varphi_M^{(m)})
        }{g(\varphi_1^{(m)}, \dots, \varphi_M^{(m)})\log^z g(\varphi_1^{(m)}, \dots, \varphi_M^{(m)})} = \tilde C < \infty.
    \end{align*}

    \noindent For some $z>0$. By lemma \ref{lem:post_hoc_key_lemma_1}, we know that for any set of parameters $A\in\mathcal{A}$ and a set of positive variables $\bigCurl{\beta_1, \dots, \beta_M}$, we have:
    \begin{align*}
        R_A &\le F\biggRound{\varphi_1(A) + \frac{1}{\beta_1}, \dots, \varphi_M(A) + \frac{1}{\beta_M}} + \frac{C_1}{\sqrt C_2}\sqrt{\log 1/\delta + 2\sum_{i=1}^M \log\bigRound{\beta_i\varphi_i(A) + 2}} \\
            &\le F\biggRound{\varphi_1(A) 
            + \frac{1}{\beta_1}, \dots, \varphi_M(A) + \frac{1}{\beta_M}} +
            \frac{C_1}{\sqrt{C_2}}\sqrt{2\sum_{i=1}^M \log\bigRound{\beta_i\varphi_i(A) + 2}}
            + C_1\sqrt{\frac{\log 1/\delta}{C_2}} \\
            &\le \underbrace{F\biggRound{\varphi_1(A) 
            + \frac{1}{\beta_1}, \dots, \varphi_M(A) + \frac{1}{\beta_M}} 
            + C_1\sqrt{\frac{2}{C_2}}\cdot\sum_{i=1}^M\log^{\frac{1}{2}}\bigRound{
                \beta_i\varphi_i(A) + 2
            }}_{\tilde F(\varphi_1(A), \dots, \varphi_M(A), \beta_1, \dots, \beta_M)} + C_1\sqrt{\frac{\log 1/\delta}{C_2}}. \\
    \end{align*}

    \noindent Denote the function $\tilde F$ as follows:
    \begin{align*}
        \tilde F(\varphi_1, \dots, \varphi_M, \beta_1, \dots, \beta_M) = 
        F\biggRound{
            \varphi_1 + \frac{1}{\beta_1}, \dots, \varphi_M + \frac{1}{\beta_M} 
        } + C_1\sqrt{\frac{2}{C_2}}\cdot \sum_{i=1}^M \log^{\frac{1}{2}}\bigRound{\beta_i\varphi_i + 2}.
    \end{align*}

    \noindent We have to prove that $\tilde F(\varphi_1, \dots, \varphi_M, \beta_1, \dots, \beta_M)\in\tilde\bigO(\tilde G(\varphi_1, \dots, \varphi_M, \beta_1, \dots, \beta_M))$. Where we denote:
    \begin{align*}
        \tilde G(\varphi_1, \dots, \varphi_M, \beta_1, \dots, \beta_M) = g(\varphi_1, \dots, \varphi_M) + \sum_{i=1}^M\log\beta_i.
    \end{align*}
    
    \noindent We have to prove that for any sequence $\bigCurl{\varphi_1^{(m)}, \dots, \varphi_M^{(m)}}_{m=1}^\infty$ and $\bigCurl{\beta_1^{(m)}, \dots, \beta_M^{(m)}}_{m=1}^\infty$ such that $\varphi_i^{(m)}\to\infty$, $\beta_i^{(m)}\to\infty$ as $m\to\infty$, we have:
    \begin{align*}
        \limsup_{m\to\infty} \frac{\tilde F(\varphi_1^{(m)}, \dots, \varphi_M^{(m)}, \beta_1^{(m)},\dots, \beta_M^{(m)})}{\tilde G(\varphi_1^{(m)}, \dots, \varphi_M^{(m)}, \beta_1^{(m)},\dots, \beta_M^{(m)})\log^z \tilde G(\varphi_1^{(m)}, \dots, \varphi_M^{(m)}, \beta_1^{(m)},\dots, \beta_M^{(m)})} < \infty.
    \end{align*}

    \noindent For some logarithm power of $z>0$. For brevity, we will denote the arguments of all the functions in sequences. By assumption, we have:
    \begin{align*}
        \limsup_{m\to\infty} \frac{
            F\bigRound{
                \bigCurl{\varphi_i^{(m)} + \frac{1}{\beta_i^{(m)}}}_{i=1}^M
            }
        }{
            g(\{\varphi_i^{(m)}\}_{i=1}^M)\log g(\{\varphi_i^{(m)}\}_{i=1}^M)
        } = \limsup_{m\to\infty} \frac{F(\{\varphi_i^{(m)}\}_{i=1}^M)}{g(\{\varphi_i^{(m)}\}_{i=1}^M)\log^z g(\{\varphi_i^{(m)}\}_{i=1}^M)} = \tilde C < \infty.
    \end{align*}

    \noindent Therefore, we have: 
    \begin{align*}
        F\biggRound{\bigCurl{\varphi_i + \frac{1}{\beta_i}}_{i=1}^M}\in\tilde\bigO\bigRound{g(\{\varphi_i\}_{i=1}^M)}\in\tilde\bigO\bigRound{\tilde G(\{\varphi_i\}_{i=1}^M)}. \ \ \ (*)
    \end{align*}
    
    \noindent Now, we just have to prove that all the logarithm terms in concern is also in $\tilde\bigO(\tilde G)$. For all $1\le i \le M$, asymtotically, we have:
    \begin{align*}
        \log^{\frac{1}{2}}\bigRound{\beta_i^{(m)}\varphi_i^{(m)}+2} 
        &\le \sqrt{
            \log\bigRound{\beta_i^{(m)}\varphi_i^{(m)}} + \log 2
        } \\
        &\le \sqrt{\log\bigRound{\beta_i^{(m)}\varphi_i^{(m)}}} + \sqrt{\log2} \\
        &= \sqrt{\log\beta_i^{(m)} + \log\varphi_i^{(m)}} + \sqrt{\log 2} \\
        &\le \log^{\frac{1}{2}}\beta_i^{(m)} + \log^{\frac{1}{2}}\varphi_i^{(m)} + \log^{\frac{1}{2}}2. 
    \end{align*}

    \noindent As $m\to\infty$, we know that $\log^{\frac{1}{2}}\varphi_i^{(m)}$ will be dominated by $g(\{\varphi_i^{(m)}\}_{i=1}^M)$ because $g$ grows at least polynomially with its arguments. On the other hand, $\log^{\frac{1}{2}}\beta_i^{(m)}$ will be dominated by $\log\beta_i^{(m)}$. Hence, the sum of the logarithm terms $\sum_{i=1}^M\log^{\frac{1}{2}}\bigRound{\beta_i^{(m)}\varphi_i^{(m)} + 2}$ will be dominated by $g(\{\varphi_i^{(m)}\}_{i=1}^M)+\sum_{i=1}^M\log\beta_i^{(m)}$ as $m\to\infty$. Therefore, we have:

    \begin{align*}
        \sum_{i=1}^M \log^{\frac{1}{2}}\bigRound{\beta_i\varphi_i + 2} \in \tilde\bigO\biggRound{
            g(\{\varphi_i\}_{i=1}^M)+\sum_{i=1}^M\log\beta_i
        } = \tilde\bigO\bigRound{\tilde G(\{\varphi_i\}_{i=1}^M)}. \ \ \ (**)
    \end{align*}

    \noindent From $(*)$ and $(**)$, we have $\tilde F(\{\varphi_i\}_{i=1}^M) \in \tilde\bigO\bigRound{\tilde G(\{\varphi_i\}_{i=1}^M)}$ and with probability of at least $1-\delta$, we have:
    \begin{align*}
         R_A \le \tilde\bigO\biggRound{g\bigRound{
            \varphi_1(A), \dots, \varphi_M(A)
        } + \sum_{i=1}^M\log\beta_i
        } + C_1 \sqrt{\frac{\log 1/\delta}{C_2}}.
    \end{align*}
\end{proof}

\noindent Using lemmas \ref{lem:post_hoc_key_lemma_1} and \ref{lem:post_hoc_key_lemma_2}, we summarize post-hoc generalization bounds corresponding to the main theorems presented in the main text in table \ref{tab:posthoc_results}. For all of the main results, we set the variables $\beta_i=L^{-1}$ for all possible statistics. For brevity, we define  several notation shortcuts in equation \ref{eq:posthoc_shortcuts} for readers' reference.
\begin{table*}[ht]
  \centering
  \begin{tabular}{lcc}
    \toprule
    \textbf{Generalization bound} & \textbf{Reference} & \textbf{Result} \\
    \midrule 

    ${}^{ \ \ }\tilde\bigO\bigRound{n^{-1/2}{\bf \hat B}_x^2\sqrt{dkL}\mathfrak{AB} + n^{-1/2}L}$ & \cite{article:arora2019theoretical} & -- \\

    $^\mathbb{*}\tilde\bigO\bigRound{n^{-1/2} {\bf \hat B}_x^2\sqrt{dL}\mathfrak{AB} + n^{-1/2}L}$ & \cite{article:lei2023generalization} & -- \\

    \midrule
    $^\mathbb{*}\tilde\bigO\bigRound{n^{-1/2} {\bf \hat B}_x^2\mathfrak{A}^2\mathfrak{C} + n^{-1/2}L}$ & Ours &
    Thm. \ref{thm:basic_bound} \\
    
    $^\mathbb{*}\tilde\bigO\bigRound{n^{-1/2}{\bf \hat B}_x {\bf \hat B}_\pA \mathfrak{A} \mathfrak{C} + n^{-1/2}L}$ & Ours &
    Thm. \ref{thm:last_act_augmentation} \\
    
    $^\mathbb{*}\tilde\bigO\bigRound{n^{-1/2}{\bf \hat B}_\pA^2 \mathfrak{D} + n^{-1/2}L}$
    & Ours & 
    Thm. \ref{thm:all_act_augmentation} \\

    ${}^{ \ \ }\bigO\bigRound{n^{-1/2}\sqrt{\mathcal{W}}\log\bigRound{\eta Ln\cdot{\bf \hat B}_x^2 \mathfrak{A}^2} + n^{-1/2}L}$& Ours &
    Thm. \ref{thm:paracount_bound} \\
    \bottomrule 
  \end{tabular}
  \caption{Summary of post-hoc results for Deep Contrastive Representation Learning (DCRL). We assume that the unsupervised loss function of concern is $\ell^\infty$-Lipschitz with constant $\eta\ge1$. The $\tilde\bigO$ notation hides poly-logarithmic terms of ALL variables and ($^\mathbb{*}$) marks the bounds that have hidden logarithmic dependency on $k$.}
  \label{tab:posthoc_results}
\end{table*}
\begin{equation}
    \label{eq:posthoc_shortcuts}
    \begin{aligned}
        \mathfrak{A} &= \prod_{l=1}^L \|\A{l}\|_\sigma, \ \mathfrak{B} = \prod_{l=1}^L \|\A{l}\|_{Fr}, \\
        \mathfrak{C} &= \biggSquare{
            \sum_{l=1}^L \frac{\|(\A{l} - \M{l})^\top\|_{2,1}^{2/3}}{\|\A{l}\|_{\sigma}^{2/3}}
        }^{3/2}, \\
        \mathfrak{D} &= \biggSquare{
            \sum_{l=1}^L \biggRound{\|(\A{l} - \M{l})^\top\|_{2,1}\cdot\sup_{x\in\Sds_2}\|\f^{1\to l-1}(x)\|_2 \cdot \max_{u\ge l}\frac{\prod_{m=l+1}^u \|\A{m}\|_\sigma}{\sup_{\tilde x\in\Sds_2}\| \f^{1\to u}(\tilde x) \|_2}}^{2/3}
        }^{3/2}.
    \end{aligned}
\end{equation}

\noindent The terms ${\bf \hat B}_x$ and ${\bf \hat B}_\pA$ are the empirical upper bound of inputs and output representations as defined in table \ref{tab:notation_table}.

\section{Lipschitzness of Common Unsupervised Losses}
\label{app:lipschitzness_of_common_un_losses}
Most of the results presented in this paper assume that the unsupervised loss functions are $\ell^\infty$-Lipschitz.  In this section, for completeness, we reproduce the proofs of the $L^\infty$ Lipschitzness of come common losses from \cite{article:lei2023generalization}.

\begin{proposition}{(\cite{article:lei2023generalization} - Lipschitzness of hinge loss).}
    \label{lem:lipschitzness_hinge_loss}
    Let the loss function $\ell:\R^k \to \R$ be defined as:
    \begin{align*}
        \ell(v) = \max\bigCurl{0, 1 + \max_{1\le i \le k}\{-v_i\}}, \ v\in\R^k.
    \end{align*}

    \noindent Then, $\ell$ is $\ell^\infty$-Lipschitz with constant $\eta=1$.
\end{proposition}

\begin{proof}
    Let $v, \bar{v} \in \R^k$, we have:
    \begin{align*}
        \bigAbs{\ell(v) - \ell(\bar{v})} &= \max\bigCurl{0, 1 + \max_{1\le i \le k}\{-v_i\}} - \max\bigCurl{0, 1 + \max_{1\le i \le k}\{-\bar{v}_i\}} \\
        &\le \bigAbs{
            \max_{1\le i \le k}\{-v_i\} - \max_{1\le i \le k}\{-\bar{v}_i\}
        } \\
        &\le \max_{1\le i \le k}\bigAbs{v_i - \bar{v}_i} \\
        &= \|v - \bar{v}\|_\infty.
    \end{align*}

    \noindent Therefore, we proved that $\ell$ is $\ell^\infty$-Lipschitz with constant $\eta=1$.
\end{proof}

\begin{proposition}{(\cite{article:lei2023generalization} - Lipschitzness of logistic loss).}
    \label{lem:lipschitzness_logistic_loss}
    Let the loss function $\ell:\R^k \to \R$ be defined as:
    \begin{align*}
        \ell(v) = \log\biggRound{
            1 + \sum_{i=1}^k \exp(-v_i)
        }, \ v\in\R^k.
    \end{align*}

    \noindent Then, $\ell$ is $\ell^\infty$-Lipschitz with constant $\eta=1$.
\end{proposition}

\begin{proof}
    We prove the following claims sequentially:
    \begin{itemize}
        \item $\ell$ is convex.
        \item $\ell$ has Jacobian with bounded $\ell^1$ norm.
    \end{itemize}

    \noindent\textbf{Claim ${\bf (i)}$}: $\ell$ is convex.
    \noindent\newline To prove that $\ell$ is convex, we prove that the Hessian matrix $H(v)\in\R^{k\times k}$ is positive semi-definite for any input vector $v\in\R^k$ (positive-semidefinite everywhere). For any $v=(v_1, \ \dots \ , v_k)^\top\in\R^k$ and $1 \le i \le k$, we have:
    \begin{align*}
        \frac{\partial \ell(v)}{\partial v_i} &= - \frac{\exp(-v_i)}{1 + \sum_{l=1}^k \exp(-v_l)} \\
        &= - \frac{1}{\exp(v_i) + \sum_{l=1}^k\exp(v_i - v_l)}
         = - \frac{1}{1 + \exp(v_i) + \sum_{l\ne i}\exp(v_i - v_l)}.
    \end{align*}

    \noindent Therefore, for $1\le i \le k$, the diagonal elements of $H(v)$ is:
    \begin{align*}
        H(v)_{ii} = \frac{\partial^2\ell(v)}{\partial v_i^2} &= \frac{\exp(v_i) + \sum_{l\ne i}\exp(v_i - v_l)}{[1 + \exp(v_i) + \sum_{l\ne i}\exp(v_i - v_l)]^2} \ge 0 \text{ for all } v\in\R^k.
    \end{align*}

    \noindent For any $1 \le j \le k$ such that $j \ne i$, the non-diagonal elements of $H(v)$ is:
    \begin{align*}
        H(v)_{ij} = \frac{\partial^2\ell(v)}{\partial{v_i}\partial{v_j}} &= \frac{-\exp(v_i - v_j)}{[1 + \exp(v_i) + \sum_{l\ne i}\exp(v_i - v_l)]^2}.
    \end{align*}

    \noindent From the above, we notice that the Hessian matrix $H(v)$ is strictly diagonally dominant, meaning:
    \begin{align*}
        |H(v)_{ii}| > \sum_{j=1, j\ne i}^k |H(v)_{ij}|. \ \ \ (*)
    \end{align*}

    \noindent For every $\bar v = (\bar v_1, \ \dots \ , \bar v_k)^\top\in\R^k$, we have:
    \begin{align*}
        \bar v^\top H(v) \bar v 
        &= \sum_{i=1}^k \sum_{j=1}^k H(v)_{ij}\bar v_i \bar v_j  \\
        &= \sum_{i=1}^k H(v)_{ii}\bar{v}_i^2 + 2\sum_{i=1}^k\sum_{j=i+1}^k H(v)_{ij}\bar{v}_i\bar{v}_j \\
        &> \sum_{i=1}^k \bar{v}_i^2 \sum_{j=1, j\ne i}^k|H(v)_{ij}| + 2\sum_{i=1}^k\sum_{j=i+1}^k H(v)_{ij}\bar{v}_i\bar{v}_j \ \ \ (\text{From } (*)) \\
        &= \sum_{i=1}^k \sum_{j=i+1}^k |H(v)_{ij}|(\bar{v}_i^2 + \bar{v}_j^2) + 2\sum_{i=1}^k\sum_{j=i+1}^k H(v)_{ij}\bar{v}_i\bar{v}_j \\
        &= \sum_{i=1}^k \sum_{j=i+1}^k |H(v)_{ij}|\bigRound{\bar v_i^2 + 2\cdot\mathrm{sgn}(H(v)_{ij})\bar v_i \bar v_j + \bar v_j^2} \\
        &= \sum_{i=1}^k \sum_{j=i+1}^k |H(v)_{ij}|\bigRound{
            \bar v_i + \mathrm{sgn}(H(v)_{ij})\bar v_j
        }^2 \ge 0.
    \end{align*}

    \noindent Since the choices of $v, \bar{v}$ are arbitrary in $\R^k$, we conclude that the Hessian of $\ell$ is positive definite everywhere in $\R^k$ and that $\ell$ is convex.

    \noindent\newline\textbf{Claim $\bf (ii)$}: $\ell$ is $\ell^\infty$-Lipschitz.

    \noindent For any $v, \bar{v}\in \R^k$ and without loss of generality, assume that $\ell(v) \ge \ell(\bar v)$, we have:
    \begin{align*}
        \bigAbs{
            \ell(v) - \ell(\bar v)
        } &= \ell(v) - \ell(\bar v) \\
        &\le \nabla\ell(v)^\top (v - \bar v) \le \bigAbs{\nabla\ell(v)^\top (v - \bar v)} \\
        &\le \|\nabla \ell(v)\|_1 \cdot \|v - \bar v\|_\infty \le \sup_{v\in\R^k}\|\nabla\ell(v)\|_1\cdot\|v - \bar v\|_\infty,
    \end{align*}

    \noindent by the Holder's inequality. Furthermore, we have:
    \begin{align*}
        \|\nabla\ell(v)\|_1 &= \sum_{i=1}^k \biggAbs{
            \frac{\partial \ell(v)}{\partial v_i}
        } = \frac{
            \sum_{i=1}^k \exp(-v_i)
        }{1 + \sum_{l=1}^k \exp(-v_l)} < 1,
    \end{align*}

    \noindent for all $v\in\R^k$. Hence, $\ell$ is $\ell^\infty$-Lipschitz with constant $\eta = 1$.
\end{proof}

\section{Discussion - Cross-Entropy Inspired Loss Functions}
\label{app:ce_inspired_losses}
In this section, we demonstrate that the main results of this paper also apply to a broader class of loss functions inspired by the Cross-Entropy (CE) loss, to which the Logistic loss belongs. The CE loss, which was originally designed for classification, is defined for a $K$-way classification task as follows:

\begin{equation}
\ell_\mathrm{ce}([p_1, \dots, p_K], y) = -\log\Bigg(
\frac{\exp(p_y)}{\sum_{i=1}^K \exp(p_i)}
\Bigg),
\end{equation}

\noindent where $p_i$ represents the score output by some classifier for class $i$. For example, if we build $K$ linear classifiers $\theta_i\in\mathbb{R}^d$, $1\le i \le K$ on top of some representation function $f:\mathcal{X}\to\mathbb{R}^d$. Then, for an input $x\in\mathcal{X}$, we have $p_i = \theta_i^\top f(x)$. By re-arranging the terms in $\ell_\mathrm{ce}$, we can write:

\begin{equation}
\ell_\mathrm{ce}([p_1, \dots, p_K], y) = \log \Bigg(
1 + \sum_{i\in[K]\setminus \{y\}}\exp(p_i - p_y)
\Bigg),
\end{equation}

\noindent which resembles the form of the logistic loss function. In fact, the logistic loss is one among several closely related methods to adapt the CE loss to the context of CRL. In general, in order to adapt the CE loss to contrastive representation learning, the outputs $p_i$'s in the original CE loss are replaced by some measures that capture the similarity between the representations of the anchors and the negative/positive samples. Specifically, for a representation function $f:\mathcal{X}\to\mathbb{R}^d$ and a similarity measure $\phi_f:\mathcal{X}\times\mathcal{X}\to\mathbb{R}_+$ (that is indexed by $f$), the contrastive CE loss is defined as follows:

\begin{equation}
\begin{aligned}
    \ell_\mathrm{crl-ce}\Big(x, x^+, \{x_i^-\}_{i=1}^k\Big) &= -\log\Bigg(
        \frac{e^{\phi_f(x, x^+)}}{e^{\phi_f(x, x^+)} + \sum_{i=1}^k e^{\phi_f(x, x_i^-)}}
    \Bigg) \\
    &= \log\Bigg( 
        1 + \sum_{i=1}^k \exp\Big[ 
            \phi_f(x, x_i^-) - \phi_f(x, x^+)
        \Big]
    \Bigg).
\end{aligned}
\end{equation}

\noindent Intuitively, $\ell_\mathrm{crl-ce}$ quantifies the negative log likelihood of correctly matching the positive sample to the anchor (the same intuition is also highlighted in \cite{article:oord2018}). For the logistic loss, we have $\phi_f(x, \tilde x) = f(x)^\top f(\tilde x)$ \citep{article:sohn2016}. Some examples of this family of contrastive loss functions adapted from CE is included in table \ref{tab:ce_inspired_losses}.

\begin{table*}[ht]
  \centering
  \begin{tabular}{lcc}
    \toprule
    \textbf{Loss Function} & \textbf{Similarity $\phi_f(x_1, x_2)$} & \textbf{Reference} \\
    \midrule 
    \textbf{N-pair/Logistic} & $f(x_1)^\top f(x_2)$ & \cite{article:sohn2016} \\
    \textbf{InfoNCE/NTXent} & $\tau^{-1}f(x_1)^\top f(x_2)$ & \cite{article:oord2018} \\
    \textbf{ArcFace} & $s\cos\bigRound{m + \frac{f(x_1)^\top f(x_2)}{\|f(x_1)\|_2\cdot\|f(x_2)\|_2}}$ & \cite{article:deng2018arcface} \\
    \textbf{SimCLR} & $\tau^{-1}\frac{f(x_1)^\top f(x_2)}{\|f(x_1)\|_2\cdot\|f(x_2)\|_2}$ & \cite{article:chen2020} \\

    \bottomrule 
  \end{tabular}
  \caption{Examples of contrastive losses adapted from the cross-entropy loss. We refer readers to the respective articles for the meaning of the hyper-parameters.}
  \label{tab:ce_inspired_losses}
\end{table*}

\section{Additional Experiments}
\label{app:additional_experiments}
In line with a more practical approach where the number of negative samples is commonly chosen to be equal to the batch size, we conducted an additional experiment  using the same settings as described in the main text with $k=64$. The results are summarized in figure \ref{fig:ablation_study_depth_k64}. In general, the progression of the bounds with the depths and widths of the network behaves similarly to when we set $k=10$.
\begin{figure*}[htbp!]
    \centering
    \includegraphics[width=\linewidth]{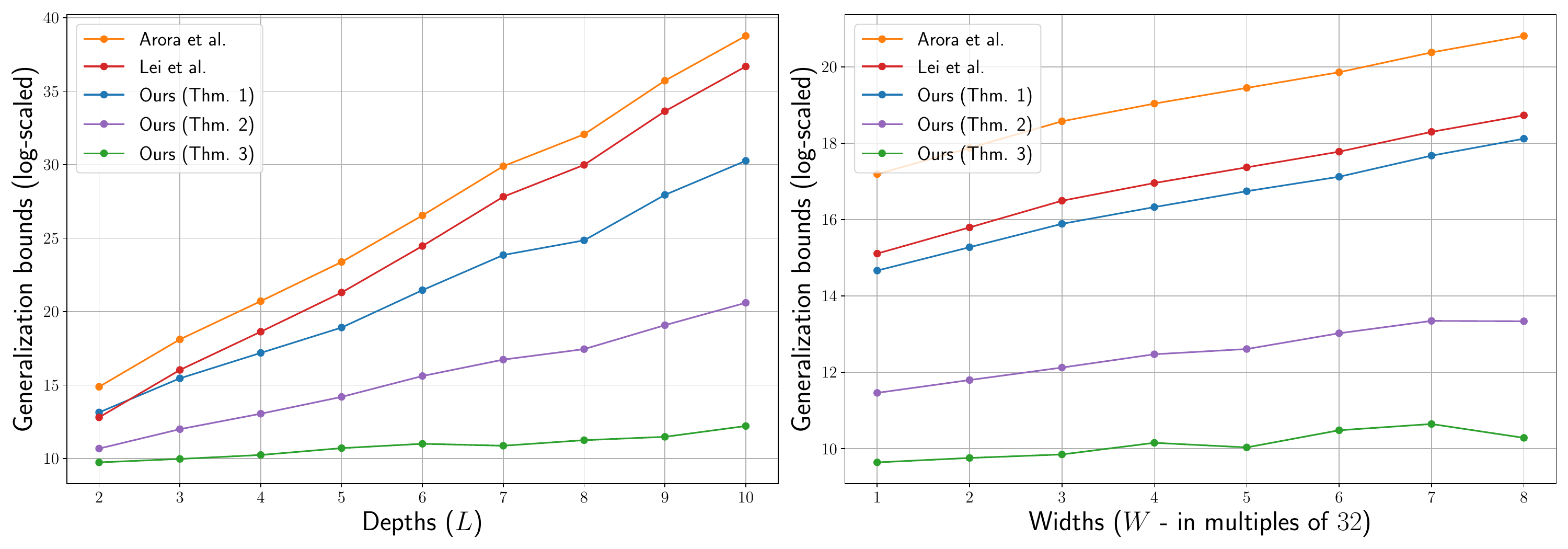}
    \caption{Graphical comparison of our results to that of previous works \citep{article:arora2019theoretical, article:lei2023generalization}. The generalization bounds for all results have their logarithmic terms, constants ($\eta, \rho_i, \dots$) and $\bigO(\sqrt{\log 1/\delta})$ terms truncated. We present the comparison at varying depths (Left) and hidden layer's dimensions (Right).}
    \label{fig:ablation_study_depth_k64}
\end{figure*}

\section{Hardware and Software Specifications}
In this section, we detail the specific computing infrastructure and libraries used to conduct the experiments in the main text. A summary is provided in table \ref{tab:hardware_software_specs}.

\begin{table*}[!hbt]
  \centering
  \begin{tabular}{lccc}
    \toprule
    & \textbf{Hardware/Software} & \textbf{Specs/Version} & \textbf{Reference} \\
    \midrule 
    \multirow{4}{4em}{\textbf{Hardware}} & OS & Windows 10 & -- \\
    & CPU & Intel(R) Xeon(R) W-2133 CPU & -- \\
    & GPU & NVIDIA GeForce RTX 2080  & -- \\
    & RAM & 32GB & -- \\
    & GPU memory & 8GB & -- \\
    \midrule 
    \multirow{3}{4em}{\textbf{Software}} & PyTorch & 1.12.1 & \cite{lib:pytorch2019} \\
    & NumPy & 1.26.4 & \cite{lib:numpy2020} \\
    & Matplotlib & 3.9.0 & \cite{lib:matplotlib} \\
    \bottomrule 
  \end{tabular}
  \caption{Summary of hardware specifications and major libraries used for the experiments.}
  \label{tab:hardware_software_specs}
\end{table*}

\end{document}